\documentclass[11pt]{article} 

\input{packages}
\usepackage[margin=1.2in]{geometry}
\usepackage{times}
\DeclareMathOperator{\BigOm}{\mathcal{O}}
\DeclareMathOperator{\BigOmtil}{\widetilde{\mathcal{O}}}

\DeclareMathOperator*{\argmin}{arg\,min}

\DeclareMathOperator*{\rank}{rank}

\newcommand\numberthis{\addtocounter{equation}{1}\tag{\theequation}}

\newcommand{\fourtup}[4]{(#1,#2,#3,#4)}
\newcommand{\boldgam}{\boldsymbol{\gamma}}

\newcommand{\BigOh}[1]{\BigOm\left({#1}\right)}
\newcommand{\BigOhTil}[1]{\BigOmtil\left({#1}\right)}
\newcommand{\Opt}{\mathsf{Opt}}
\newcommand{\Ovfit}{\mathsf{Ovfit}}
\newcommand{\Grammian}{\mathscr{G}}
\newcommand{\controln}{\mathscr{C}_n}
\newcommand{\observen}{\mathscr{O}_n}

\newcommand{\veczero}{\mathbf{0}}

\newcommand{\Opthat}{\widehat{\Opt}}
\newcommand{\Optil}{\widetilde{\Opt}}

 \newlength\fsep
    \newsavebox\widebox

\newcommand{\Lbar}{\overline{L}}

\newcommand{\ptil}{\widetilde{p}}

\newcommand{\matA}{\mathbf{A}}

\newcommand{\matUbar}{\overline{\matU}}
\newcommand{\Ubar}{\overline{\matU}}

\newcommand{\deff}{\mathbf{d}_{\mathrm{eff}}}
\newcommand{\defftil}{\widetilde{\deff}}

\DeclareMathOperator{\lil}{lil}

\newcommand{\delU}{\delta_{\Ubar}}

\newcommand{\underN}{\underline{N}}

\newcommand{\matutotal}{\matu_{N:1}}
\newcommand{\matwtotal}{\matw_{N:1}}
\newcommand{\matztotal}{\matz_{N:1}}

\newcommand{\Disk}{\mathbb{D}}
\newcommand{\Torus}{\mathbb{T}}

\newcommand{\algspec}{\mathsf{blkspec}}

\newcommand{\phaserank}{\mathsf{phaserank}}

\newcommand{\blockop}{\mathrm{bop}}

\newcommand{\Fsf}{\bm{\mathsf{F}}}
\newcommand{\Fsfone}{\Fsf^{(1)}}
\newcommand{\Fsfonev}{\Fsf^{(1)}_v}
\newcommand{\Fsftwo}{\Fsf^{(2)}}

\newcommand{\Ysf}{\mathbf{\mathsf{Y}}}
\newcommand{\Ysfu}{\Ysf^{(\matu)}}
\newcommand{\Ysfuv}{\Ysf^{(\matu)}_v}

\newcommand{\Hsf}{\bm{\mathsf{H}}}
\newcommand{\Hsfst}{\Hsf_{\star}}

\newcommand{\Hsfone}{\Hsf^{(1)}}

\newcommand{\Gsfst}{\Gsf_{\star}}
\newcommand{\Fsfst}{\Fsf_{\star}}

\newcommand{\Gsf}{\bm{\mathsf{G}}}
\newcommand{\Gsfone}{\Gsf^{(1)}}
\newcommand{\Gsfonev}{G^{(1)}_v}
\newcommand{\Gsfonevt}{\widetilde{G}^{(1)}_v}
\newcommand{\Gsftwo}{\Gsf^{(2)}}
\newcommand{\Gsftwov}{G^{(2)}_v}

\newcommand{\Gsftwoz}{\Gsf^{(2)}_\matz}
\newcommand{\Gsftwozv}{G^{(2)}_{\matz,v}}

\newcommand{\Boundb}{\mathsf{B}}
\newcommand{\Const}{\mathsf{C}}
\newcommand{\Conf}{\mathsf{Conf}}

\newcommand{\Ghat}{\widehat{G}}

\newcommand{\dist}{\mathsf{d}}

\newcommand{\calZ}{\mathcal{Z}}

\newcommand{\vbartoep}{\mathsf{Toep}_T(\vbar)}

\newcommand{\diam}{\mathsf{diam}}

\newcommand{\longvector}[1]{\overset{\rightarrow}{#1}}

\newcommand{\eventU}{\calE_{\Ubar}}

\newcommand{\Egood}{\calE_{\mathrm{good}}}

\newcommand{\jnot}{{\boldsymbol{j}_{0}}}

\newcommand{\mutil}{\widetilde{\mu}}

\newcommand{\pred}{\phi}
\newcommand{\predtil}{\widetilde{\phi}}

\newcommand{\predclass}{\boldsymbol{\Phi}}

\newcommand{\predhat}{\widehat{\pred}}
\newcommand{\predrid}{\pred_{\mathsf{rdg}}}

\newcommand{\Lmax}{L_{\max}}
\newcommand{\Lhat}{\widehat{L}}

\newcommand{\calC}{\mathcal{C}}

\newcommand{\calR}{\mathcal{R}}

\newcommand{\filtg}{\mathcal{F}}
\newcommand{\filtr}{\filtg}
\newcommand{\stoptau}{\boldsymbol{\tau}}

\newcommand{\adv}{\mathsf{adv}}

\newcommand{\rootn}{\sqrt{N}}

\newcommand{\partsize}{r}

\newcommand{\boldDel}{\boldsymbol{\Delta}}

\newcommand{\matDel}{\boldsymbol{\Delta}}
\newcommand{\matdel}{\boldsymbol{\delta}}

\newcommand{\Delpast}{{\boldDel}}

\newcommand{\matS}{\mathbf{S}}

\newcommand{\Apast}{A}
\newcommand{\Vpast}{\mathbf{V}}

\newcommand{\calE}{\mathcal{E}}

\newcommand{\calN}{\mathcal{N}}
\newcommand{\calF}{\mathcal{F}}
\newcommand{\calT}{\mathcal{T}}

\newcommand{\calS}{\mathcal{S}}
\newcommand{\calU}{\mathcal{U}}

\newcommand{\calA}{\mathcal{A}}

\newcommand{\op}{\mathrm{op}}

\newcommand{\calG}{\mathcal{G}}

\newcommand{\blkdiag}{\mathrm{blkdiag}}

\newcommand{\dbar}{\overline{\mathbf{d}}}

\newcommand{\chtwo}{c_{\Htwo}}
\newcommand{\chq}{c_{\mathcal{H}_q}}
\newcommand{\chinf}{c_{\Hinf}}

\newcommand{\complextwo}{H^{(2)}}
\newcommand{\complexinf}{H^{(\infty)}}

\newcommand{\complexq}{H^{(q)}}

\newcommand{\Mcomplextwo}{M}

\newcommand{\Ktwotwo}{K_2}

\newcommand{\regu}{\mu}
\newcommand{\matYplus}{\matY_+}
\newcommand{\Mbar}{\overline{M}}
\newcommand{\Mbarstoch}{\Mbar}
\newcommand{\Mbaradv}{\overline{M}_{\mathsf{adv}}}
\newcommand\SmallMatrix[1]{{%
  \arraycolsep=0.3\arraycolsep\ensuremath{\begin{bmatrix}#1\end{bmatrix}}}}

\newcommand{\hinf}{\mathcal{H}_{\inf}}

\newcommand{\Hinf}{\mathcal{H}_{\infty}}
\newcommand{\Htwo}{\mathcal{H}_{2}}
\newcommand{\Htwoop}{\mathcal{H}_{2}^{\op}}
\newcommand{\Markov}{\mathcal{M}}
\newcommand{\Mknorm}[2]{\lVert \Markov_{#1}\left(#2\right)\rVert_{\mathrm{op}}}

\newcommand{\Nmin}{N_{\min}}
\newcommand{\mixnorm}[2]{\Gamma_{#1}(#2)}
\newcommand{\Hnk}{\mathscr H}

\newcommand{\matU}{\mathbf{U}}

\newcommand{\mateps}{\boldsymbol{\epsilon}}

\newcommand{\conv}{\mathsf{conv}}

\newcommand{\cond}{\mathrm{cond}}

\newcommand{\N}{\mathbb{N}}
\newcommand{\Ast}{A_{\star}}
\newcommand{\Jst}{J_{\star}}

\newcommand{\Abar}{\overline{A}}
\newcommand{\Bbar}{\overline{B}}
\newcommand{\Cbar}{\overline{C}}
\newcommand{\Dbar}{\overline{D}}

\newcommand{\Ahat}{\widehat{A}}
\newcommand{\Bhat}{\widehat{B}}
\newcommand{\Chat}{\widehat{C}}
\newcommand{\Dhat}{\widehat{D}}

\newcommand{\Bst}{B_{\star}}
\newcommand{\Cst}{C_{\star}}
\newcommand{\Dst}{D_{\star}}
\newcommand{\Fst}{F_{\star}}
\newcommand{\Gst}{G_{\star}}

\newcommand{\Gls}{\widehat{G}_{\mathrm{LS}}}
\newcommand{\Gph}{\widehat{G}_{\mathrm{PF}}}
\newcommand{\Gvr}{\widehat{G}_{\mathrm{fil}}}

\newcommand{\iidsim}{\overset{\mathrm{iid}}{\sim}}

\newcommand{\C}{\mathbb{C}}

\newcommand{\matz}{\mathbf{z}}
\newcommand{\ubar}{\overline{\matu}}
\newcommand{\matx}{\mathbf{x}}
\newcommand{\matu}{\mathbf{u}}
\newcommand{\matw}{\mathbf{w}}
\newcommand{\maty}{\mathbf{y}}

\newcommand{\matY}{\mathbf{Y}}

\newcommand{\calM}{\mathcal{M}}

\newcommand{\matmu}{\boldsymbol{\mu}}
\newcommand{\matsig}{\boldsymbol{\sigma}}

\newcommand{\vbar}{\overline{v}}

\newcommand{\Proj}{\mathsf{Proj}}

\newcommand{\R}{\mathbb{R}}
\newcommand{\I}{\mathbb{I}}
\newcommand{\Exp}{\mathbb{E}}

\newcommand{\tr}{\operatorname{tr}}

\renewcommand{\Pr}{\mathbb{P}}

\newtheorem{thm}{Theorem}[section]
\newtheorem{claim}[thm]{Claim}
\newtheorem{lem}[thm]{Lemma}

\newtheorem{cor}[thm]{Corollary}

\newtheorem{prop}[thm]{Proposition}

\newtheorem{defn}{Definition}[section]

\usepackage{mathtools}

\DeclarePairedDelimiter\floor{\lfloor}{\rfloor}

\renewcommand{\theequation}{\arabic{section}.\arabic{equation}}

\newcommand{\phitil}{\widetilde{\phi}}

\newcommand{\ominorm}[1]{\left\lVert  #1\right\rVert_{\ominus}}

\newcommand{\calX}{\mathcal{X}}
\newcommand{\calO}{\mathcal{O}}
\newcommand{\Mon}{\mathsf{Mon}}

\renewcommand{\vec}{\mathrm{vec}}

\newcommand{\Rk}{R_{\matk}}

\newcommand{\ball }[1]{\mathcal{B}_{#1}}
\newcommand{\Leff }{L_{\mathsf{eff}}}

\newcommand{\ridgepred}{\mathsf{ridge}}
\newcommand{\Tbar}{TL}

\newcommand{\matxi}{\boldsymbol{\xi}}

\newcommand{\Err}{\mathsf{Err}}
\newcommand{\Errone}{\Err^{(1)}}
\newcommand{\Errtwo}{\Err^{(2)}}

\newcommand{\oestoch}{\mathsf{ObsErr}_{\mathsf{stoc}}}
\newcommand{\oeadv}{\mathsf{ObsErr}_{\mathsf{adv}}}

\newcommand{\fro}{\mathrm{F}}
\newcommand{\opnorm}[1]{\lVert #1 \rVert_{\mathrm{op}}}
\newcommand{\fronorm}[1]{\lVert #1 \rVert_{\fro}}

\newcommand{\twonorm}[1]{\lVert #1 \rVert_{2}}
\newcommand{\blockopnorm}[1]{\lVert #1 \rVert_{\blockop}}
\newcommand{\rowspace}{\mathsf{rowspace}}

\newcommand{\vcat}[1]{\mathsf{vcat}\left[#1\right]}

\newcommand{\PD}[1]{\mathbb{S}_{++}^{#1}}

\newcommand{\sphere}[1]{\calS^{#1-1}}

\newcommand{\None}{N_1}

\newcommand{\matcir}[1]{\overline{\mathbf{\uppercase {#1}}}}

\newcommand{\Varproc}{\mathbb{V}}

\newcommand{\Nbar}{{\widetilde{N}}}

\newcommand{\matK}{\mathbf{K}}
\newcommand{\matk}{\mathbf{k}}

\newcommand{\mata}{\mathbf{a}}

\newcommand{\tmaty}{\tilde{\mathbf{y}}}
\newcommand{\tmatk}{\tilde{\mathbf{k}}}

\newcommand{\xtil}{\widetilde\matx}
\newcommand{\ytil}{\widetilde\maty}
\newcommand{\ktil}{\widetilde\matk}

\newtoggle{colt}
\togglefalse{colt}

\title{Learning Linear Dynamical Systems with \\Semi-Parametric Least Squares}

\author{ Max Simchowitz\thanks{Department of Electrical Engineering and Computer Sciences, UC Berkeley, Berkeley CA.} \thanks{Denotes equal contribution.} \and Ross Boczar~\footnotemark[1]~\footnotemark[2] \and
Benjamin Recht~\footnotemark[1]}

\begin{document}
\maketitle
\begin{abstract}%

	We analyze a simple prefiltered variation of the least squares estimator for the problem of estimation with biased, \emph{semi-parametric} noise, an error model studied more broadly in causal statistics and active learning. We prove an oracle inequality which demonstrates that this procedure provably mitigates the variance introduced by long-term dependencies. 
	We then demonstrate that prefiltered least squares yields, to our knowledge, the first algorithm that provably estimates the parameters of partially-observed linear systems that attains rates which do not not incur a worst-case dependence on the rate at which these dependencies decay. 
	The algorithm is provably consistent even for systems which satisfy the weaker \emph{marginal stability} condition obeyed by many classical models based on Newtonian mechanics. In this context, our semi-parametric framework yields guarantees for both stochastic and worst-case noise.
\end{abstract}


\section{Introduction}

	Serial data are ubiquitous in machine learning, control theory, reinforcement learning, and the  physical and social sciences. A major challenge is that such data exhibit long-term  correlations, which often obscure the effects of design variables on measured observations and drive up the variance of statistical estimators. 

	In the study of linear, time-invariant dynamical (LTI) systems, for example, a vast literature of both classical and contemporary work typically assumes that the system exhibits a property called \emph{strict stability}, which ensures that long term correlations decay geometrically~\citep{verhaegen1993subspace}. While recent works show this condition can be removed in the special case when the full system state is perfectly observed~\citep{simchowitz2018learning,sarkar2018fast,faradonbeh17a}, it is not known whether the condition is necessary in general. Moreover, it is not fully understood whether the rate of decay of correlations is the right quantity to parametrize the difficulty of learning, even for strictly stable systems.

	Among the many challenges introduced by both non-strictly stable and almost-unstable LTI systems is that the more one probes, the more the long-term correlations compound to yield measurements with large magnitudes, thereby driving up the variance of statistical estimators.  This problem of growing variance arises in many other problem domains as well: for example, the reinforcement learning community has produced a great body of work dedicated to reducing variance in the present of long time horizons and large reward baselines~\citep{sutton98,greensmith2004variance}. 

	This work intervenes by making two contributions. First, we analyze a simple prefiltered variation of the least squares estimator (\emph{PF-LS}) for the problem of estimation with biased, \emph{semi-parametric} noise, an error model studied more broadly in causal statistics and active learning~\cite{chernozhukov2017double,krishnamurthy2018semiparametric}. We prove an oracle inequality which demonstrates that this procedure provably mitigates the variance introduced by long-term dependencies. 
	Second, we demonstrate that prefiltered least squares yields, to our knowledge, the first algorithm that provably estimates the parameters of partially-observed linear systems that attains rates which do not incur a worst-case dependence on the rate at which these dependencies decay. 
	The algorithm is provably consistent even for systems which satisfies the weaker \emph{marginal stability} condition obeyed by many classical models based on Newtonian mechanics.
	In this context, our semi-parametric framework yields guarantees for both stochastic and worst-case noise.

\subsection{Problem Statement\label{sec:problem_statement}}
	We consider the problem of regressing a sequence of observations $(\maty_t) \subset \R^m$ to a sequence of inputs $(\matu_t) \subset \R^p$ for times $t \in [N]$.
	For a fixed $T \in \N$, we define the concatenated input vectors $\ubar_{t} =  [\matu_t | \dots | \matu_{t-T+1}] \in \R^{Tp}$, and assume an \emph{serial, semi-parametric} relationship between $\maty_t$ and $\ubar_t$; that is, there exists a filtration $\{\filtr_t\}$ and a $\Gst \in \R^{m \times Tp}$ for which
	\begin{align}\label{eq:semi_par_def}
	\maty_t = \Gst \ubar_t + \matdel_{t}, \quad \matu_{t} \in \filtr_t, \quad \matdel_t \in \filtr_{t-T}.
	\end{align}
	We choose inputs such that $\matu_t | \filtr_{t-1} \sim \calN(0,I_p)$ which ensures (i) the \emph{Neyman-orthogonality} condition $\Exp[\matdel_t \ubar_t^\top] = 0$, indispensable for consistent semi-parametric estimation~\citep{chernozhukov2017double} and (ii) that the inputs $(\matu_t)$ have well-conditioned covariance.
	This setting is motivated by the problem of estimating the parameters $(\Ast,\Bst,\Cst,\Dst)$ of a discrete-time linear system, which evolves according to the updates 
	\begin{align}
	\matx_{t+1} &= \Ast \matx_t + \Bst \matu_t + B_{w} \matw_t  \nonumber \\
	\maty_{t} &= \Cst \matx_t + \Dst \matu_t + D_{z} \matz_t \label{eq:dynamics},
	\end{align}
	where $(\matu_t) \subset \R^p$ is the sequence of inputs, $(\maty_t) \subset \R^m$ the sequence of outputs, $(\matx_t) \subset \R^n$ is a sequence of states, $(\matw_t) \subset \R^{d_w}$ and $(\matz_t) \subset \R^{d_z}$ are sequences of process noise and sensor noise, respectively, and all above matrices are of appropriate dimension.\footnote{We do not estimate $B_w$ and $D_z$, which are in general unidentifiable without a specific noise model.} Crucially, we \emph{do not} observe the system states $\matx_t$ or the noises $\matz_t$ and $\matw_t$. 
	We shall assume that the process and sensor noises can be chosen \emph{semi-adversarially} in the sense that $\matw_t, \matz_t \in \filtr_{t-T}$ (i.e. an adversary may only act with a $T$-step delay). This model captures both stochastic and worst-case oblivious noise. A simple recursion shows that this condition implies that~\eqref{eq:semi_par_def} holds in this setting, with $\Gst$ equal to
	\begin{align*}
	\Gst := [\Dst, ~\Cst\Bst, ~\Cst \Ast \Bst, \Cst \Ast^2 \Bst, \dots, \Cst \Ast^{T-2} \Bst] \in \R^{m \times Tp},
	\end{align*}
	the length-$T$ response function\footnote{We suppress the dependence on $T$ to ease notation.} from the inputs $\matu_{t-T+1},\dots,\matu_T$ to the observation $\maty_{t}$. Examining the dynamical equations~\eqref{eq:dynamics}, we see that the error $\matdel_{t}$ corresponds to the residual part of $\maty_t$ which is does not depend on $\ubar_{t}$, and is therefore $\{\filtr_{t-T}\}$-adapted.

	An important recent result due to~\cite{Samet18} demonstrates that for these LTI systems, a consistent estimate of $\Gst$ can produce a consistent estimate $(\Ahat,\Bhat,\Chat,\Dhat)$ of an \emph{equivalent realization} of $(\Ast,\Bst,\Cst,\Dst)$.\footnote{That is, a pair $(\Abar,\Bbar,\Cbar,\Dbar)$ such that $\Dbar = \Dst$, and for all $j \in \{0,1,2,\dots\}$, $\Cst\Ast^j \Bst = \Cbar\Abar^j \Bbar$.}
	Furthermore,~\cite{Samet18} show that if the operator norm $\|\Ast\|_{\op}$ is strictly less than one, then ordinary least squares (OLS) can efficiently recover $\Gst$ from the inputs $\matu_1,\dots,\matu_N$. Formally, if the process and sensor noises are i.i.d normally distributed, the least squares estimator which uses samples $\matu_{\None-T+1},\ldots,\matu_N$ for some $\None \le N/10$,
	\begin{align}\label{eq:GLS_def}
	\Gls := \argmin_{G \in \R^{m \times Tp}} \sum_{t=\None}^{N} \|  G\ubar_t  - \maty_t\|_2^2,
	\end{align}
	converges to $\Gst$ at a rate of $\BigOh{N^{-1/2}}$.  
	Unfortunately, the condition $\|\Ast\|_{\op} < 1$ is quite stringent, and the learning rates degrade as $\opnorm{\Ast}$ approaches $1$.
	Indeed, many systems of interest do not even satisfy a weaker condition known as \emph{strict stability}: $\rho(\Ast) < 1$, where $\rho(\cdot)$ denotes the spectral radius. For example, simple oscillators, integrators, and elementary systems that arise from Newton's laws yield realizations where $\rho(\Ast) = 1$. For example, the LTI system corresponding to the discretization of the differential equation $F=m\ddot x$, with sampling time $\Delta > 0$, includes the matrix
	$\Ast =\exp\left( \Delta \cdot \scriptstyle{\begin{bmatrix} 0 & 1 \\ 0 & 0  \end{bmatrix}} \right) = \scriptstyle{\begin{bmatrix} 1 & \Delta \\ 0 & 1  \end{bmatrix}}$. 
	This matrix violates the strict stability condition, yet satisfies $\rho(\Ast) = 1$. As mentioned above, non-asymptotic bounds for learning LTI systems typically yield rates which depend on the \emph{inverse stability gap} $1/(1 - \rho(\Ast))$\citep{oymak2018stochastic,hardt16,shah12}; for example,~\citet{oymak2018stochastic} requires one to select a horizon length $T$ for which $\|\Ast\|_{\op}^T \le .99$, which necessitates that $T \gtrsim \frac{1}{1 - \rho(\Ast)}$. This work, on the other hand, suggests that a dependence on stability gap can be avoided in many cases, and the difficulty of learning can instead by parametrized by quantities that are often less conservative.

\subsection{Prefiltered Least Squares (PF-LS)} 
	In light of the limitations of ordinary least squares, we analyze PF-LS, a simple prefiltering step to improve the estimation of the matrix $\Gst$ in the general semi-parametric setting~\eqref{eq:semi_par_def}. In Section~\ref{sec:lti_results}, we specialize our analysis to establish consistent recovery of any linear dynamical system for which $\rho(\Ast) \le 1$.
	Prefiltering mitigates the magnitude of the errors $\matdel_t$ by learning a coarse linear filter of future outputs, denoted $\predrid \in \R^{m \times \Lbar}$, for $\Lbar \in\N$. This filter uses a sequence $(\matk_t) \in \R^{\Lbar}$ encoding past observations to estimate $(\maty_t)$ via the prediction $\predrid \cdot \matk_t$. We then estimate $\Gst$ by regressing the filtered observations $(\maty_t - \predrid \cdot \matk_t)$ to $(\ubar_t)$. Concretely, our procedure is achieved with the following two steps of least squares.
	\begin{align}
  \predrid &\leftarrow \argmin_{\phi \in \R^{m \times \Lbar }} \sum_{t=\None}^N\| \maty_t - \phi \cdot  \matk_t \|_{2}^2 + \regu^2 \|\phi\|_{\fro}^2 &\label{eq:prefilter}\\
	\Gph &\leftarrow \argmin_{G \in \R^{ m \times Tp}} \sum_{t=\None}^N\|(\maty_t - \predrid \cdot \matk_t) -  G\ubar_t \|_2^2 \label{eq:post_hoc_def}
	\end{align}
	Throughout, we let $\Nbar = N - N_1 + 1$, and we use the notation $\matDel \in \R^{\Nbar \times m}$ to denote the matrix whose rows are $\matdel_{\None}^\top,\dots,\matdel_{N}^\top$ and $\matK \in \R^{\Nbar \times \Lbar}$ the matrix whose rows are $\matk_{\None}^\top,\ldots,\matk_N^\top$. Our first contribution is the following inequality. Throughout, we use $\lesssim$ to denote inequality up to universal multiplicative constants.

	\begin{thm}[Prefiltering Oracle Inequality, Informal]\label{thm:main_semi_par_informal} Consider the general semi-parametric setting described in Section~\ref{sec:problem_statement}, and suppose that $\matk_t \in \filtr_{t-T}$. Then, with high probability,
	\begin{align*}
	\opnorm{\Gph - \Gst} \lesssim \frac{\Opt_{\regu} + \Ovfit_{\regu}}{N}\cdot\BigOhTil{ \sqrt{T(p+m+\Lbar)}}\:.
	\end{align*}
	Here, $\BigOhTil{\cdot}$ hides logarithmic terms in $N$, $\opnorm{\matK}$, and $1/\mu$. The term
	\begin{equation*}
	\Ovfit_{\regu} = \BigOhTil{\sqrt{T(p+\Lbar) \opnorm{\Gst}}}
	\end{equation*}
	captures the extent to which prefiltering overfits to $(\ubar_t)$, and
	\begin{align}\label{eq:opt_def}
	\Opt_{\mu} := \min_{\phi \in \R^{m \times \Lbar}} \|\matDel - \matK \phi^\top\|_{\op} + \mu\|\phi\|_{\op}
	\end{align} describes the data-dependent prediction error of the best filter. 
	\end{thm}
	We defer a precise statement of Theorem~\ref{thm:main_semi_par_informal} to Theorem~\ref{thm:ph_vr_oracle}.
	Note that this result makes no assumptions on the structure of the noise $\matdel_t$ or the features $\matk_t$, other than the measurability assumptions that $\matk_t,\matw_t,\matz_t \in \filtr_{t-T}$. Moreover, the term $\Opt_{\mu}$ captures to the actual sequence of errors $\matdel_t$, rather than an a priori upper bound.
	For a sense of sense of scaling, the overfitting term  $\Ovfit_{\regu}$ is typically dominated by $\Opt_{\mu}$, and by setting $\phi = 0$, $\Opt_{\mu} \le \opnorm{\matDel}$. When $\matdel_t = \BigOh{1}$ on average, this terms behaves as $\sim \sqrt{N}$, and thus $\opnorm{\Gph - \Gst}$ decays at a rate of of $\BigOhTil{N^{-1/2}}$. In general, we only need to ensure $\Opt_{\mu} \sim \sqrt{N}$.

\subsection{Organization}  Section~\ref{sec:lti_results} presents a precise statement of the PF-LS oracle inequality for LTI systems, Proposition~\ref{prop:opt_to_final_lti}, as well as bounds for the associated term $\Opt_{\mu}$ in terms of 
%
the \emph{phase rank} of $\Ast$~(Definition~\ref{def:phase_rank}). 
Consistency of estimation for marginally stable LTI systems is presented as a consequence (Corollary~\ref{main:cor}).
Section~\ref{sec:pfls_oracle} walks the reader through the analysis of the PF-LS estimator, culminating in a formal statement of the oracle inequality, Theorem~\ref{thm:ph_vr_oracle}. Section~\ref{sec:bounding_opt} provides a proof sketch for results described in Section~\ref{sec:lti_results}, and Section~\ref{sec:related_work} addresses related work. Complete proofs are deferred to the Appendix, which is divided into three parts: Part~\ref{part:second_results} contains graphical illustrations of phase rank~(Appendix~\ref{app:phase_rank_examples}), the proof of Corollary~\ref{main:cor}, and the lower bound for OLS, Theorem~\ref{thm:main_lb}. Part~\ref{part:PH-VR} pertains to the the oracle inequality and related material from Section~\ref{sec:pfls_oracle}. Part~\ref{part:LTI} addresses results specific to LTI systems. The appendix begins with a \hyperref[sec:appendix_organization]{preface} which consolidates notation and outlines the organization of the subsequent appendices in greater detail.

\section{Rates for Learning LTI Systems\label{sec:lti_results}}
In the setting of marginally stable systems, we can not guarantee that $\opnorm{\matDel}$ grows as $\sqrt{N}$ due to the possible accumulation of system inputs. Indeed, Theorem~\ref{thm:main_lb} in the appendix shows that the OLS estimator is inconsistent whenever $\rho(\Ast) \ge 1$ and the system satisfies a weak identifiability criterion. Even for strictly systems, it is not clear whether the inverse stability gap $(1 - \rho(\Ast))^{-1}$, correctly describes the difficulty of estimation. 
What we show is that by choosing a large enough filter length $L \in \N$ and features
	\begin{align}\label{eq:k_def}
	\matk_t := [\maty_{t-T}^\top | \maty_{t-2T}^\top | \dots | \maty_{t-TL}^\top]^\top\in \R^{Lm},
	\end{align} 
	we can ensure both that $\Opt_{\mu} \sim N^{1/2}$ for marginally stable systems and that $\Opt_{\mu}$ need not depend on the stability gap $1-\rho(\Ast)$.
	Our choice of $\matk_t$ in~\eqref{eq:k_def} corresponds to filtering a subsampled history of the outputs in order to predict $\maty_t$. This linear prefiltering step is in the spirit of many schemes detailed in the system identification and time-series literature (see Section~\ref{sec:related_work}), many of which ensure performance in both stochastic and adversarial settings for strictly stable systems. From the perspective of prefiltered semi-parametric learning, we observe that $\Lbar$ corresponds to $L\cdot m$, and that the dynamical equations~\eqref{eq:dynamics} imply that features~\eqref{eq:k_def} are $\{\filtr_{t-T}\}$-adapted.
	
	We begin the task of deriving explicit estimation rates for LTI systems by first establishing an oracle inequality for the PF-LS scheme given by \eqref{eq:prefilter} and \eqref{eq:post_hoc_def}, for two particular noise models.

\begin{defn}[Noise Models]\label{asm:noise}  In the \textbf{\emph{stochastic noise model}}, $\matw_t | \filtr_{t-T}$ and $\matz_{t-T} | \filtr_{t-T}$ are conditionally $1$-subgaussian.\footnote{That is, $\Exp[\exp(\lambda\langle v, \matw_t \rangle) \mid \filtr_{t-T}] \le \exp(\lambda^2 \|v\|_2^2)$ for all $v \in \R^{d_w}$, and analogously for $\matz_t$}%
 In the $\textbf{\emph{adversarial noise model}}$, the noise processes satisfy the bounds  $\|\matw_t\|_2^2 \le d_w$ and $\|\matz_t\|_2^2 \le  d_z$ with probability $1$.
 \end{defn}

Observe that shaped and/or scaled noise can be addressed by altering $B_w$ and $D_z$ appropriately. The conditions $\|\matw_t\|_2^2 \le d_w$ and $\|\matz_t\|_2^2 \le  d_z$ make the adversarial and stochastic noise models comparable, as in the stochastic noise model we have $\Exp[\|\matw_t\|_2^2] \lesssim d_w$ and $\Exp[\|\matz_t\|_2^2] \lesssim d_z$ by subgaussianity. We shall assume $\matx_1 = 0$ for the rest of the section, and we address general $\matx_1$ in the Appendix. Now,  we introduce the following parameter, which illustrates the dependence of our bounds on the eigenstructure of $\Ast$, the conditioning of the eigenvalues, and the magnitude of the noises encoded in $B_w$ and $D_z$.
\begin{defn}\label{defn:magnitude_bound} Let $\Ast = S \Jst S^{-1}$ denote the Jordan-normal decomposition of $\Ast$. We define
\begin{align*}
\Mbar &:=  \|\Cst S\|_{\op} \left( \|S^{-1}\Bst\|_{\op} + \|S^{-1}B_w\|_{\op} \right) + \|\Dst\|_{\op} + \|D_z\|_{\op}\:.
\end{align*}
\end{defn}
Lastly, our results apply once $N$ satisfies a moderate lower bound. Specifically, we define\iftoggle{colt}{~$\Nmin~= \max\{10TL, c Tp \log^4(2Tp) \},$}{
	\begin{align*}
\Nmin  = \max\{10TL,c Tp \log^4(2Tp)\},
\end{align*}
} where $c$ is a sufficiently large constant. Our result bounds $\opnorm{\Gph  - \Gst}$ for LTI systems in terms of $\Opt_{\mu}$, $\Mbar$, and dimension quantities. Furthermore, we let $\log_+(x) := \max\{1,\log(x)\}$. With these quantities defined, we state a specialized version of our general oracle inequality, Theorem~\ref{thm:ph_vr_oracle}.
\begin{restatable}{prop}{optfinal}\label{prop:opt_to_final_lti} Fix $T$ and $L$, and suppose that $N \ge \Nmin$, $N_1 = TL$, $\rho(\Ast) \le 1$, and that the largest Jordan block of $\Ast$ is of size $k$. Choosing some $\mu \ge 1$ and defining
\begin{align*}
\dbar := \ptil +  Lm\left(\log_+ \Mbarstoch + k\log_+ N\right) = \BigOhTil{p + Lmk},
\end{align*}
it holds with probability at least $1-\delta - (2Np)^{-\log^2 (2Tp)\log^2(2Np)}$ in the stochastic noise model that
\begin{align*}
\opnorm{\Gph  - \Gst}  \lesssim   \left(\frac{\Opt_{\regu} +\opnorm{\Gst} \sqrt{T(\dbar+\log\tfrac{1}{\delta})}  + \mu}{\sqrt{N}}\right) \cdot \sqrt{ \frac{T(\dbar +   \log \tfrac{1}{\delta})}{N} }.
\end{align*}
In the adversarial noise model, we instead take $\dbar := \ptil +  Lm\left(\log_+ \Mbarstoch  + \log_+ (d_z + d_w) + k\log_+ N\right)$.
\end{restatable}
We remark that the logarithmic terms in $\dbar$ can be refined further, but we we state the above bound for its relative simplicity. In the following subsection, we show that 
for any marginally stable system, one can ensure that $\Opt_{\mu} \sim N^{1/2}$ as long as $L$ is chosen to be sufficiently large.  As a consequence, we verify in Appendix~\ref{sec:cor_append} that combining prefiltered least squares with the Ho-Kalman algorithm analyzed in~\cite{Samet18} provides the consistent estimation of the underlying system parameters themselves.
\begin{cor}[Recovery of System Parameters]\label{main:cor} Suppose that the system $(\Ast,\Bst,\Cst,\Dst)$ is minimal (Definition~\ref{defn:obs_control}), and $\rho(\Ast) \le 1$. Then for $N$, $L$, and $T$ sufficiently large and $\mu \lesssim \sqrt{N}$, there exists a constant $\Const_1$ and $\Const_2$ depending on system parameters $(\Ast,\Bst,\Cst,\Dst)$, the dimension $(n,m,T,p)$, and $\|\matx_1\|$; such that with high probability there exists a unitary matrix $S$ with
\begin{align*}
&(a) \quad \quad \Opt_{\mu} \lesssim \Const_1\sqrt{N}\\
&(b)\quad \quad \|\Gph - \Gst\| \lesssim \Const_2\sqrt{\log(N)/N}\\
&(c)\quad \quad \max\left\{\|\Ahat - S\Abar S^*\|_{\fro}, \|\Bhat - S \Bbar\|_{\fro}, \|\Chat - \Cbar S^*\|_{\fro}, \|\Dhat - \Dbar\|_{\fro} \right\} \lesssim \Const_3(\log(N)/N)^{1/4}\:.
\end{align*}
Here, $(\Abar,\Bbar,\Cbar,\Dbar)$ is a certain realization of $\Gst$: the output of the Ho-Kalman algorithm on $\Gst$, and $(\Ahat,\Bhat,\Chat,\Dhat)$ is the output of the Ho-Kalman algorithm on $\Gph$.
\end{cor}
We remark that the condition that $(\Ast,\Bst,\Cst,\Dst)$ is minimal is necessary to ensure identifiability in the manner described by Corollary~\ref{main:cor}. When minimality fails, Appendix~\ref{sec:cor_append} explains that there always exists a reduced system model equivalent (in an input-output sense) to $(\Ast,\Bst,\Cst,\Dst)$, which can be estimated in the sense of Corollary~\ref{main:cor}. 
%
As mentioned above, Theorem~\ref{thm:main_lb} shows that OLS inconsistently estimates $\Gst$ when $\rho(\Ast) = 1$, even with no process or sensor noise. 

\subsection{Learning without the stability gap\label{sec:learning_no_stab}}
While Corollary~\ref{main:cor} ensures consistent for marginally stable systems with $\rho(\Ast) \le 1$, we in fact wish to answer the more ambitious question: how accurately does the inverse stability gap $(1 - \rho(\Ast))^{-1}$ describe the intrinsic difficulty of learning an LTI system? To this end, we describe an alternative criterion we call \emph{phase rank}, which does not depend on the stability radius in the worst case. We use phase rank to bound the term $\Opt_{\mu}$, and thus, through Proposition~\ref{prop:opt_to_final_lti}, upper bound the learning rate of the estimator $\Gph$. As a second alternative to the inverse stability gap, in Appendix~\ref{app:strong_observability} we define a condition called \emph{strong observability}, related to the classical notion of observability in control theory~\citep{hautus1983strong}. This condition also allows us to bound estimation rates in terms of quantities that do not degrade (in the worst case) as $\rho(\Ast) \to 1$.

\textbf{Phase Rank:} Many approaches in the recent learning theory literature have developed bounds which depend directly on the spectrum of $\Ast$ and magnitudes of the state-space realization matrices $(\Ast,\Bst,\Cst,\Dst)$; however, many existing works have incurred dependencies on the minimal polynomial of $\Ast$, which is exponentially large in the worst case.
Our first approach adopts a new measure of complexity we call \emph{phase rank}, inspired by~\cite{hazan18},  derives bounds from the spectrum of $\Ast$ without paying for the size of the minimal polynomial. 
	Rather than capturing the effects of all eigenvalues, the phase rank groups together eigenvalues with approximately the same phase, and only considers the eigenvalues of $\Ast$ which lie near the boundary of the unit disk.  Formally, let $\Disk: \{z \in \C: |z|\le 1\}$ denote the complex unit disk, and for a marginally stable $\Ast$, let $\algspec(\Ast) \subset \C \times \N$ denote the set\footnote{As we will be taking the maximum over $\algspec(\Ast)$, set is equivalent to multiset for our purposes.} of all pairs $(\lambda, k)$, where $\lambda$ is an eigenvalue of $\Ast$ and $k$ is a size of an associated Jordan block. We define phase rank as follows:
\begin{restatable}[Phase Rank]{defn}{phaserank} \label{def:phase_rank}  Let $\alpha \ge 1$. We say that $\Ast$ has $(\alpha,T)$-phase rank $d$ if there exists $\mu_1,\dots,\mu_d \in \Disk$ such that, for any $(\lambda,k) \in \algspec(\Ast)$ with $|\lambda| \ge 1 - ((1+\alpha)T)^{-1}$, there exists at least $k$ elements $\mu_{i_1},\dots,\mu_{i_k} \subset \{\mu_1,\dots,\mu_d\}$ satisfying
\begin{align*}
\max_{j \in [k]}\min_{\mutil: \mutil^T = \mu_{i_j}^T}|\lambda - \mutil| \le \alpha \left(1 - |\lambda|\right).
\end{align*}
\end{restatable}

In the above definition, the parameter $T$ allows us to group together eigenvalues having approximately the same phase mod $2\pi/T$, and the $\alpha$ parameter controls the `width' of the approximation. The phase rank is typically small for many systems of interest; for example, for real diagonalizable systems, it is at most $2$, \emph{regardless of the stability radius}, thereby obviating a dependence on $1 - \rho(\Ast)$ in the worst case. Phase rank can also take advantage of benign systems which do exhibit stability; for example, phase rank is equal to $0$ for strictly stable systems with  $\rho(\Ast) < 1 - ((1+\alpha)T)^{-1}$. In Appendix~\ref{app:phase_rank_examples}, we give some visual diagrams to aid the intuition for this condition. Moreover, the $(1,T)$ phase-rank is at most the degree of the  ``minimal-phase polynomial'' of $\Ast$, the measure of complexity studied by~\cite{hazan18} which inspired this condition. With this definition in hand, we provide the following bound on $\Opt_\mu$:

\begin{restatable}[Bounds for Phase Rank]{prop}{propstochphase}\label{prop:phase_rank_intro} Suppose that $\Ast$ has $(\alpha,T)$ phase rank $d$, and maximum Jordan block size $k$. Then, for any $\delta \in (0,1)$ and $N \ge T(d+1+\alpha) \max\{ m,\log(1/\delta)\}$ and $N_1 \ge T L$, it holds with probability $1-\delta$ under the \emph{stochastic noise} setting of Definition~\ref{asm:noise} that
\begin{align*}
N^{-1/2}\Opt_{\mu} \lesssim &\; (\Mbar + \mu N^{-1/2}) \cdot  T^{k-1/2}C_{\alpha,d,k},\quad \text{where}\\
C_{\alpha,d,k} := &\; 2^d\left(k^2(1+\alpha)^{k-\frac{1}{2}} + d^{k-\tfrac{1}{2}}\right) \:.
\end{align*}
\end{restatable}
For systems where the eigendirections of $\Ast$ can be `disentangled' when observed by $\Cst$, we show in Appendix~\ref{sec:main_results_disent} that one may decouple the eigenmodes of $\Ast$ to only have to consider the phase rank restricted to smaller portions of the spectrum of $\Ast$. For example, if there exists a well-conditioned matrix $V$ for which $\Cst V$ and $V^{-1} \Ast$ are diagonal with blocks $(C_1,C_2)$, $(A_1,A_2)$, respectively, then we only incur a penalty for the maximum of the phase ranks of $A_1$ and $A_2$.  Appendix~\ref{sec:poly_approx_main} also gives refinements of Proposition~\ref{prop:phase_rank_intro} that take more granular aspects of $\algspec(\Ast)$ into account. For constant phase rank $d$, we have the bound $N^{-1/2}\Opt_{\mu} = \BigOh{T^{k-1/2}}$, which roughly matches the dependence on term $\opnorm{\Fst}$ in the bounds in~\cite{Samet18}.


\section{Oracle Inequality for Prefiltered Least Squares\label{sec:pfls_oracle}}
The goal of this section is to present Theorem~\ref{thm:ph_vr_oracle}, a technical version of the general oracle inequality Theorem~\ref{thm:main_semi_par_informal} for the estimator $\Gph$. To guide the proofs and intuition, we first consider the performance of the estimator which uses an arbitrary, fixed filter $\pred$, rather than the data-dependent filter chosen by~\eqref{eq:prefilter}: 
\begin{align}
\Gvr(\pred) := \argmin_{G \in \R^{m \times Tp}}\sum_{t=\None}^N \|(\maty_t - \phi \cdot \matk_t) -  G \ubar_t\|_2. \label{eq:Gvr_def}
\end{align}  
Note that for $\phi = 0$, $\Gvr(\phi) = \Gls$ is the OLS estimator.
To analyze $\Gvr(\pred) $, we must define the associated error $\matdel_{\phi,t} := \matdel_{t} - \phi \cdot \matk_t$. Since $\matdel_t,\matk_t $ are $\filtr_{t-T}$-adapted, $\matdel_{\phi,t}$ is as well. Thus, $\maty_t - \phi \cdot \matk_t = \Gst \ubar_t + \matdel_{\phi,t}$ is a semi-parametric model describing the relationship between  $\maty_t - \phi \cdot \matk_t$ and $\ubar_t$. 
We shall now establish two crucial properties which ensure estimation. The first shows that the matrix $\overline{\matU} := [\ubar_{\None} | \dots |\ubar_N]^\top$ is well-conditioned with high probability.
\begin{lem}[Lemma C.2 in~\cite{Samet18}] \label{lem:Uconditioned} Define the event $\eventU$ and number $\delU$ as
\begin{align*}
\eventU := \{\frac{N}{2}I \preceq \Ubar^\top \Ubar \preceq 2I\} \quad \delU := (2Np)^{-\log^2 (2Tp)\log^2(2Np)}.
\end{align*}
Then, if $\None \le \frac{1}{10} N$ and the sample size $N$ satisfies $N \ge c' Tp \log^2(2Tp) \log^2(2N p)$ for a sufficiently large $c' > 0$, it holds that $	\Pr[ \eventU^c] \le \delU$.
\end{lem}
Inverting, we see that it suffices to take $N \ge \Nmin$ for a possibly larger constant $c$ to ensure the condition $N \ge c' Tp \log^2(2Tp) \log^2(2N p)$ holds.
The second property we shall use is \emph{Neyman orthogonality}, which states that
\begin{align*}
\Exp[\ubar_{t}\matdel_{\phi,t}^\top] = \Exp[\Exp[\ubar_{t}|\filtr_{t-T}]\matdel_{\phi,t}^\top]  = 0\:.
\end{align*}
This property is satisfied in our setting, since $\matdel_{\phi,t}$ is $\filtr_{t-T}$ adapted and $\Exp[\matu_{t} |\filtr_{t-1}] = 0$. \cite{chernozhukov2017double} show that under general conditions, Neyman orthogonality generally ensures consistency of least squares, and \cite{krishnamurthy2018semiparametric} recently demonstrated that this idea implies consistency in a time-series setting. Our first result adapts this argument to handle the fact that our regression variables, $\ubar_t$ have structure; namely, they are concatenated subsequences of $(\matu_t)$. Denoting $\matDel_{\phi} \in \R^{\Nbar \times m}$ to be the matrix $[\matdel_{\phi,\None}^\top\mid\ldots\mid\matdel_{\phi,N}^\top]^\top$, we show the following bound.
\begin{prop}[Error Bound for Fixed Filter]\label{prop:gvr_bound}  For any fixed filter $\pred \in \R^{m \times Lm}$, $\delta \in (0,1)$, and $\kappa > 0$, it holds with probability at least $1- \delta$ that on $\eventU$,
    \begin{align}\label{eq:fixed_phi_eq_body}
    \| \Gvr(\pred) - \Gst\|_{\op} \lesssim \frac{(\opnorm{N^{-1/2}\matDel_{\phi}} + \kappa)T^{1/2}}{\sqrt{N}}\sqrt{\ptil + m + \log \tfrac{1}{\delta} + \lil (\tfrac{\opnorm{\matDel_{\phi}}}{\kappa N^{1/2}})},
    \end{align}
    where $\lil(x) := \log_+(\log_+(x))$ and $\ptil:= p \min\{T,\log^2(eTp)\log^2(Tp)\}$. In particular, the complement of~\eqref{eq:fixed_phi_eq_body} occurs with probability at most $ \delta + \delU$.
\end{prop}
For a sense of scaling, observe that whenever the sequence $(\matdel_{t,\phi})$ is $\calO(1)$ in magnitude on average, then $\tfrac{\opnorm{\matDel_{\phi}}}{\sqrt{N}}$ is $\mathcal{O}(1)$ with high probability, yielding estimation rates of $\widetilde{\mathcal{O}}\left(\sqrt{\tfrac{T(p+m)}{N}}\right)$. 

\textbf{Proof Sketch:} Proposition~\ref{prop:gvr_bound} is derived as a special case of a more general result, Theorem~\ref{thm:chain_semi_par}, which relies on self-normalized tail bounds for martingale sequences due to~\cite{yasin11}. This theorem is similar in spirit to the tail bounds obtained by~\cite{krishnamurthy2018semiparametric} for semi-parametric contextual bandits. The parameter $\kappa > 0$ arises from the use of these tools, but it can be chosen quite small due to the doubly-logarithmic dependence in $1/\kappa$. The novelty of our bound comes from a careful chaining argument in Appendix~\ref{sec:chaining_toeplitz_proof} specific to semi-parametric regression with the concatenated sequence $(\ubar_t)$, based on the techniques in~\cite{krahmer2014suprema}. This bound yields a dependence on $\ptil$ instead the larger quantity $Tp$. In service of this argument, we give a recipe for applying Talagrand's chaining~\citep{talagrand2014upper} to self-normalized martingale tail bounds in Appendix~\ref{sec:martingale_chaining}, which may be of general interest. 

\subsection{Statement of the Oracle Inequality}
In Proposition~\ref{prop:gvr_bound}, we bounded the error for the least squares estimate associated with a fixed predictor, $\Gvr(\pred)$, in terms of its associated error $\opnorm{\matDel_{\pred}}$. Specifically, Proposition~\ref{prop:gvr_bound} implied that when $\opnorm{\matDel_{\pred}}$ grows as $\BigOhTil{N^{1/2}}$, $ \| \Gvr(\pred) - \Gst\|_{\op}$ decays as $\BigOhTil{N^{-1/2}}$. 

In many cases, such as our setting of marginally stable systems, it is not possible to select a filter $\pred$ a priori in such a way that $\opnorm{\matDel_{\pred}} \le \BigOhTil{N^{1/2}}$. Instead, in light of Proposition~\ref{prop:gvr_bound}, one would like to choose the filter $\predhat$ which minimizes $\opnorm{\matDel_{\pred}}$, and pay for the magnitude of its associated error $\opnorm{\matDel_{\predhat}}$. Our main result of this section is an oracle inequality, proved in Appendix~\ref{sec:post_hoc_analysis}, which shows that prefiltering the output sequence $(\maty_t)$ essentially accomplishes this goal.
\begin{thm}[PF-LS Oracle Inequality]\label{thm:ph_vr_oracle} 
Let $\Opt_{\mu}$ be as in~\eqref{eq:opt_def}, and define
\begin{align*}
\Ovfit_{\regu}(\delta) &:= \opnorm{\Gst} \cdot \min\left\{N^{1/2},T^{1/2}\sqrt{\log\tfrac{1}{\delta} + \ptil + \log \det(I + \regu^{-2}\matK \matK^{\top})^{1/2}  }\right\} \\
\deff(\Opt,\Lbar,\mu) &:= \ptil + m+ \lil \tfrac{{\Opt}}{\mu} + \Lbar\log_+(\Opt + \tfrac{\sqrt{N}\|\matK\|_{\op}}{\regu^2 })\:. 
\end{align*}
Then for any $\delta \in (0,1)$, then following inequality holds probability with $1 - \delta - \delU$, provided $N$ satisfies the conditions of Lemma~\ref{lem:Uconditioned}:
\begin{align*}
\opnorm{\Gph  - \Gst}  \lesssim \frac{N^{-1/2}(\Opt_{\regu} + \Ovfit_{\regu}(\delta) + \mu)}{\sqrt{N}}  \cdot \sqrt{ T\left(\log \tfrac{1}{\delta} + \deff(\Opt_{\regu}+\Ovfit_{\regu}(\delta),\Lbar,\regu) \right)}~.
\end{align*}
\end{thm}
Here, the term $\Opt_{\regu}$ corresponds to the error achieved by the best filter $\pred$, and the term $\deff$ captures the ``effective dimension'' of the estimation problem, totaling the dimensions of the filter class, inputs $(\matu_t)$, and observations $(\maty_t)$. For the case of linear systems where $\matk_{t} = [\maty_{t-T}^\top|\dots | \maty_{t-TL}^\top]^\top$ and $\Lbar = Lm$, in Appendix~\ref{app:selecting_L} we give an algorithm for selecting the parameter $L$ which admits an oracle inequality, Proposition~\ref{prop:L_oracle}. Moreover, the bound of Theorem~\ref{thm:ph_vr_oracle} depends only logarithmically on $1/\regu$, so $\mu$ may be taken to be very small.

\textbf{Proof Sketch: }In proving Theorem~\ref{thm:ph_vr_oracle}, we first obtain an intermediate but analogous result in terms of the intermediate quantity $\|\matDel_{\predrid}\|_{\op} + \mu \|\predrid\|_{\op}$, which we bound in Appendix~\ref{sec:post_hoc_part_b} by $\Opt_{\mu} + \Ovfit_{\mu}$ using KKT arguments and a variant of Proposition~\ref{prop:gvr_bound}. To prove the the intermediate result, Appendix~\ref{sec:ph_uniform_bound} considers ``slices'' $v^\top \phi$ along directions $v \in \sphere{m}$, and establish uniform bounds with respect to a hierarchy of coverings of $\R^{Lm}$, each with a different scale and granularity, such that the bounds hold for each covering in the hierarchy simultaneously. This lets us tailor the granularity of the covering for each specific filter $\phi$, which (a) yields bounds depending on the data-dependent errors $\|\matDel_{\phi}\|_{\op}$, (b) ensures tighter control on filters $\phi$ with smaller norm, and (c) tolerates logarithmically more error as $\|\phi\|_{\op}$ grows. Specializing the uniform bound to $\predrid$ requires loose control of $\|\predrid\|_{\op}$ (hence the regularization in~\eqref{eq:prefilter}) and the choice of $\regu$ trades off between the magnitude of $\Opt_{\regu}$ and the quantity $\tfrac{\|\matK\|_{\op}}{\kappa \regu }$ inside the logarithmic term. This is somewhat of an artifact of the proof and is unnecessary if $\sigma_{\min}(\matK)$ is bounded from below, for example.


\section{Proof Sketch for Bounding \texorpdfstring{$\Opt_{\mu}$}{Optmu}}
\label{sec:bounding_opt}

Since $\Opt_{\regu} = \min_{\phi} \|\matDel_{\phi}\|_{\op} + \regu\opnorm{\phi}$, it suffices to exhibit some $\phi$ with reasonable operator norm for which $\|\matDel_{\phi}\|_{\op}$ grows as $\sqrt{N}$. To this end, we define the auxiliary signal $\xtil_{n;t}$ and associated observation $\ytil_{n;t}$ via  
\vspace{-1em}
\begin{align*}
\ytil_{n;t} = \Cst \xtil_{n;t}, \quad \xtil_{n;t} := \begin{cases} \Ast^{n - (t - TL)}\matx_{t - TL} & n \ge t - L T \\
\matx_{n} & n \le t - L T
\end{cases}.
\end{align*}
Here, $\xtil_t$ is the state as if the noise and inputs had been ``shut off'' at time $t-TL$. We further define the features $\ktil_t := [\ytil_{t-T;t}^\top \mid \ytil_{t-2T;t}^\top\mid\dots\mid\ytil_{t-TL;t}^\top]^\top\:$ and decompose the error term as {$\matdel_{\phi,t} = \matdel_t - \phi \cdot \matk_t = \Errone_{\phi,t} + \Errtwo_{\phi,t}$}, where 
\begin{align*}
\Errone_{\phi,t} := \tmaty_{t;t} - \phi \cdot \tmatk_t \quad\text{and} \quad \Errtwo_{\phi,t} := (\matdel_t - \tmaty_{t;t}) - \phi \cdot (\matk_t - \ktil_t).
\end{align*}
Here, $\Errone_{\phi,t}$ describes the approximation error $\tmaty_{t;t} - \phi \cdot \tmatk_t$ incurred in predicting $\ytil_{t;t}$ from the shut-off sequence $\ytil_{t-T;t},\ytil_{t-2T;t}\dots,\ytil_{t-TL;t}$, and $\Errtwo_{\phi,t}$ accounts for the additional noise induced by the shut-off sequence. 
In Propositions~\ref{prop:error_bound_stochastic} (resp.~\ref{prop:error_bound_adversarial}), we prove bounds on the total contributions of these two errors under stochastic (resp. adversarial) noise models outlined in Assumption~\ref{asm:noise}. For any fixed $\phi$, the error terms $\Errtwo_{\phi,t}$ do not grow with time, since they only account for the contribution of noise over $TL$ time steps; thus, the contribution of $\Errtwo_{\phi,t}$ to $\opnorm{\matDel_{\phi}}$ grows as $\sqrt{N}$. 

The terms $\Errone_{\phi,t}$, on the other hand, may grow with time because they depend on the state $\matx_{t-TL}$, which can grow in magnitude for marginally stable systems under consistent excitation.  Fortunately, by the Cayley-Hamilton theorem, we can observe that for large enough $L$, there always exists a $\phi$ for which $\Errone_{\phi,t} = 0$ for all $t$. Indeed, let $f(z) = z^d + f_1 z^{d-1} + \dots + f_d$ denote the minimal polynomial of $\Ast^T$, and let $\phi_{f} = -[f_1 I_{m} | f_2 I_m | \dots | f_d I_m | \mathbf{0} ]$. Then, if $L \ge d$, a short computation shows that $\Errone_{\phi_{f},t}  = 0$. Unfortunately, $\phi = \phi_f$ requires $L$ to be at least the degree of the minimal polynomial of $\Ast^T$, which can be as large as $n$ in general. Moreover, the minimal polynomial $f$ may have exponentially large coefficients, which can amplify the effect of noise in $\Errtwo_{\phi_f,t}$ and also affect the contribution of the regularization term $\regu \opnorm{\phi}$. As introduced in Section~\ref{sec:learning_no_stab}, bounded phase rank ensures that there exists a smaller (both in length and in norm) filter $\phi$ than one would obtain by applying the minimal polynomial. Throughout, we shall consider the stochastic case; the adversarial case is similar and deferred to the Appendix.

\textbf{Applying Phase Rank:}  Proposition~\ref{prop:phase_rank_intro} in the introduction gave an explicit bound on $\Opt_{\mu}$ in terms of the largest Jordan block $k$ of $\Ast$ and the $(\alpha,T)$ phase rank of the system. The formal proof of this bound is deferred to Appendix~\ref{sec:poly_approx_main}. Here, we shall instead provide an informal intuition about why phase rank is also a natural quantity. Consider the state transition matrix $\Ast = \left[\begin{smallmatrix} 1 & 0 & \\ 0 & 1 - \epsilon \end{smallmatrix}\right]$, and suppose $\Bst = B_w = I_2$. The first coordinate corresponds to a marginally unstable eigenvalue $1$, and the second corresponds to an eigenvalue which is strictly stable by a small margin $\epsilon$. Taking $L = 1$, we see the filter $\phi = I_m$ corresponding to the polynomial $g(z) = z- 1$ not only exactly cancels the first mode, it also \emph{downweights} the second mode by a factor of $(1 - (1-\epsilon)^T)$, as $\Errone_{\phi,t} =   \Cst[2](1 - (1-\epsilon)^T) \cdot \xtil_{t - T}[2]$.
Moreover, we can express the second coordinate as $\xtil_{t}[2] = \sum_{s = 1}^{t} (1-\epsilon)^{t - s} (\matu_s[2] + \matw_s[2])$. Due to the geometric decay, this sum roughly depends on only the last $\BigOh{1/\epsilon}$ terms in the sum. Therefore, even for adversarial noise, $|\xtil_{t}[2]|$ should be at most $\BigOh{1/\epsilon}$ on average. With this observation in hand, $\|\Errone_{\phi,t}\| \lesssim  (1- (1-\epsilon)^T) \cdot \frac{1}{\epsilon} \lesssim T \text{ on average}$:
this bound depends neither on the time step $t$ \emph{nor} the parameter $\epsilon$.

Now, how does this connect to phase rank? We show in Appendix~\ref{sec:poly_approx_main} that the salient feature of our choice of $\phi$ was that the corresponding polynomial $g(z)$ had a root with the same \emph{phase} as the eigenvalues of $\Ast$, which exactly offset the magnitude of the state along the corresponding eigendirections.
Generalizing to systems with Jordan blocks and multiple phases, we prove that small phase rank lets us construct small-norm filters $\phi$ which yield small  $\|\Errone_{\phi,t}\|$. 

\textbf{Explicit Bounds on $\opnorm{\matDel_{\phi}}$:}
A central technical step in bounding $\Opt_{\mu}$ is obtaining explicit upper bounds on  $\opnorm{\matDel_{\phi}}$. In what follows, we use bold sans-serif notation $\Gsf = (A,B,C,D)$ to denote a dynamical system and denote $\Gsfst := (\Ast,\Bst,\Cst,\Dst)$. We render $\phi = [\Psi_1 \mid\dots \mid\Psi_L]$ and define  $\|\phi\|_{\blockop} := \sum_{\ell =1 }^L \|\Psi_{\ell}\|_{\op}$. Lastly, we define the associated observation matrix 
\begin{align*}
	C_{\phi} := \Cst \Ast^{LT} - \sum_{\ell=1}^L \Psi_{\ell}\Cst\Ast^{(L-\ell)T} \in \R^{m \times n},
\end{align*}
which controls the size of the filtered output sequence $(\maty_t - \phi \cdot \matk_t)$, as well as the associated LTI systems $\Gsf_\phi := \fourtup{\Ast}{\Bst}{C_\phi}{0}$ and $\Fsf_\phi := \fourtup{\Ast}{B_w}{C_\phi}{0}$. We remark that the parameter $L$ in $C_{\phi}$ depends on the length of the filter $\pred$; e.g. for $\pred \in \R^{m \times dm}$, we replace $L$ by $d$. 
To state our bound on $\|\matDel_{\phi}\|_{\op}$, we shall also need to define, for arbitrary dynamical systems $\Gsf = (A,B,C,D)$, the Markov parameter matrix
\begin{align*}
\Markov_{k}(\Gsf) := \begin{bmatrix} D\mid C B \mid \dots \mid CA^{k-1}B\mid C A^{k - 2}B. 
\end{bmatrix}\:.
\end{align*}
For example, we see that $\Markov_T(\Gsf) = \Gst$. We shall also identify dynamical systems by their discrete-time transfer functions; that is, we associate $\Gsf = (A,B,C,D)$ with the real rational transfer function $\Gsf(z) = C(zI - A)^{-1}B + D$, mapping $\C \to \R^{m \times p}$. The notation $\Markov_k(\Gsf)$ and $\Gsf(z)$ allows us to define the following two control-theoretic norms:
\begin{defn}[Control Norms] Consider a dynamical system $\Gsf = (A,B,C,D)$ with $\rho(A) \le 1$. We define the norms $\|\Gsf\|_{\Hinf} := \sup_{z \in \C:|z| = 1} \|\Gsf(z)\|_{\op}$, and $\Mknorm{\infty}{\Gsf} := \lim_{k \to \infty}\Mknorm{k}{\Gsf}$.  We allow these norms to take on the value $\infty$.
\end{defn}
The $\Hinf$-norm admits a variational interpretation. It corresponds  the induced $\ell_2^p\to\ell_2^m$ norm for LTI systems. 
We also remark that $\|\Markov_{\infty}(\Gsf)\|_{\op}$ is equal to the square root of the largest eigenvalue of the so-called ``infinite-horizon Gramian" and is an operator-norm representation of the $\Htwo$-norm in control theory; see~\citet[Chapter 4]{zhou1996robust} for a discussion on both the $\Hinf$ and $\Htwo$ system norms. Note that $\rho(A) < 1$ guarantees that both norms are finite. We are now ready to state our bound on the norm of the error $\|\matDel_{\phi}\|_{\op}$. The adversarial case is similar and is given in Proposition~\ref{prop:error_bound_adversarial}; both are proven in Appendix~\ref{sec:error_calcs}.
\begin{prop}[Stochastic Noise Bound]\label{prop:error_bound_stochastic}
Consider a filter of the form $\phi = [\Psi_1 | \dots | \Psi_{d}] \in \R^{m \times dm}$ for some $1 \le d \le L$, and suppose that $N \ge Td\max\{m,\log(1/\delta)\}$. Then, in the stochastic noise model of Assumption~\ref{asm:noise}, the extended $\phitil := [\phi \mid | \mathbf{0}_{m\times(L-m)d}] \in \R^{L \times m}$ satisfies the following with probability $1-\delta$:
\begin{align*}
    \|\matDel_{\phitil}\|_{\op} &\;\lesssim \sqrt{N}(\|\Markov_{\infty}(\Gsf_\phi)\|_{\op}+\|\Markov_{\infty}(\Fsf_\phi)\|_{\op}) 
    + \sqrt{m+\log(1/\delta)}(\mixnorm{N}{\Gsf_\phi} + \mixnorm{N}{\Fsf_\phi})\\
    &\;+ \sqrt{N} (1+\|\phi\|_{\blockop}) \left(\Mknorm{Td}{\Gsfst} + \Mknorm{Td}{\Fsfst}+ \opnorm{D_z}\right),
\end{align*}
where we define $\mixnorm{N}{\Gsf} := \min\{\sqrt{N}\Mknorm{\infty}{\Gsf}, \|\Gsf\|_{\Hinf}\}$\:.
\end{prop}
The parameter $d$ optimizes for filters that arise from the phase rank of $\Ast$ and allows for sharper bounds (replacing $d$ by $L$) in that setting. 
Note that for this proposition to not be vacuous, we must choose $\phi$ such that the $\Hinf$- and $\Mknorm{\infty}{\cdot}$ norms of the systems $\Gsf_{\phi},\Fsf_{\phi},$ described above are finite. This requires canceling out the effects of the modulus-$1$ eigenvalues of $\Ast$. This is why the phase rank definition is at least as large as the number of such eigenvalues.


\section{Related Work}\label{sec:related_work}

Identifying LTI systems from data has a decades-old history in both the time-series and system identification communities (see \cite{ljung99,verhaegen1993subspace,galrinho_least_nodate} and references therein) with least squares estimation being a central tool for dozens of algorithms, many of them similar in spirit to PF-LS,~\eqref{eq:prefilter},\eqref{eq:post_hoc_def}. One can regard PF-LS as a specific instance of a prefiltered autoregressive model (such as ARX or ARMAX); much work has been done on explicit filtering and debiasing schemes for these types of models~\cite{spinelli_role_2005,ding_two-stage_2013,zheng_revisit_2004,guo_least-squares_1989,zhang_unbiased_2011,wang_brief_2011,galrinho_weighted_2014}. However, analyses of these schemes are often (i) asymptotic, (ii) for strictly stable systems only, or (iii) use a limited noise model. A complementary viewpoint comes from a family of techniques techniques known broadly as \emph{subspace identification} (e.g.~\cite{qin_overview_2006}), which take a singular value decomposition (SVD) of the raw data; following \cite{Samet18} and the classical algorithm of~\cite{kung1978new}, we instead use SVD as a post-processing step via the Ho-Kalman algorithm. It is an interesting direction for future work to explore of SVD-based algorithms can modified to enjoy guarantees for marginally stable systems as well.

There has also been considerable recent work from the machine learning community on non-asymptotic rates for prediction and estimation in LTI systems.  While many have shown that strict stability is not necessary when the full system state can be observed~\citep{simchowitz2018learning,sarkar2018fast,faradonbeh17a}, stability  been central to other works providing guarantees for when only $(\maty_t)$ are observed~\citep{shah12,hardt16,oymak2018stochastic}. Strict stability can be removed at the expense of requiring a number of independent trajectories which grows with desired accuracy~\cite{oymak2018stochastic}, or for online prediction problems in which $\Ast$ is diagonalizable and persistent process noise is minimal~\citep{hazan17,hazan18}. The regret bounds in~\cite{hazan18} depended on the $\ell_1$-norm of the minimal phase polynomial, the inspiration for the phase rank condition in this work.

Beyond linear systems, our prefiltering step bears similarities to the \emph{instrumental variables} technique in used in controls~\citep{viberg1997analysis}, econometrics~\citep{hansen1982generalized} and causal statistics~\citep{angrist1996identification}, which is used more for debiasing than for denoising. More broadly, variance reduction has become an indispensable component of reinforcement learning~\citep{weaver2001optimal,greensmith2004variance,tucker2017rebar,sutton98}, including the theoretical study of tabular Markov Decision Processes~\citep{kakade2018variance,sidford2018variance}.

\section*{Acknowledgements}

We thank Vaishaal Shankar for timely PyWren support. We thank Cyril Zhang and Holden Lee for their generous and thorough exposition of~\cite{hazan18} through personal correspondence. We also thank Samet Oymak for providing the code accompanying~\cite{Samet18}. This work was generously supported in part by ONR awards N00014-17-1-2191, N00014-17-1-2401, and N00014-18-1-2833, the DARPA Assured Autonomy (FA8750-18-C-0101) and Lagrange (W911NF-16-1-0552) programs, and an Amazon AWS AI Research Award. MS is also generously supported by a Berkeley Fellowship, sponsored by the Rose Hill Foundation.

\newpage
\bibliography{main}
\bibliographystyle{plainnat}
\bibpunct{(}{)}{;}{a}{,}{,}
\newpage
\appendix

\tableofcontents
\newpage

\part*{Preface\label{sec:appendix_organization}}
\addcontentsline{toc}{part}{Preface}

The appendix is divided into three main parts. Part~\ref{part:second_results} begins with Appendix~\ref{app:phase_rank_examples}, which provides illustrated examples of the the phase rank condition for various systems. This is then followed by proofs of our secondary results: the proof of Corollary~\ref{main:cor} from Proposition~\ref{prop:opt_to_final_lti} and the proof of the lower bound for ordinary least squares, Theorem~\ref{thm:main_lb}.  Part~\ref{part:PH-VR} contains the supporting material for the results in Section~\ref{sec:pfls_oracle}, as well as generalizations beyond the setting of linear dynamical systems. Our bounds make use of a a general recipe for applying chaining to self-normalized martingale inequalities, described in Appendix~\ref{sec:martingale_chaining}.   
Lastly, Part~\ref{part:LTI} provides the analysis underlying the results in Section~\ref{sec:bounding_opt}. Appendix~\ref{sec:error_calcs} gives the corresponding results bounding $\opnorm{\matDel_{\phi}}$ in terms of various control~theoretic quantities. Appendix~\ref{sec:opt_to_final_lti} gives a detailed proof of Proposition~\ref{prop:opt_to_final_lti}, a specific version of our generalized oracle inequality for linear dynamical systems. The constant $\Mbarstoch$, its analogue $\Mbaradv$ for adversarial noise, as well as the intermediate constants $M_B,M_C,M_D$ and $M_0$, are defined in Appendix~\ref{sec:Mnotation}.
This section also includes Appendix~\ref{app:selecting_L}, which presents and analyzes a procedure for selecting the parameter $L$ in a data-dependent fashion, as well as defining refinements of the constant $\Mbarstoch$. 
Appendices~\ref{sec:poly_approx_main} and \ref{sec:supporting_proofs_poly} give more granular interpretations of our estimation bounds in terms of the phase rank, as well as supporting technical proofs. Finally, Appendix~\ref{app:strong_observability} defines \emph{strong observability} and gives alternative interpretations of our estimation bounds in terms of this quantity.

\newpage

\section*{Notation\label{sec:appendix_notation}}
\addcontentsline{toc}{section}{Notation}
\begin{table}[ht]
\centering
\begin{tabular}{| l | l |}
\hline
General Mathematical Notation \\
\hline
$\lesssim$ denotes inequality up to a universal constant.\\
$\sphere{d}:= \{v \in \R^d:\|v\|_2 = 1\}$\\
$\log_+(x) := \max\{1,\log(x)\}$ \\
$\lil(x) := \log_+(\log_+(x))$ \\
$[n] := \{1,\ldots,n\}$\\
$\|\cdot\|_{\op}$ denotes matrix operator norm\\
$\|\cdot\|_{\fro}$ denotes matrix Frobenius norm\\
$\|\cdot|\|_{2}$ denotes vector two-norm\\
$\sigma_{k}(\cdot)$ denotes the $k$-th largest singular vector\\
$\sigma_{\min}(A)$ denotes $\sigma_{n \wedge m}$ for $A \in \C^{n \times m}$\\
$\cond(A) := \frac{\opnorm{A}}{\sigma_{\min}(A)}$ denotes the condition number\\
\hline
\end{tabular}
\end{table}
\begin{table}[h!]
\centering
\begin{tabular}{| l | l |}
\hline
Semi-Parametric Notation \\
\hline
$N \in \N$ denotes the sample size, \\
$\matu_1,\dots,\matu_N \in \R^p$ denote inputs\\
$\maty_1,\dots,\maty_N \in \R^{m}$ denote observations \\
$\matk_{\None},\dots,\matk_{N} \in \R^{\Lbar}$ denote prefiltering features\\
$\ubar_t := [\matu_t^\top | \matu_{t-1}^\top | \dots | \matu_{t-1}^\top]^\top \in \R^{Tp}$, for length $T \in \N$\\
$\matUbar$ denotes the matrix whose rows are $\ubar_{N_1},\dots,\ubar_{N}$ \\
$\matY$ denotes the matrix whose rows are $\maty_{N_1},\dots,\maty_N$\\
$\matK$ denotes the matrix whose rows are $\matk_{N_1},\dots,\matk_N$\\
$\matdel_t = \maty_t - \Gst \matu_t$ denotes semiparametric error\\
$\matDel$  denotes the matrix whose rows are $\matdel_{N_1},\dots,\matdel_N$ \\
$\{\filtr_t\}$ denotes our filtration,  $(\matdel_t)$ is $\{\filtr_{t-T}\}$ adapted \\
$(\matu_t)$ is $\{\filtr_t\}$-adapted, $\matu_t | \filtr_{t-1} \sim \calN(0,I_p)$\\
$\matk_t \in \R^{\Lbar}$ are $\{\filtr_{t-T}\}$-adapted prefiltering features\\
$\None$ denotes first recorded observation \\
$\Nbar = N - \None$ is effective sample size \\
$\Nmin = cTp\log^4 (Tp)$ for a  sufficiently large $Tp$ \\
$\predrid$ is the filter from~\eqref{eq:prefilter}\\
$\Gph$ is the estimator from~\eqref{eq:post_hoc_def}\\
$\Gls$ is the least squares estimator from~\eqref{eq:GLS_def}\\
$\Gvr(\pred)$ is the fixed-filter estimator from~\eqref{eq:Gvr_def}\\
\hline
\end{tabular}
\end{table}
\begin{table}[h!]
\centering
\begin{tabular}{| l | l |}
\hline
LTI System Notation \\
\hline
$(\matu_t) \subset \R^p$ denote inputs, $(\matx_t) \subset \R^n$ denote states, $(\maty_t) \subset \R^m$ denote observations\\
 $(\matw_t) \subset \R^{d_w}$ denote process noise, $(\matz_t) \subset \R^{d_z}$ denote sensor noise \\
The parameters $(\Ast,\Bst,\Cst,\Dst)$ and $(B_w,D_z)$ are clarified in~\eqref{eq:dynamics}\\
$\Gst:= [\Dst \mid \Cst\Bst \mid \Cst \Ast \Bst \mid \dots \mid \Cst \Ast^{T-2}\Bst] \in \R^{m \times Tp}$\\
\hline
$\matk_t := [\maty_{t-T}^\top \mid \maty_{t-2T}^\top \mid \dots \mid \maty_{t-LT}^\top ]^\top$.\\
$\Gsf = (A,B,C,D)$ is a place holder variable for dynamical systems\\
$\Gsf(z) = D + C(zI - A)^{-1}B$ for $\Gsf = (A,B,C,D)$ \\
$\Gsfst = \Dst + \Cst(zI - A)^{-1}\Bst$ \\
$\Fsfst = \Dst + \Cst(zI - A)^{-1}B_w$.\\
$\Hsfst = \Dst + \Cst(zI - A)^{-1}\matx_1$.\\
$\phi = [\Psi_{1} \mid \dots \mid \Psi_L]$ for $\phi \in \R^{m \times Lm}$.\\
$C_{\phi} = \Cst - \sum_{\ell = 1}^L \Psi_{\ell} \Cst$.\\
$\Gsf_{\phi} = (\Ast,\Bst,C_{\phi},\Dst)$\\
$\Fsf_{\phi} = (\Ast,B_w,C_{\phi},\Dst)$\\
$\Hsf_{\phi} = (\Ast,B_w,C_{\phi},\matx_1)$\\
\hline
$\Markov_n({\Gsf}) = [D \mid CB \mid CAB \mid \dots \mid CA^{n-2}B]$ for $\Gsf = (A,B,C,D)$\\
$\Mknorm{n}{\Gsf} = \opnorm{\Markov_n({\Gsf})}$ for $\Gsf = (A,B,C,D)$\\
$\Mknorm{\infty}{\Gsf} = \lim_{n\to \infty}\opnorm{\Markov_n({\Gsf})}$ \\
$\|\Gsf\|_{\hinf} = \sup_{z \in \C:|z|=1}\|\Gsf(z)\|_{\infty}$ for $\Gsf = (A,B,C,D)$\\
$\Gamma_N(\Gsf) = \min\{\sqrt{N}\Mknorm{\infty}{\Gsf}, \|\Gsf\|_{\hinf} \} $\\
$\|\Gsf\|_{\Htwoop} =  \max_{v \in \sphere{m}}\sqrt{\frac{1}{2\pi}\int_{0}^{2\pi}\|v^\top \Gsf(e^{i \pi \theta})\|_{2}^2}$ for $\Gsf = (A,B,C,D)$\\
$\|\Gsf\|_{\Htwoop} = \Mknorm{\infty}{\Gsf}$ (Lemma~\ref{lem:Htwoopequiv}) \\
\hline
$\algspec(\Ast)$ denotes the set of pairs $(\lambda,k)$ corresponding\\ 
to eigenvalues of $\Ast$ and corresponding Jordan block sizes $k$ \\
$\Ast = S\Jst S^{-1}$ denotes the Jordan decomposition of $\Ast$  \\
(note $\algspec(\Ast) = \algspec(\Jst)$)\\
$M_B,M_C,M_D,M_0,\Mbarstoch,\Mbaradv$ are constants clarified in Appendix~\ref{sec:Mnotation}.\\
\hline
\end{tabular}
\end{table}
\newpage

\renewcommand{\theequation}{\thesection.\arabic{equation}}

\newpage
\part{Proof of Secondary Results\label{part:second_results}}
\section{Examples of Phase Rank\label{app:phase_rank_examples}}

We first recall the definition of phase rank.
\phaserank*
\begin{figure}[ht!]
\centering
\includegraphics[width=0.9\textwidth]{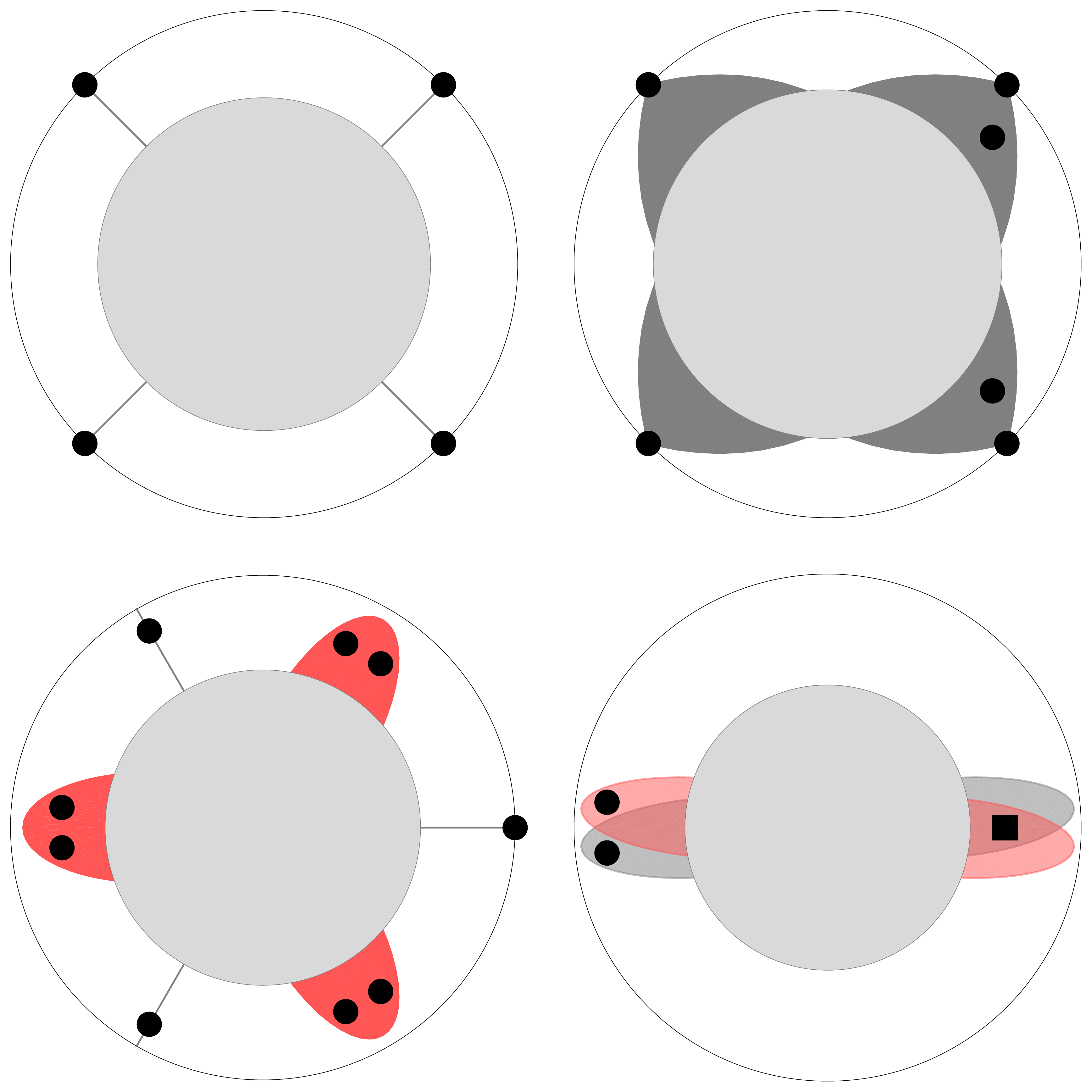}
\caption{Examples (a), (b), (c), and (d) of phase rank, ordered clockwise from the top-left.}
\label{fig:phase_rank}
\end{figure}

Phase rank represents how many phases are required to cover ``large eigenvalues'': eigenvalues with magnitude at least $1-((1+\alpha)T)^{-1}$, where $\alpha \geq 1$ and $T\geq 1$ is an integer. Moreover, the condition $\mutil^T = \mu_{i_j}^T$ means we actually only care about phases mod $2\pi/T$. Graphically, this means every $\mu_i$  contributes $T$ ``bumps'' or ``spokes'' toward covering the eigenvalues of $\Ast$. Figure~\ref{fig:phase_rank} gives four example spectra for which we will calculate the phase rank. Each example shows the regions in the complex disk we are covering by the choices of $\mu_i$ which witness the phase rank conditions. First, we must be clear with graphical notation. The circular gray region\footnote{These regions have been artificially shrunk in Figure~\ref{fig:phase_rank} for the sake of legibility, but the examples are morally correct.} represents the region of ``small'' eigenvalues, and we will assume spectra lie outside of this region without loss of generality. Single eigenvalues are represented by a dot; doubly repeated eigenvalues by a square. Examples (a), (c), (d) have $\alpha = 1$, whereas $(b)$ depicts $\alpha = 2$ to demonstrate the effect of increasing the parameter.
\begin{enumerate}[(a)]
\item In example (a), all eigenvalues lying on the unit circle means we \emph{must} choose each $\mu_i$ to have modulus $1$. As a result, $T$ equally-spaced spokes are added to the covering region for each $\mu_i$. We then see the $(1,2)$ phase rank is $2$ and the $(1,4)$ phase rank is $1$.
\item The $\alpha$ parameter controls the width of the covering regions. We see in example (b) that increasing $\alpha$ to $2$ allows us to cover the two additional eigenvalues, such that the $(2,4)$ phase rank is again $1$.
\item When there are spectra with modulus strictly less than $1$, the $\mu_i$ witnessing the phase rank condition may also have modulus strictly less than one; this also results in the spoke regions transforming into ``bumps''. Example (c) illustrates the $(1,3)$ phase rank being 2, as the $\mu_i$ associated to the red bumps allows covering of the eigenvalues which are not quite at phases $\{\pi/3,-\pi/3, \pi\}$. 
\item Recall that we need to cover a repeated eigenvalue multiple times as to its multiplicity. However, in example (d), we are able to do so while also covering other eigenvalues; the $(1,2)$ phase rank is indeed $2$.
\end{enumerate}
\section{Proof of Corollary~\ref{main:cor}}\label{sec:cor_append}
\subsection{Parts (a) and (b)}\label{sec:cor_ab}
We will choose to describe the constant of Corollary~\ref{main:cor} parts (a) and (b) in terms of phase rank rather than strong observability, though both suffice to prove the Corollary. Now, recall Proposition~\ref{prop:opt_to_final_lti}.:
\optfinal*
Using $\mu \leq \sqrt{N}$ and Proposition~\ref{prop:phase_rank_intro}, we have 
\begin{align*}
N^{-1/2}(\Opt_{\mu} + \mu) \lesssim &\; (\Mbar + \mu N^{-1/2}) \cdot  T^{k-1/2}C_{\alpha,d,k}  + \mu N^{-1/2}\leq \Const\:,
\end{align*}
where $\Const$ depends on the prescribed parameters (and may change at any mention). Furthermore, we see that $\dbar \leq \Const_1 \log_+(N) \leq \Const_1 \log(N)$ for $N\geq 2$. Thus, for (say) $\delta < 1/e$,
\begin{align*}
\opnorm{\Gph - \Gst} \lesssim &\; \left(\Const + \Const' \sqrt{\frac{\log\frac{N}{\delta}}{N}}\right)\sqrt{\frac{\log\frac{N}{\delta}}{N}}\:.
\end{align*}
Finally, taking $\delta$ very small in $N$ (say, $\delta = N^{-10}$), we have that with probability at least $1-\delta -\delU$,
\begin{align*}
\opnorm{\Gph - \Gst} \lesssim &\; \Const\sqrt{\frac{\log N }{N}}\;.
\end{align*}
We note that since $\delU = N^{-\omega(1)}$, $1 - \delta - \delU = 1 - \BigOh{N^{-10}}$ for $N$ sufficiently large.

\subsection{Part (c)}

We begin by formally introducing standard regularity conditions in control theory, observability, controllability, and minimality:
\begin{defn}[Observability, Controllability, Minimality]\label{defn:obs_control} A linear system $(\Ast,\Bst,\Cst,\Dst)$ is said to be \emph{controllable} if 
\begin{align*}
\rank\left(\begin{bmatrix} \Bst & \Ast \Bst & \cdots & \Ast^{n-1} \Bst
\end{bmatrix}\right) = n. 
\end{align*}
A system is said to be \emph{observable} if 
\begin{align*}
\rank\left(\begin{bmatrix} \Cst \\
\Cst \Ast  \\
\vdots \\
\Cst \Ast^{n-1} 
\end{bmatrix}\right) = n. 
\end{align*}
A linear system $(\Ast,\Bst,\Cst,\Dst)$ is said to be \emph{minimal} if it is both \emph{observable and controllable}.
\end{defn}
Even if $(\Ast,\Bst,\Cst,\Dst)$ is not minimal, there always exists an $n' \le n$ and an equivalent system $(\Abar,\Bbar,\Cbar,\Dbar)$ with $\Abar \in \R^{n' \times n'}$ such that $(\Abar,\Bbar,\Cbar,\Dbar)$ is minimal. In this case, the Ho-Kalman algorithm correctly recovers this reduced, minimal system.

Section 4 of in~\cite{Samet18} concerns the robustness of the Ho-Kalman algorithm~\cite{ho1966effective}, which generates state-space matrices $\Ast,\Bst,\Cst,\Dst$ from the matrix of Markov parameters $\Gst$. Oymak and Ozay show how these estimates degrade when the matrix of Markov parameters is replaced by a noisy estimate. To apply these results, we must first define the block Hankel matrix $\Hnk \in \R^{T_1 m\times (T_2+1) p}$ to be the block matrix\footnote{Note that this does not include $\Dst$.} with 
\begin{align*}
\Hnk[i,j]:=\Gst[i+j]=\Cst\Ast^{i+j-2}\Bst\:,
\end{align*}
and we use $\widehat \Hnk$ to denote its analogous estimated version. Furthermore, define $\Hnk^-,\widehat \Hnk^-$ to be the size $(T_1,T_2)$ Hankel matrices created by dropping the last block column of $\Hnk,\widehat \Hnk$ respectively.

At this stage, we would like to note that, in contrast to our main results, the bounds and choice of $T$ depend on the unknown system order $n$. Concretely, we will take $cn \geq T_1\geq n$ and $cn \geq T_2\geq n$ for some constant $c$, with $T_1 + T_2 + 1 = T$. Note that with this choice, under the assumption that $(\Ast, \Bst, \Cst, \Dst)$ is minimal, the Hankel matrix is rank-$n$ and $\sigma_n(\Hnk^-) > 0$ (as noted in Section 4.1 of~\cite{Samet18}). Now, we synthesize their relevant results below, which make use of our choice of $T_1$ and $T_2$.
\begin{prop}[Sections 4.1 and 4.2 of~\cite{Samet18}]\label{prop:Samet} Let $\overline A,\overline B,\overline C,\overline D$ be the state-space realization corresponding to the output of the Ho-Kalman algorithm with input $\Gst$ and let $\widehat A,\widehat B,\widehat C,\widehat D$ be the state-space realization corresponding to the output of Ho-Kalman with input $\Ghat$. Suppose the system $\Ast, \Bst, \Cst, \Dst$ is observable and controllable, and suppose
\begin{align}\label{eq:samet_condition}
\opnorm{\Hnk - \widehat \Hnk}\leq \sigma_{\min}(\Hnk^-)/4\:.
\end{align}
Then, there exists an unitary matrix $S$ such that
\begin{align*}
\max\left\{\|\Bhat - S \Bbar\|_{\fro}, \|\Chat - \Cbar S^*\|_{\fro}\right\} \lesssim &\; n^{3/4}\sqrt{\opnorm{\Gst-\Ghat}}\\
\|\Ahat - S\Abar S^*\|_{\fro} \lesssim &\; \frac{n^{3/4}\sqrt{\opnorm{\Gst-\Ghat}}\opnorm{\Hnk}}{\sigma_{\min}(\Hnk^-)^{3/2}}\:.
\end{align*}
\end{prop}
Since $\|\Dhat - \Dbar\|_{\fro}\leq\|\Gst - \Ghat\|_{\fro}$, for $N$ sufficiently large we can combine Proposition~\ref{prop:Samet} and Section~\ref{sec:cor_ab} to arrive at Corollary~\ref{main:cor}. Note that one can witness~\eqref{eq:samet_condition} using the relation $\opnorm{\Hnk - \widehat \Hnk}\leq \sqrt{\min\{T_1,T_2+1\}}\opnorm{\Gst-\Ghat}$, as stated in Lemma 4.2 of~\cite{Samet18}.


\section{Lower Bound for OLS for Marginally Stable Systems\label{sec:OLS_Lower_bound}}

In this section, we show that OLS cannot recover a linear system for a marginally stable system, even with zero process $\matw_t$ and sensor noise $\matz_t$. Specifically, we have the following 
\begin{thm}\label{thm:main_lb}  Suppose that $\Ast$ has a Jordan block of magnitude $|\lambda| = 1$ and size $k$, and that $(\Ast,\Bst,\Cst,\Dst)$ minimal. Suppose also that the process noise $\matw_t$ and sensor noise $\matz_t$ are identically zero. Then there exists a constant $\Const$ depending on $\Ast,\Bst,\Cst$, as well as the ambient dimension $n$ and $k$ such that, for all $N$ sufficiently large, $\|\Gls - \Gst\|_{\op} \ge \Const N^{k-1}$ with constant probability. 
\end{thm}
The zero-noise assumption is for simplicity, and the above lower bound can also be demonstrated in the presence of Gaussian noise. The important takeaway is that \emph{even without noise}, the ordinary least squares estimator is inconsistent. 

The proof of Theorem~\ref{thm:main_lb} has two components. The first is a lower bound, based on small-ball technicals, which bounds $\|\Gls - \Gst\|_{\op} $ in terms of the Gramian matrices
\begin{align*}
\Grammian_t := \sum_{s = 0}^t \Cst \Ast^{t - s}\Bst \Bst^\top \Ast^{t-s} \Cst^\top\:.
\end{align*}
The following proposition is proved in Section~\ref{prop:small_ball_lb_sec}:
\begin{prop}\label{prop:small_ball_lb} For $N$ sufficiently large, and with no process or sensor noise, there with constant probability
\begin{align*}
\|\Gls - \Gst\|_{\op} \gtrsim \frac{1}{N}\sqrt{\left\|\sum_{t=\None}^N \Grammian_{t - T} \right\|_{\op}},
\end{align*}
where we define the Gramian matrix
\end{prop}
The second component, which completes the proof of Theorem~\ref{thm:main_lb}, is a lower bound estimate on the operator norm of the sum of the Gramian matrices, proved in Section~\ref{sec:Gram_lb}:
\begin{lem}\label{lem:Gram_Lb} Suppose that $\Ast$ has a Jordan block of magnitude $|\lambda| = 1$ and size $k$, and that $(\Ast,\Bst,\Cst,\Dst)$ minimal. Then there exists a constant $\Const$ depending on $\Ast,\Bst,\Cst$, as well as the ambient dimension $n$ and $k$ such that, for all $N \ge \max\{20n^2,2k\}$, 
\begin{align*}
\sqrt{\opnorm{\sum_{t = 0}^{N} \Grammian_{t}}} \ge \Const N^{k}.
\end{align*}
\end{lem}

\subsection{Proof of Proposition~\ref{prop:small_ball_lb}\label{prop:small_ball_lb_sec}}
 When $N \ge \underN$, \citet[Lemma C.2]{Samet18} yields that $\sigma_{\max}(\Ubar)^2 \le 2N$ with probability $1 - \delU$ for $\delU = \frac{1}{N^{\omega(1)}}$, and on this event,
\begin{align*}
\|\Gls - \Gst\|_{\op} = \|\Ubar^{\dagger}\matDel\|_{\op} \ge \sigma_{\max}(\Ubar)^{-2}\|\Ubar^{\dagger}\matDel\|_{\op} \ge \frac{\|\Ubar^{\top}\matDel\|_{\op}}{2N} \ge \frac{\|\matU^{\top}\matDel\|_{\op}}{2N},
\end{align*}
where we recall that with our notation, $\matU$ is the matrix corresponding to the first $p$ columns of $\Ubar$. We now lower bound bound $\|\Ubar^{\top}\matDel\|_{\op}$. Observe that when $\matw_t$ and $\matz_t = 0$, we can write 
\begin{align*}
\matdel_t = \sum_{s = 1}^{t - T}\Cst \Ast^{t - s} \Bst \matu_s.
\end{align*}
Observe that we can represent this quantity as a quadratic form in a long vector $\matu_{1:N} \in \R^{Np}$. We now invoke the follow lemma
\begin{lem}[Small Ball for Gaussian Quadratic Forms]\label{lem:small_ball} There exists a universal constant $c$ such that, for any matrix $M \in \R^{d \times d}$ and vector $\matz \sim \calN(0,I_d)$, 
\begin{align*}
\Pr[(\matz^\top M \matz)^2 \ge \frac{1}{2}\Exp(\matz^\top M \matz)^2 ] \ge c.
\end{align*} 
\end{lem}
We now compute that
\begin{align*}
v^\top\Ubar^{\top}\matDel  w = \sum_{t=\None}^N v^\top \matu_t\matdel_t^\top  w = \sum_{t=\None}^N\sum_{s = 1}^{t - T} \langle v, \matu_t \rangle \langle (\Cst \Ast^{t - s} \Bst)^\top w, \matu_s \rangle
\end{align*}
And thus, for $v \in \sphere{p}$, 
\begin{align*}
\Exp(v^\top\Ubar^{\top}\matDel  w)^2 &= \Exp \sum_{t=\None}^N\sum_{s = 1}^{t - T} \langle v, \matu_t \rangle^2 \langle (\Cst \Ast^{t - s} \Bst)^\top w, \matu_s \rangle^2 + \Exp [\text{mean zero terms}]\\
&=  \sum_{t=\None}^N\sum_{s = 1}^{t - T} \|\Cst \Ast^{t - s} \Bst)^\top w\|_2^2 \\
&= w^{\top}\left(\sum_{t=\None}^N \sum_{t=1}^{t-T} \Cst \Ast^{t - s}\Bst \Bst^\top \Ast^{t-s} \Cst^\top\right )w = w^{\top}\left(\sum_{t=\None}^N \Grammian_{t-T} \right )w.
\end{align*}
Thus, optimizing for $w$ to be a lead eigenvector of $\sum_{t=\None}^N \Grammian_{t-T}$, we see that 
\begin{align*}
\Pr\left[ \|\Ubar^{\top}\matDel\|_{\op} \ge \frac{1}{\sqrt{2}}\sqrt{\opnorm{\sum_{t=\None}^N \Grammian_{t-T}}} \right] \ge \Pr\left[ v\Ubar^{\top}\matDel w \ge \frac{1}{2}\opnorm{\sum_{t=\None}^N \Grammian_{t-T}} \right] \ge c.
\end{align*}
for a universal constant $c$. Hence, we conclude that, with probability at least $c - \delU$, 
\begin{align*}
\|\Gls - \Gst\|_{\op} \ge  \frac{\sqrt{ \|\Ubar^{\top}\matDel\|_{\op}^2 }}{2N} \ge \frac{\sqrt{\opnorm{\sum_{t=\None}^N \Grammian_{t-T}}}}{2\sqrt{2} N}.
\end{align*}

\subsubsection{Proof of Lemma~\ref{lem:small_ball}}
 We may assume that $M$ is symmetric, since replacing $M$ by $\frac{1}{2}(M + M^\top)$ does not affect the result. Since $\matz$ has a unitary invariant distribution, we may also assume that $M$ is diagonal whose vector is a diagonal $\mata\in \R^d$. Note that $\|\mata\|_2 = \frac{1}{2}\|M + M^\top\|_{\fro}$. Then, $\matz^\top M \matz = \sum_{i=1}^n \mata_i \matz_i^2$. Let $Z = (\matz^\top M \matz )^2$. The Paley-Zygmund inequality states that
\begin{align*}
\Pr[ Z \ge \theta \Exp[Z]] \ge (1 - \theta^2) \frac{\Exp[Z]^2}{\Exp[Z^2]}. 
\end{align*}
First, we have that 
\begin{align*}
\Exp Z = \Exp (\sum_{i=1}^n \mata_i^2 \matz_i^4) + \Exp (\sum_{i \ne j =1}^n \mata_i\mata_j \matz_i^2 \matz_j^2) = 3\sum_{i=1}^n \mata_i^2  + \sum_{i \ne j = 1}^n \mata_i\mata_j = 2\|\mata\|_2^2 + (\sum_{i =1}^n \mata_i)^2.
\end{align*}
On the other hand, a standard $\chi^2$-concentration \citep[Lemma 1]{laurent2000adaptive} inequality (adapted for nonnegative coefficients)  yields:
\begin{align*}
\Pr[  |\matz^\top M \matz - \sum_{i} \mata_i| \ge 2(\|\mata\|_2 \sqrt{t}  + \|\mata\|_{\infty} t] \le 2e^{-t},
\end{align*}
which crudely implies $\Pr[  Z^2 \gtrsim (\sum_{i} \mata_i)^4  + \|\mata\|_2^2 t^4] \le 2e^{-t}$. Integrating, we find that $\Exp[Z^2] \lesssim (\sum_{i} \mata_i)^4+  \|\mata\|_2^4 $. Thus, Paley-Zygmund and the lower bound on $\Exp Z$ imply for a universal constant $c$ that
\begin{align*}
 \Pr[ Z \ge \theta \Exp[Z]] \ge (1 - \theta^2) \frac{\Exp[Z]^2}{\Exp[Z^2]} \ge (1- \theta^2) \frac{(2\|\mata\|_2^2)^2}{\Exp[Z]^2} \ge c(1 - \theta^2) 
\end{align*}
for a universal constant $c$. Taking $\theta = 1/2$ concludes. 
\subsection{Proof of Lemma~\ref{lem:Gram_Lb}\label{sec:Gram_lb}}

If $\Ast \in \R^{n \times n}$, then $\rank(\sum_{t = 0}^{N} \Grammian_{t}) \le n$, and thus $\sqrt{\opnorm{\sum_{t = 0}^{N} \Grammian_{t}}} \ge \sqrt{ \frac{1}{n}\tr(\sum_{t = 0}^{N} \Grammian_{t})} $. We now turn to lower bounding $\tr(\sum_{t = 0}^{N} \Grammian_{t})$. We can bound

\begin{align*}
\tr(\sum_{t = 0}^{N} \Grammian_{t}) &= \tr( \sum_{t=1}^N \sum_{s = 0}^{t-1}  \Cst \Ast^{s} \Bst \Bst^\top (\Ast^s)^{\top} \Cst ) \\
&= \sum_{t=0}^{N-1} \sum_{s = 0}^{t}  \|\Cst \Ast^{s} \Bst\|_{\fro}^2 \\
&\ge \sum_{j = 0}^{\floor{(N-1)/n}-1} \sum_{p = 0}^{n-1} \sum_{s = 0}^{nj + p} \|\Cst \Ast^{s} \Bst\|_{\fro}^2 \\
&\ge \sum_{j = 0}^{\floor{(N-1)/n}-1} \sum_{p = 0}^{n-1} \sum_{s = 0}^{nj} \|\Cst \Ast^{s + p} \Bst\|_{\fro}^2 \\
&\ge \sum_{j = 0}^{\floor{(N-1)/n}-1} \sum_{p = 0}^{n-1} \sum_{\ell=0}^{j-1} \sum_{q = 0}^{n-1}\|\Cst \Ast^{\ell n + q + p} \Bst\|_{\fro}^2 \\
&\ge \sum_{j = 0}^{\floor{(N-1)/n}-1}\sum_{\ell=0}^{j-1}  \left(\sum_{p = 0}^{n-1} \sum_{q = 0}^{n-1}\|\Cst \Ast^{\ell n + q + p} \Bst\|_{\fro}^2\right)
\end{align*}
Next, we define the controllability matrix
\begin{align*}
\controln := 
\rank\left(\begin{bmatrix} \Bst & \Ast \Bst & \cdots & \Ast^{n-1} \Bst
\end{bmatrix}\right) = n. 
\end{align*}
and the observability matrix
\begin{align*}
\observen := \rank\left(\begin{bmatrix} \Cst \\
\Cst \Ast \\
\vdots \\
\Cst \Ast^{n-1} 
\end{bmatrix}\right),
\end{align*}
By observability and controllability, $M := \sigma_{n}(\observen) \cdot \sigma_{n}(\controln) > 0$. Moreover, we observe that 
\begin{align*}
\left(\sum_{p = 0}^{n-1} \sum_{q = 0}^{n-1}\|\Cst \Ast^{\ell n + q + p} \Bst\|_{\fro}^2\right) = \|\observen\Ast^{kn}\controln\|_{\fro}^2 \ge M^2\|\Ast^{\ell n}\|_{\fro}^2.
\end{align*}
 Therefore, 
\begin{align*}
\tr(\sum_{t = 0}^{N} \Grammian_{t}) \ge M^2\sum_{j = 0}^{\floor{(N-1)/n}-1}\sum_{\ell =0}^{j-1} \|\Ast^{\ell n}\|_{\fro}^2.
\end{align*}
Next, we lower bound $\|\Ast^{\ell n}\|_{\fro}^2$. Let $\Ast = S^{-1}\Jst S$. Then, 
\begin{align*}
\|\Ast^{kn}\|_{\fro} \ge \|\Jst\|_{\fro}\sigma_{\min}(S^{-1})\sigma_{\min}(S) = \cond(S)\|\Jst^{\ell n}\|_{\fro}. 
\end{align*}
Lastly, we note that if $\Jst$ has a Jordan block of eigenvalue $|\lambda| = 1$ and multiplicity, $\Jst^{\ell n}$ then it contains an entry of magnitude $\binom{\ell n}{k-1}$. Thus, we can crudely lower bound 
\begin{align*}
\tr(\sum_{t = 0}^{N} \Grammian_{t}) &\ge M^2\cond(S)^{-2}\sum_{j = 0}^{\floor{(N-1)/n}-1}\sum_{\ell=0}^{j-1} (\binom{\ell n}{k})^2\\
&\gtrsim M^2\cond(S)^{-2}\frac{N^2}{n^2} \binom{N/2}{k-1}^2, 
\end{align*}
provided that $N \ge \max\{20 n^2, 2k\}$.  Hence,
\begin{align*}
\sqrt{\opnorm{\sum_{t = 0}^{N} \Grammian_{t}}} \gtrsim \frac{M\cond(S)N}{n^{3/2}}\binom{N/2}{k-1} \ge \Const N^{k}, 
\end{align*}
where the last line is by Stirling's approximation (where $\Const$ depends on $k$).

\newpage
\part{General Bound for Prefiltered Least Squares\label{part:PH-VR}}
\section{General Statement and Analysis PF-LS \label{sec:post_hoc_analysis}}
In this section, we a more explicit version of Theorem~\ref{thm:ph_vr_oracle}, which decouples an oracle bound with a bound on $\opnorm{\predrid}$. For this slightly refined bound, we shall use a parameter $\kappa > 0$ as in the bound Proposition~\ref{prop:gvr_bound}) for a fixed filter, and define a filter-specific effective dimension 
\begin{align*}
\defftil(\phi;\Leff,\regu,\kappa) := \ptil + m+ \lil \tfrac{\opnorm{\matDel_{\phi}}}{\kappa\sqrt{N}} + \Leff \log_+(\regu \opnorm{\phi} + \frac{\regu^{-1}\|\matK\|_{\op}}{\kappa + N^{-1/2}\opnorm{\matDel_{\phi}}}).
\end{align*}
In this setting, our main result is as follows:
\begin{thm}[General Statement of Theorem~\ref{thm:ph_vr_oracle}]\label{thm:ph_vr_general} Let $\regu, \kappa > 0$ be fixed. Then
\begin{enumerate}
	\item[(a)] Let $\predclass \subset \R^{\Lbar \times m}$ be a set of filters such that each slice $\predclass_v := \{\phi^\top v: \phi \in \predclass\}$ is contained in a subspace of dimension $\Leff \le \Lbar$.  Then, with probability $1 - \delta$, the following holds on $\eventU$:
	\begin{align*}
	 \forall \pred \in \predclass,\quad \opnorm{\Gst - \Gvr(\pred)} \lesssim \frac{N^{-1/2}\opnorm{\matDel_{\phi}} + \kappa }{\sqrt{N}} \cdot T^{1/2} \sqrt{\log\tfrac{1}{\delta} + \defftil(\phi;\Leff,\regu,\kappa) }.
	\end{align*}
	\item[(b)] On $\eventU$, the filter $\predrid \in \R^{\Lbar \times m}$ satisfies $\max\{\regu \opnorm{\predrid}, \opnorm{\matDel_{\predrid}}\} \le \Opt_{\regu} + \Ovfit_{\regu}(\delta)$. Thus, taking $\predclass = \R^{\Lbar \times m}$, part (a) implies that on $\eventU$, with probability $1 - \delta$,
	\begin{align*}
	\opnorm{\Gst - \Gvr(\pred)} &\lesssim \frac{N^{-1/2}(\Opt_{\regu} + \Ovfit_{\regu}(\delta)) + \kappa}{\sqrt{N}}  \\
	&\quad\cdot \sqrt{ T\left(\log \tfrac{1}{\delta} + \defftil(\Opt_{\regu} + \Ovfit_{\regu}(\delta),\Lbar,\regu,\kappa) \right)}~.
	\end{align*}
\end{enumerate}
\end{thm}
Theorem~\ref{thm:ph_vr_oracle} is a direct consequence of combining parts (a) and (b) of the above theorem, setting $\kappa = \mu/\sqrt{N}$, and applying some routine simplifications to bound  $\defftil(\Opt_{\regu} + \Ovfit_{\regu}(\delta),\Lbar,\regu,\kappa) \le (\Opt_{\regu} + \Ovfit_{\regu}(\delta),\Lbar,\mu)$.

\subsection{Proof of Uniform Bound: Theorem~\ref{thm:ph_vr_general}, Part (a)\label{sec:ph_uniform_bound}}
	Fix $\regu,\kappa > 0$. Recall that
	\begin{align*}
	\Gvr(\pred) := \argmin_{G \in \R^{Tp \times m}}\sum_{t=\None}^N \|(\maty_t - \phi \cdot \matk_t) -  G \ubar_t\|_2,
	\end{align*}
	and thus
	\begin{align*}
	\Gvr(\pred) =  (\matY - \matK \phi^\top)\Ubar^\top (\Ubar^\top \Ubar)^{-1} = \matDel_{\phi}\Ubar^{\dagger}.
	\end{align*}
	Thus, $\|\Gst-\Gvr(\phi)\|_{\op} = \opnorm{\matDel_{\phi} \Ubar^{\dagger} }$. 

	\textbf{Step 0: Reduction To $m=1$ and $\predclass = \R^{\Leff}:$}  We claim that prove the bound in the setting where $m = 1$ and $\predclass = \R^{\Leff}$. Indeed, given $\phi \in \predclass = \R^{m \times \Lbar}$ and a direction $v \in \sphere{m}$, let $v^\top \phi  \in \R^{1 \times \Lbar}$ denote the filter along $v$. Lastly, define 
	\begin{align*}
	\defftil(\phi;\regu,\kappa,\Leff,v) := \ptil + m+ \lil \tfrac{\twonorm{v^\top \matDel_{\phi} }}{\kappa\sqrt{N}} + \Leff \log_+(\regu \twonorm{v^\top \phi } + \frac{\regu^{-1}\|\matK\|_{\op}}{\kappa + N^{-1/2}\twonorm{v^\top \matDel_{\phi}}}),
	\end{align*}
	Suppose we can prove that for any fixed $v \in \sphere{m}$, it holds with probability $1-\delta$ on $\eventU$ that
	\begin{align}\label{eq:fixed_v_WTS}
	\|v^\top \Delpast_{\phi} \matcir{u}^{\dagger}\|_{2} &\lesssim \tfrac{ \twonorm{v^\top \matDel_{\pred} }N^{-1/2} + \kappa}{\sqrt{N}}\cdot T^{1/2} \cdot\sqrt{\log\tfrac{1}{\delta} + \defftil(\phi;\regu,\kappa,v)}).
	\end{align}
	Then, noting that $\twonorm{v^\top \matDel_{\pred} } \le \opnorm{\matDel_{\pred}}$ and $\defftil(\phi;\regu,\kappa,\Leff,v) \le \defftil(\regu,\kappa,\Leff)$, the theorem follows from a standard covering argument. Moreover, the bound~\eqref{eq:fixed_v_WTS} is equivalent to the setting where we observe $\maty_t^\top v$ and make predictions $v^\top \phi \cdot \matk_t$, where $\phi \in \predclass$, or equivalently, $(v^\top \phi ) \in \predclass_v := \{v^\top \phi : \phi \in \predclass\}$. By assumption, there exists an $\Leff$-dimensional subspace of $\R^{\Lbar}$ containing $\predclass_v$. 
	By projecting $\matK$ onto this subspace and applying a unitary change of coordinates, we may assume without loss of generality that $\predclass_v \subset \R^{\Leff}$. Finally, since we are proving a uniform bound, it is only stronger to prove the bound for all filter $(v^\top\phi) \in \R^{\Leff}$.

	Thus, to simplify notation, we assume $m=1$ and  $\predclass \subset \R^{\Leff}$, and will let $\defftil(\dots)$ correspond to these one-dimensional predictions.

	\textbf{Step 1: Pointwise bound} Let's start of with a pointwise bound for a fixed filter. By Proposition~\ref{prop:gvr_bound}, we have with probability at least $1-\delta$,
	  \begin{align}\label{eq:fixed_phi_eq}
   \twonorm{\Delpast_{\phi}\matcir{u}^{\dagger}} \le C\frac{(\twonorm{\matDel_{\phi}}N^{-1/2} + \kappa)T^{1/2}}{\sqrt{N}}\sqrt{\ptil + \log \tfrac{1}{\delta} + \lil (\tfrac{\opnorm{\matDel_{\phi}}}{\kappa N^{1/2}})},
    \end{align}
    for a sufficiently large constant $C$. \\

	\textbf{Step 2: Uniform bound over nets.} For $j \ge 1$, let $c_j = e^{e^{j}}$.  Let $\calT_{j}$ denote a $1/c_j$ net of the set $c_j \ball{\R^{\Leff}}/\regu$ in the norm $\regu\|\cdot\|_2$. We shall argue a uniform concentration bound over the nets $\calT_{j}$, and then uses the nets to approximate bounds for each $\phi \in \R^{\Leff}$. Precisely, define the bound 
	\begin{align*}
	 \Boundb(\twonorm{\matDel_{\pred}},j)  := \frac{(N^{-1/2}\twonorm{\matDel_{\pred}} + \kappa)T^{1/2}}{\sqrt{N}}\sqrt{\ptil + \Leff \log (1+2c_j^2) + \log(\tfrac{2j^2}{\delta}) + \lil \frac{\twonorm{\matDel_{\pred}}}{\kappa\sqrt{N}}},
	\end{align*}
	and, for $C$ as in~\eqref{eq:fixed_phi_eq}, define the event
	\begin{align}\label{eq:unif_net_guarantee}
	\Egood := \{\forall j, \forall \pred \in \calT_j: \twonorm{\Delpast_{\phi}\Ubar^{\dagger}} \le C \Boundb(\twonorm{\matDel_{\pred}},j)\},
	\end{align}
	\sloppypar By a standard volumetric argument (see e.g.~\citet[Corollary 4.2.13]{vershynin2018high}), ${\log |\calT_{j}| \le  \Leff \log (1+2c_j^2)}$, and therefore by applying~\eqref{eq:fixed_phi_eq} with $\delta_j := \frac{\delta}{2|\calT_j|j^2}$ for each $\phi \in \calT_{j}$ uniformly, we see that $\Egood$ holds with probability at least $1 - \frac{\pi^2}{12}\delta \geq 1 - \delta$. It remains to show that on $\Egood$,~\eqref{eq:fixed_v_WTS} holds as well. 

	\textbf{Step 3: Uniform bound over $\phi \in \R^{\Leff}$.} We shall now need a covering argument to translate~\eqref{eq:unif_net_guarantee} into a uniform guarantee over $\R^{\Leff}$: To this end, we introduce the shorthand $\Rk:= \opnorm{\matK / \regu}$,establish the following covering claim:
	\begin{claim} For any $\phi \in \R^{\Leff}$ with $\regu\|\phi\|_{2} \le c_j$, there exists $\phitil \in \calT_j$ satisfying $\twonorm{\Delpast_{\phi} - \Delpast_{\phitil}} \le \Rk/c_j$.
	\end{claim}
	\begin{proof}
	By definition, $\twonorm{\phi} \in \frac{c_j}{\regu} \ball{\R^{\Leff}}$. Since $\calT_j$ is a $1/c_j$ net of $\frac{c_j}{\regu} \ball{\R^{\Leff}}$ in the norm $\frac{1}{\regu} \|\cdot\|_2$, there exists a $\phitil \in \calT_j$ with $\|\phi - \phitil\|_2 \le \frac{1}{c_j\regu}$. Thus,  $\twonorm{\Delpast_{\phi} - \Delpast_{\phitil}} = \twonorm{ \matK (\phi - \phitil) } \le \opnorm{\matK}  \frac{1}{c_j\regu} = \Rk/c_j$.
	\end{proof}

	As a consequence, we obtain a uniform bound as follows. Consider any $\phi$, and any $j$ for which $\regu \|\phi\|_2 \le c_j$. For this $j$, let $\phitil$ denote the filter guaranteed by the above claim. Then on $\Egood$,
	\begin{align*}
	\twonorm{\Delpast_{\phi}\matcir{u}^{\dagger}}  &\le \twonorm{\Delpast_{\phitil}\matcir{u}^{\dagger}} +  \frac{\Rk}{c_j\sigma_{\min}(\matcir{u})}~\le C \Boundb(\twonorm{\Delpast_{\phitil}},j) + \frac{\Rk}{c_j\sigma_{\min}(\matcir{u})} \\
	&\le C\Boundb(\twonorm{\Delpast_{\phi}} + \frac{\Rk}{c_j},j) + \frac{\Rk}{c_j\sigma_{\min}(\matcir{u})}.
	\end{align*}
	Thus we conclude that on $\Egood \cap \eventU$, the following holds simultaneously for all $\phi \in \predclass$
	\begin{align}
	\twonorm{\Delpast_{\phi}\matcir{u}^{\dagger}} \lesssim&\;\inf_{j \in \N: \regu \|\phi \|_2 \le c_j} \Boundb\left(\twonorm{\Delpast_{\phi}} + \frac{\Rk}{c_j},j\right) + \frac{\Rk}{c_j\sqrt{N}} \notag \\
	\lesssim &\; \inf_{j \in \N: \regu \|\phi \|_2 \le c_j} \Boundb\left(\twonorm{\Delpast_{\phi}} + \frac{\sqrt{N} \Rk}{c_j},j\right) \label{eq:unif_class_guarantee}, 
	\end{align}
	where we note that $\Boundb(A + B,j) + B/\sqrt{N} \le \Boundb(A + B + \sqrt{N} B ,j) \lesssim \Boundb(A+B\sqrt{N},j)$. 

	\textbf{Step 4: Tuning the bound.} Lastly, we tune the bound in~\eqref{eq:unif_class_guarantee} to each specific $\pred \in \predclass$. Note that the infimum over $j$ in~\eqref{eq:unif_class_guarantee} is taken for each $\phi \in \predclass$, and thus we can tailor $j$ to our filter $\pred$. The remainder of the proof is simply choosing $j$ appropriately. Define the \emph{random} index
	\begin{align*}
	\jnot&:= \inf\left\{j\in \N:  \max\{\regu\|\pred \|_{2}, \tfrac{\Rk}{(\kappa + N^{-1/2}\twonorm{\Delpast_{\pred}})} \}\le c_j\right\}. 
	\end{align*}
	Then $\twonorm{\Delpast_{\phi}\matcir{u}^{\dagger}}$ is bounded by (up to constants) by
	\begin{align*}
	&\Boundb(\twonorm{\Delpast_{\pred}}+ \frac{\Rk \sqrt{N}}{c_\jnot},\jnot) \\
	&\overset{(i)}{\lesssim}   \tfrac{(N^{-1/2}\twonorm{\Delpast_{\pred}}+ \tfrac{\Rk}{c_\jnot} + \kappa)}{\sqrt{N}}\cdot\sqrt{T}\cdot\sqrt{\ptil +   e^{\jnot} \Leff  + \log \tfrac{1}{\delta} + \lil (\tfrac{N^{-1/2}\twonorm{\Delpast_{\pred}}+ \Rk/c_\jnot}{\kappa}})\nonumber\\
	 &\overset{(ii)}{\lesssim}    \tfrac{ N^{-1/2}\twonorm{\Delpast_{\pred}}+ \kappa}{\sqrt{N}}\cdot\sqrt{T}\cdot\sqrt{\ptil +  e^{\jnot} \Leff  + \log \tfrac{1}{\delta} + \lil (\tfrac{2N^{-1/2}\twonorm{\Delpast_{\pred}}}{\kappa} + 1})\nonumber\\
	 &\overset{(iii)}{\lesssim}    \tfrac{ N^{-1/2}\twonorm{\Delpast_{\pred}}+ \kappa}{\sqrt{N}}\cdot\sqrt{T}\cdot\sqrt{\ptil +   e^{\jnot} \Leff  + \log \tfrac{1}{\delta} + \lil \tfrac{\twonorm{\Delpast_{\pred}}}{\kappa N^{1/2}}}\:,
	\end{align*}
	where $(i)$ follows by combining the definition of $\Boundb$ in~\eqref{eq:unif_net_guarantee} with the chain of inequalities
	\begin{align*}
	1 + \Leff \log (1+2c_\jnot^2) + \log(\tfrac{2\jnot^2}{\delta}) \lesssim &\;\Leff \log(c_\jnot) + \log(\jnot) + \log(1/\delta) \\
	\lesssim &\; \Leff e^\jnot + \log(1/\delta),
	\end{align*}
	which uses the facts that $\Leff \ge 1$ and $c_\jnot = e^{e^\jnot} \ge \jnot$. Furthermore, $(ii)$ uses $\frac{\Rk}{c_\jnot} \le \sqrt{N}\twonorm{\Delpast_{\pred}} + \kappa$ ands $(iii)$ uses $\lil(2 x + 1) \le \lil(x)+1$. Lastly, by the definition of $\jnot$, we see that 
	\begin{align*}
	e^{\jnot} \lesssim \log_+( \regu\|\pred \|_{2} +  \tfrac{\Rk}{(\kappa + N^{-1/2}\twonorm{\Delpast_{\pred}})} )  = \log_+( \regu\|\pred \|_{2} + \tfrac{\opnorm{\matK}}{\regu(\kappa + N^{-1/2}\twonorm{\Delpast_{\pred}})}),
	\end{align*}
	which concludes the proof. 

\subsection{Proof of \texorpdfstring{$\Optil_{\regu}$}{Opt} bound, Theorem~\ref{thm:ph_vr_general}, Part (b)\label{sec:post_hoc_part_b}}
Let $(\matk_t) \subset \R^{\Lbar}$ be a sequence. For sequences $(\mata_t) \subset \R^m$  and regularizer $\regu > 0$, and define the ridge function:
\begin{align*}
\ridgepred(\matA,\regu) := \arg\min_{\pred \in \R^{m \times \Lbar}} \sum_{t=N_1}^N \|\pred \cdot \matk_t - \mata_t \|_2^2 + \regu^2 \fronorm{\pred}^2,
\end{align*}
where $\matA \in \R^{\Nbar\times m}$ is the matrix whose rows are $(\mata_t)$. Note then that
\begin{align*}
\predrid := \ridgepred(\matY,\regu).
\end{align*}
To prove the desired bound, we begin by stating some properties of $\ridgepred$, which are derived from the KKT conditions in Section~\ref{sec:LS_char}:
\begin{lem}[Properties of the Ridge Estimator]\label{lem:LS_char} 
\begin{enumerate}[(a)]
	\item $\ridgepred(\matA,\regu)^\top = (\matK^\top \matK + \regu^2 I )^{-1}\matK^\top \matA$. In particular, $\ridgepred$ is linear in its first argument.
	\item $\ridgepred(\matA,\regu)\in \argmin_{\phi \in \R^{m \times \Lbar}}\|\vcat{\matK \phi^\top - \matA , \regu \phi^\top }\|_{\op}$, where  $\vcat{A,B} := [A^\top | B^\top]^\top$.
\end{enumerate}
\end{lem}
The second point of Lemma~\ref{lem:LS_char} implies that $\ridgepred(\matA,\regu)$ is nearly optimal in the operator norm, up to the regularization. In particular, recall 
\begin{align*}
\Opt_{\mu} := \min_{\phi \in \R^{m \times \Lbar}}\|\matDel_{\phi} \|_{\op} + \regu \| \phi\|_{\op}.
\end{align*}
Then it follows from Lemma~\ref{lem:LS_char} that 
\begin{align*}
\max\{\opnorm{\matDel_{\ridgepred(\matDel,\regu)}}, \regu \opnorm{\ridgepred(\matDel,\regu) }\}  &\leq \min_{\phi \in \R^{m \times \Lbar}}\|\vcat{\matK \phi^\top - \matDel , \regu \phi^\top }\|_{\op} \\
&= \min_{\phi \in \R^{m \times \Lbar}}\|\vcat{\matDel_{\phi} , \regu \phi^\top }\|_{\op} \le \Opt_{\regu}. \numberthis \label{eq:ridge_error}
\end{align*}
Unfortunately, our algorithm does not have direct access to $\matDel$, only $\matY$. The following lemma accounts for this discrepancy.
\begin{lem}\label{lem:LS_error_terms} We have that
\begin{enumerate}[(a)]
	\centering
	\item \quad$\opnorm{\Delpast_{\predrid}} \le \Opt_{\regu} + \|\matK \:\ridgepred(\matUbar\Gst^\top,\regu )^\top  \|_{\op}$
	\item\quad$\regu \opnorm{\predrid} \le \Opt_{\regu} + \regu\| \ridgepred(\matUbar \Gst^\top,\regu )  \|_{\op}$\:.
\end{enumerate}

\end{lem}
Lastly, we require a bound on $ \|\matK \:\ridgepred(\matUbar\Gst^\top,\regu )^\top  \|_{\op}$ and $ \regu\| \ridgepred(\matUbar \Gst^\top,\regu )^\top  \|_{\op}$. This is accomplished by the following lemma:
\begin{lem}\label{lem:ridge_on_u} With probability at least $1 - \delta$ on $\eventU$, $ \|\matK \:\ridgepred(\matUbar\Gst^\top,\regu )^\top  \|_{\op}$ and $ \regu\| \ridgepred(\matUbar \Gst^\top,\regu )^\top  \|_{\op}$ are both bounded as $\lesssim  \Ovfit_{\mu}(\delta)$.
\end{lem}
Theorem~\ref{thm:ph_vr_general}, Part (b) is now an immediate consequence of the two lemmas above.

\subsubsection{Proof of Lemma~\ref{lem:LS_char}\label{sec:LS_char}}
	The first statement is standard. For the second statement, observe that $\ridgepred(\matA v,\regu) = \ridgepred(\matA,\regu)v$. Hence, we have
	\begin{align*}
	\min_{\phi \in \R^{m \times \Lbar}} \twonorm{\vcat{\matK \phi^\top - \matA , \regu\phi^\top} v}^2 &= \min_{\phi} \|\matK (\phi^\top v) - \matA v\|_{2}^2 + \|\regu \phi^\top v\|_2^2 \\
	 &= \min_{w \in \R^{\Lbar}} \|\matK w - \matA v\|_{2}^2 + \|\regu  w\|_2^2 \\
	 &= \|\matK \ridgepred(\matA v,\regu) - \matA v\|_{2}^2 + \|\regu \ridgepred(\matA v,\regu)\|_2^2\\
	&= \|\matK \ridgepred(\matA,\regu) v - \matA v\|_{2}^2 + \|\regu \ridgepred(\matA,\regu) v\|_2^2\\
	&= \twonorm{\vcat{\matK \ridgepred(\matA,\regu) - \matA , \regu\ridgepred(\matA,\regu)} v}^2 \numberthis \label{eq:ridge_equiv}.
	\end{align*}
	To conclude, we see that 
	\begin{align*}
	\opnorm{\vcat{\matK \predrid^\top - \matY , \regu\predrid^\top}}=&\;\opnorm{\vcat{\matK \ridgepred(\matY,\regu) - \matY , \regu\ridgepred(\matY,\regu)}}\\
	&= \max_{v \in \sphere{m}} \twonorm{\vcat{\matK \ridgepred(\matY,\regu) - \matY , \regu\ridgepred(\matY,\regu) } v}\\
	&\overset{(i)}{=}\max_{v \in \sphere{m}} \min_{\phi}\twonorm{\vcat{\matK \phi^\top - \matY , \regu\pred^\top} v}\\
	&\overset{(ii)}{\le}  \min_{\phi} \max_{v \in \sphere{m}} \twonorm{\vcat{\matK \phi^\top - \matY , \regu\pred^\top} v}\\
	&=  \min_{\phi} \:\opnorm{\vcat{\matK \phi^\top - \matY , \regu\pred^\top }},
	\end{align*} 
	where $(i)$ uses~\eqref{eq:ridge_equiv} and $(ii)$ uses weak duality. Since, by definition,
	\begin{equation*}
	\min_{\phi} \:\opnorm{\vcat{\matK \phi^\top - \matY , \regu\phi^\top}}  \le \opnorm{\vcat{\matK \predrid^\top - \matY , \regu\predrid^\top}},
	\end{equation*}
	they must be equal.

\subsubsection{Proof of Lemma~\ref{lem:LS_error_terms}}
\textbf{Proof of $(a)$:} By Lemma~\ref{lem:LS_char}, $\ridgepred$ is linear in its first argument. Therefore,
\begin{align*}
\predrid  &= \ridgepred( \matY, \regu)= \ridgepred( \Delpast +  \matUbar \Gst^\top , \regu)= \ridgepred( \Delpast, \regu) +  \ridgepred( \matUbar\Gst^\top ,\regu).
\end{align*}
Therefore,
\begin{align*}
\left\|\Delpast_{\predrid} \right\|_{\op} &= \left\|\Delpast - \matK \predrid^\top \right\|_{\op} \\
&= \|\Delpast - \matK \left(\ridgepred( \Delpast,  \regu ) +  \ridgepred( \matUbar \Gst^\top, \regu)\right)\|_{\op} \\
&\le \|\Delpast - \matK \cdot \ridgepred( \Delpast,  \regu )\|_{\op}  + \|\matK\ridgepred( \matUbar \Gst^\top,\regu)\|_{\op}\\
&= \|\Delpast_{\ridgepred( \Delpast,  \regu )}\|_{\op}  + \|\matK\ridgepred( \matUbar \Gst^\top)\|_{\op}
\le \Opt_{\regu} + \|\matK\ridgepred( \matUbar \Gst^\top,\regu)\|_{\op},
\end{align*}
by~\eqref{eq:ridge_error}.

\noindent\textbf{Proof of $(b)$:}
Again, using the linearity of $\ridgepred$ and~\eqref{eq:ridge_error}, we see that
\begin{align*}
\regu \left\|\predrid\right\|_{\op} &= \regu \|\ridgepred( \Delpast,  \regu ) +  \ridgepred( \matUbar \Gst^\top,\regu\|_{\op} \\
&\le \regu \|\ridgepred( \Delpast,  \regu )\|_{\op} +  \regu \|\ridgepred (\matUbar \Gst^\top, \regu)\|_{\op}.\\
&\le \Opt_{\regu} +  \regu \|\ridgepred (\matUbar \Gst^\top, \regu)\|_{\op}.
\end{align*}

\subsubsection{Proof of Lemma~\ref{lem:ridge_on_u}} 
For $M \in \{\matK,\regu\}$, we have
\begin{align*}
\opnorm{M\ridgepred(\matUbar \Gst^\top, \regu)^\top } &= \opnorm{M(\matK^\top \matK + \regu^2)^{-1} \matK^\top\matUbar\Gst^\top }\\
&\le \opnorm{M(\matK^\top\matK + \regu^2)^{-1/2}}\cdot \opnorm{(\matK^\top \matK + \regu^2)^{-1/2} \matK^\top \matUbar} \opnorm{\Gst}\numberthis \label{eq:ridge_min2}\\
&\overset{(i)}\le \opnorm{(\matK^\top \matK + \regu^2)^{-1/2} \matK^\top}\opnorm{\matUbar}\opnorm{\Gst}\\
&\overset{(ii)}{\le}\opnorm{\matUbar}\opnorm{\Gst}.
\end{align*}
where $(i)$ and $(ii)$ use the fact that $\opnorm{\matK(\matK^\top\matK + \regu^2)^{-1/2}}$, $\opnorm{\regu(\matK^\top\matK + \regu^2)^{-1/2}}$, and \\$\opnorm{(\matK^\top\matK + \regu^2)^{-1/2}\matK^\top}$ are all upper bounded by $1$. For example,
\begin{align*}
\opnorm{\matK(\matK^\top\matK + \regu^2)^{-1/2}} \leq \; \opnorm{\matK}\opnorm{(\matK^\top\matK + \regu^2)^{-1/2}} \leq \frac{\opnorm{\matK}}{\sqrt{\sigma_{\min}(\matK^\top\matK + \regu^2)}} \leq \;1.
\end{align*}
Thus, $\opnorm{M\ridgepred(\matUbar\Gst ^\top, \regu) } \le \opnorm{\matUbar}\opnorm{\Gst}$ which is $\lesssim \sqrt{N}\opnorm{\Gst}$ on $\eventU$. 

For the second bound, starting from~\eqref{eq:ridge_min2}, we use the semi-parametric regression bound, part (a) of Theorem~\ref{thm:chain_semi_par}. Since $\matk_t$ is $\filtg_{t-T}$-adapted, by applying  part (a) of Theorem~\ref{thm:chain_semi_par} with $V_0 \leftarrow \regu^2I$ and $\matDel \leftarrow \matK$, we see that with probability at least $1 - \delta$,
\begin{align*}
    \opnorm{(\matK^\top \matK + \regu^2)^{-1/2}\matK^\top \matUbar} \lesssim T^{1/2}\sqrt{\ptil + \log \tfrac{1}{\delta} + \log \det(I + \regu^{-2} \matK \matK^{\top})^{1/2}  }.
    \end{align*}
 Concluding we have that, on $\eventU$, with probability at least $1-\delta$,
\begin{align*}
\opnorm{M\ridgepred(\matUbar \Gst^\top, \regu) } &\lesssim  \opnorm{\Gst} \min \{ T^{1/2}\sqrt{\ptil + \log \tfrac{1}{\delta} + \log \det(I + \regu^{-2} \matK \matK^{\top})^{1/2}  }, \sqrt{N} \} \\
&= \Ovfit_{\mu}(\delta).
\end{align*}

\section{Semi-Parametric Regression\label{sec:semi_par}}

In this section we prove Proposition~\ref{prop:gvr_bound} as consequence of a more general setting. Our main theorem is as follows:

\begin{thm}\label{thm:chain_semi_par} Suppose the semi-parametric model~\eqref{eq:semi_par_def} holds, where $\matu_t | \filtr_{t-1}$ is mean-zero and $1$ subgaussian, and let $\ptil:= p \min\{T,\log^2(eTp)\log^2(Tp)\}$. Then,
\begin{enumerate}
    \item[(a)] For any fixed $\delta \in (0,1)$ and $V_0 \in \PD{m}$, it holds with probability $1-\delta$ that
    \begin{align*}
    \opnorm{(\matDel^\top \matDel + V_0)^{-1/2}\matDel^\top \Ubar} \lesssim T^{1/2}\sqrt{\ptil + \log \tfrac{1}{\delta} + \log \det(I + \matDel  V_0^{-1}\matDel^{\top})^{1/2}  }
    \end{align*}
    \item[(b)]  For any fixed $\delta \in (0,1)$ and $\kappa > 0$, it holds with probability $1- \delta$ that
    \begin{align*}
    \sigma_{\min}(\matUbar)^2 \opnorm{\matUbar^{\dagger}\matDel} \le \opnorm{\matDel^\top\matUbar} \lesssim T^{1/2}(\opnorm{\matDel} + \kappa)\sqrt{\ptil + m + \log \tfrac{1}{\delta} + \lil (\tfrac{\opnorm{\matDel}}{\kappa})},
    \end{align*}
    where $\lil(x) := \log(1 + \log(1+x))$.
\end{enumerate}
\end{thm}
Proposition~\ref{prop:gvr_bound} corresponds exactly to part (b) of the above theorem, with the substitution $\matdel_t \leftarrow \matdel_{\phi,t}$. Indeed, we have that (a)
\begin{align*}
	\Gvr(\pred) =  (\matY -  \matK \phi^\top)^\top\Ubar \left(\Ubar^\top \Ubar\right)^{-1} = \Gst + \matDel_{\phi}^\top\Ubar^{\dagger \top},
\end{align*}
(b) that the inputs $\matu_{t+1} | \filtg_t \sim \calN(0,I_p)$ are $1$-subgaussian and $\{\filtg_t\}$-adapted, and where we substitute the errors $\matdel_{\phi,t}$ are $\filtg_{t-T}$-adapted. Hence, whenever $\eventU$ occurs, applying Theorem~\ref{thm:chain_semi_par}, part (b) with $\kappa \leftarrow \kappa \sqrt{N}$ and $\matdel_{t} \leftarrow \matdel_{\phi,t}$ yields 
\begin{align*}
    N\opnorm{\matUbar^{\dagger}\matDel_{\phi}} \lesssim \sigma_{\min}(\matUbar)^2 \opnorm{\matDel_{\phi}\matUbar^{\dagger}} \lesssim T^{1/2}(\opnorm{\matDel_{\phi}} + \kappa \sqrt{N})\sqrt{\ptil + m + \log \tfrac{1}{\delta} + \lil (\tfrac{\opnorm{\matDel_{\phi}}}{\sqrt{N}\kappa})},
    \end{align*}
which coincides with the statement of Proposition~\ref{prop:gvr_bound} after dividing both sides by $N$.

For comparison, observe that if one instead regressing $(\maty_t)$ to independent white noise $\widetilde{u}_t \iidsim \calN(0,I_{Tp}) \in \R^{Tp}$, we would have rates $\|\Gls - \Gst\|_{\op}  = \Theta({\sqrt{\frac{1}{N}}\cdot\sqrt{(Tp + m) + \log\tfrac{1}{\delta}}})$. In our setting, martingale tail bounds incurs an addition $\log$-factor (part (a) of Theorem~\ref{thm:chain_semi_par}), which we refine to $\lil(\cdot) $ for operator norm error using Lemma~\ref{lem:log_to_lil}, below. In addition, correlation introduced by the concatenated sequence forces us to pay an additional factor of $T$ multiplying the $m$ and $\log \tfrac{1}{\delta}$ terms as well, which we conjecture is in the worst case. If handled naively, we would also have to pay for $T \cdot (Tp) = pT^2$, and the 
major effort in the proof of Theorem~\ref{thm:chain_semi_par} is to use a careful chaining argument based on~\citep{krahmer2014suprema} to ensure that we instead pay for the generally smaller term $T\ptil$. This ensures that our bound is only suboptimal (up to more than log factors) once $m + \log \tfrac{1}{\delta} \gg p$.

We present the proof of part (a) in~\ref{sec:thm:chain_semi_par_proof_a}, and derive (b) as a consequence in~\ref{sec:thm:chain_semi_par_proof_b}. Part (a) relies on detailed chaining arguments, which we defer to Appendix~\ref{sec:martingale_chaining}.

\subsection{Proof of Theorem~\ref{thm:chain_semi_par}, Part (a)\label{sec:thm:chain_semi_par_proof_a}}
We shall also use the notation
\begin{align*}
\Nbar = N - \None \quad \text{ and } \quad N_0 = \None - T.
\end{align*}
Adopting the setting of Theorem~\ref{thm:chain_semi_par}, $\{\filtg_t\}$ be a filtration, and suppose that $(\matu_t) \subset \R^{p}$ is a $\{\filtg_t\}$-adapted $1$-subgaussian and $(\matdel_t) \subset \R^{m}$ is a sequence of $\{\filtg_{t-T}\}$ adapted sequence.

We shall begin with the proof of part (a); the proof of part (b) will be derived as a consequence in~\ref{sec:thm:chain_semi_par_proof_b}. Begin by fixing $V_0 \in \PD{m}$. For vectors $\vbar = [v_1^\top,\dots,v_T^\top]^\top \in \R^{Tp}$, we define the process 
\begin{align*}
\matz_{\vbar} := \|\matDel^\top \matcir{u} \vbar\|_{\matDel^\top \matDel + V_0}.
\end{align*}
 We see that
\begin{align*}
\|(\matDel^\top \matDel + V_0)^{-1/2}\matDel^\top \matcir{u}\|_{\op} = \sup_{\vbar \in \sphere{Tp}} \matz_{\vbar}\:.
\end{align*}
Hence, it suffices to show that, with probability $1-\delta$,
\begin{align*}
\sup_{\vbar \in \sphere{Tp}}\matz_{\vbar} \lesssim T^{1/2}\sqrt{\log(1/\delta) + \log \det(I + (\matDel^\top  \matDel  V_0^{-1})^{1/2} + \ptil}
\end{align*}

The first step is to express $\matcir{u}\vbar = \vbartoep \longvector{\matu}$, where $\longvector{\matu} \in \R^{(N - N_0+1)p}$ is the vector obtained by concatenating $\matu_{N_0}, \dots,\matu_{N}$ and $\vbartoep\in \mathbb{R}^{\Nbar \times (N-N_0+1)p}$  is an  $[\Nbar \times (N-N_0+1)]$ block Toeplitz matrix, where each of the size $(1\times p)$ blocks on the $k$-th superdiagonal are equal to $v_{T-k}^\top$ for $0\leq k\leq T-1$, i.e.
\begin{align*}
\vbartoep :=  
\begin{bmatrix}
v_T^\top & \cdots & v_1^\top& &\\
& \ddots & &\ddots&\\
& &v_T^\top & \cdots & v_1^\top 
\end{bmatrix}\:.
\end{align*}
We can then express $\matDel^\top \vbartoep \in \R^{m \times ((N - N_0+1)p}$ with blocks $\Apast_{t-1}(\vbar)^\top \in \R^{m \times p}$ for $t \in \{N_0,\dots,N\}$, where 
\begin{align*}
\Apast_{t-1}(\vbar)^\top  := \sum_{j=0}^{T - 1} (\matdel_{t+j} \I_{ t+j \in [N_1,N]} ) v_{T - j}^\top.
\end{align*}
Since $\Apast_{t-1}(\vbar)^{\top}$ depends only on terms $\delta_s$ for $s \le t + T - 1$, we see that $\Apast_{t-1}(\vbar)$ is $\filtg_{t-1}$-adapted. Then we see that
\begin{align*}
\matDel^\top \matcir{u} \vbar = \matDel^{\top}\vbartoep \longvector{\matu} = \sum_{t = N_0}^N \Apast_{t-1}(\vbar)^\top\matu_t.
\end{align*}
We now introduce the variance process:
\begin{align*}
\Varproc_{\matDel}(\vbar) := \sum_{t=N_0}^N \Apast_{t-1}(\vbar)^\top\Apast_{t-1}(\vbar) = \matDel^\top \vbartoep^\top \vbartoep \matDel,
\end{align*}
which corresponds to a variance proxy for $\matDel^\top \matcir{u} \vbar $. Indeed, consider the simple case where $\matu_t$ are deterministic and  $\matdel_t \iidsim \calN(0,1)$ are independent Gaussian. 
	Then, then the covariance matrix 
	\begin{align*}
	\Exp[ (\matDel^\top\matcir{u}\vbar)(\matDel^\top \matcir{u} \vbar)^\top]
	\end{align*} 
	would be equal to $\Varproc_{\matDel}(v)$. 
	In fact, if $\matu_t$ is any martingale sequence with $\Exp[\matu_t \matu_t ] \preceq I$, and $\matu_t$ were deterministic, then it still holds $\Exp[ (\matDel^\top\matcir{u}\vbar)(\matDel^\top \matcir{u} \vbar)^\top] \preceq \Varproc_{\matDel}(v)$. 

	In our general case where both $\matdel_t$ and $\matdel_t$ are martingales, and $\matdel_t$ is subgaussian,  we can use $\Varproc_{\matDel}(v)$ as a \emph{data-dependent} subgaussian variance proxy. Recall that $\matdel_t | \filtg_{t-1}$ is subgaussian and $\filtg_{t}$ measurable. Crucially, we will also use that for each $\vbar$, $\{\Apast_{t}(\vbar)\}$ is $\{\filtg_{t}\}$-adapted, since $\Apast_t(\vbar)$ involves the terms only the terms $\matdel_{t+1},\dots,\matdel_{t+T}$, and $\{\matdel_{s}\}$ is $\{\filtg_{s-T}\}$-adapted. These two points let us invoke the following lemma, which generalizes a bound due to \cite{yasin11}:
	\begin{lem}[Generalization of Theorem 3 in~\cite{yasin11}]\label{lem:general_martingale}
	Let $\calF_t$ denote an arbitrary filtration. Let  $\{\Apast_t\} \subset \R^{p \times m}$ and $\{\matu_t\} \subset \R^m$ be $\filtg_{t}$ adapted, and suppose further that $\matu_t | \filtg_{t-1}$ is mean zero and $1$-subgaussian. Define the variance process $\Vpast_{k}:= \sum_{t=1}^k \Apast_{t-1}^\top\Apast_{t-1}$. Then, for any $\{\calG_t\}$-adapted stopping time $\stoptau$, one has
	\begin{align*}
	\left\|\sum_{t=1}^{\stoptau} \Apast_{t-1}^\top\matu_t \right\|_{(\Vpast_{\stoptau} + V_0)^{-1}}^2 \le 2\log \left(\frac{\det(I + \Vpast_{\stoptau}V_0^{-1})^{1/2}}{\delta}\right) \quad\text{ w.p. } 1- \delta.
	\end{align*}
	\end{lem}
	To prove part (a) of the proposition, we shall set $\stoptau = N$,  and $\Apast_{t-1} = \Apast_{t-1}(\vbar)$  for $t \ge N_0$, and $\Apast_{t-1} = 0$ for $t \le N_0$. 
	For this choice of $\Apast$, then $\sum_{t=1}^{N} \Apast_{t-1}^\top \matu_t = \matDel^\top \matcir{u} \vbar$ and $\Vpast_{\stoptau} = \Vpast_N = \Varproc_{\matDel}(\vbar)$. Thus, we have that for any $V_0 \succeq 0 $, 
	\begin{align}
	\left\| \matDel^\top \matcir{u} \vbar \right\|_{(\Varproc_{\matDel}(v) + V_0)^{-1}}^2 \le 2\log \left(\frac{\det(I + \Varproc_{\matDel}(v) V_0^{-1})^{1/2}}{\delta}\right) \quad\text{ w.p. } 1- \delta. \label{eq:lowner_semi_par_martingale}
	\end{align}

	To use~\eqref{eq:lowner_semi_par_martingale} to bound $\sup_{\vbar \in \sphere{Tp}}  \matz_{\vbar}$, we shall show that that the random variables $\matz_{\vbar}$ behave like a subgaussian process, with a random offset $\matmu$. This will allow us to apply a chaining argument to bound their supremum. To this end, introduce $\ominorm{\vbar} := \opnorm{\vbartoep}$; it is straightforward to check that $\ominorm{\cdot}$ defines a norm on $\R^{Tp}$. We shall need the following bound, proved in Section~\ref{sec:ominorm_lem}:
\begin{lem}[Bounds on $\ominorm{\cdot}$]\label{lem:ominorm_bound} $\twonorm{\vbar} \le \ominorm{\vbar} \le \sqrt{T}\twonorm{\vbar}$.
\end{lem}
 By definition, $\Varproc_{\matDel}(\vbar) = \matDel^\top \vbartoep \vbartoep^\top \matDel$, and thus the above lemma implies 
 \begin{align*}
 \Varproc_{\matDel}(\vbar) \preceq \opnorm{\vbartoep}^2 \matDel^\top \matDel  = \ominorm{\vbar}^2 \matDel^\top \matDel.
 \end{align*} 
 Thus, we see 
\begin{align*}
\matz_{\vbar} = \|\matDel^\top \matcir{u} \vbar\|_{(\matDel^\top \matDel + V_0)^{-1}} &= \ominorm{\vbar}\|\matDel^\top \matcir{u} \vbar\|_{(\ominorm{\vbar}^2\matDel^\top  \matDel + \ominorm{\vbar}^2 V_0)^{-1}}\\
&\le \ominorm{\vbar}\|\matDel^\top \matcir{u} \vbar\|_{(\Varproc_{\matDel}(\vbar) + \ominorm{\vbar}^2 V_0)^{-1}}.
\end{align*}
Thus, by~\eqref{eq:lowner_semi_par_martingale}, we have that with probability $1-\delta$, 
\begin{align*}
\|\matDel^\top \matcir{u} \vbar\|_{(\Varproc(\vbar) + \ominorm{\vbar}^2 V_0)^{-1}} &\le \sqrt{2\log \left(\frac{\det(I + \Varproc_{\matDel}(\vbar) (\ominorm{\vbar}^2 V_0^{-1}) ^{1/2}}{\delta}\right)}\\
&\overset{(i)}{=} \sqrt{2\log \left(\frac{\det(I + \matDel^\top  \matDel  V_0^{-1}) ^{1/2}}{\delta}\right)}\\
&\le \sqrt{2\log \left(\det(I + \matDel^\top  \matDel  V_0^{-1}) ^{1/2}\right)} + \sqrt{2\log(1/\delta)},
\end{align*}
where in $(i)$ we use the fact that if $A_1,B \succ 0$ and $A_2 \succeq A_1$, then $\det(I + A_1B^{-1}) \le \det(I + A_2B^{-1})$, together with the bound $\Varproc_{\matDel}(\vbar) \preceq \ominorm{\vbar}^2 \matDel^\top \matDel$. 

Hence, introducing the random offset $\matmu := \sqrt{2\log \det(I + (\matDel^\top  \matDel  V_0^{-1})^{1/2}}$, which does not depend on $\vbar$, we see that for any $\vbar \in \R^{Tp}$, 
\begin{align*}
&\Pr[ \matz_{\vbar} \ge\ominorm{\vbar}(\matmu + \sqrt{2\log(1/\delta)}) ] \\
&\;=  \Pr[\|\matDel^\top \matcir{u} \vbar\|_{(\matDel^\top \matDel + V_0)^{-1}} \ge \ominorm{\vbar}(\matmu + \sqrt{2\log(1/\delta)})] \le \delta. \numberthis\label{eq:z_chain_single}
\end{align*}
Therefore, we find that
\begin{align*}
&\Pr\left[\matz_{x} - \matz_{y} > \ominorm{x-y}\left(\matmu + \sqrt{2\log(1/\delta)}\right)\right] \\
&= \Pr\left[\|\matDel^\top \matcir{u} x\|_{(\matDel^\top \matDel + V_0)^{-1}} - \|\matDel^\top \matcir{u} y\|_{(\matDel^\top \matDel + V_0)^{-1}}  > \ominorm{x-y} \left(\matmu + \sqrt{2\log(1/\delta)}\right)\right]\\
&\le \Pr\left[\|\matDel^\top \matcir{u} (x-y)\|_{(\matDel^\top \matDel + V_0)^{-1}} > \ominorm{x-y} \left(\matmu + \sqrt{2\log(1/\delta)}\right)\right] \le \delta \numberthis\label{eq:z_chain_pairwise}.
\end{align*}
Intuitively this inequality says that $\matz_{\vbar}$ has subgaussian tails, modulo the offset $\matmu$. This enables us to  a chaining argument to the increments $\matz_{\vbar}$, which is detailed in Appendix~\ref{sec:martingale_chaining}:
	\begin{cor}\label{cor:chaining_toep} There exists universal constants $c_1,c_2$ such that, with probability $1-\delta$
	\begin{align*}
	\sup_{\vbar \in \sphere{Tp}} \matz_{\vbar} \le c_1\left((\sqrt{\log(c_2/\delta)} + \matmu)\max_{x \in \sphere{Tp}}\ominorm{x}+ \gamma_{2}(\sphere{Tp}, \ominorm{\cdot}) \right),
	\end{align*}
	where $\gamma_{2}(\sphere{Tp}, \ominorm{\cdot})$ denotes Talagrand's $\gamma_2$-functional (see, e.g.~\cite{talagrand2014upper}).
	\end{cor}
	By Lemma~\ref{lem:ominorm_bound}, we have $\max_{\vbar \in \sphere{Tp}}\ominorm{x} \le\sqrt{T}\ \max_{\vbar \in \sphere{Tp}}\twonorm{\vbar} = \sqrt{T}$. In Section~\ref{sec:chaining_toeplitz_proof}, we sharpen a computation of Dudley's bound due to~\cite{krahmer2014suprema} to control $\gamma_{2}(\sphere{Tp}, \ominorm{\cdot})$:
	\begin{prop}[Control of $\gamma_{2}(\sphere{Tp}, \ominorm{\cdot})$] \label{prop:chaining_toeplitz} We have the bound
	\begin{align*}\gamma_{2}(\sphere{Tp}, \ominorm{\cdot}) \le \sqrt{T\ptil}, \quad \text{where } \ptil := p\min\{T, \log^2(eTp) \log^2(eT)\}.
	\end{align*} 
	\end{prop}
	Combining the above Proposition with Corollary~\ref{cor:chaining_toep}, we have with probability $1-\delta$,
	\begin{align*}
	 \sup_{\vbar \in \sphere{Tp}} \matz_{\vbar} &\lesssim \left((\sqrt{\log(c_2/\delta)} + \matmu)\sqrt{T}+ \sqrt{T\ptil} \right)\\
	&\lesssim T^{1/2}\sqrt{\log(1/\delta) + \matmu^2 + \ptil} \\
	&\lesssim T^{1/2}\sqrt{\log(1/\delta) + \log \det(I + \matDel^\top  \matDel  V_0^{-1})^{1/2} + \ptil}\\
	&= T^{1/2}\sqrt{\log(1/\delta) + \log \det(I + \matDel^\top  V_0^{-1} \matDel  )^{1/2} + \ptil}, \quad \text{as needed}.
	\end{align*}

\subsection{Proof of Theorem~\ref{thm:chain_semi_par}, Part (b)\label{sec:thm:chain_semi_par_proof_b}}
	We shall bound $\sup_{w \in \sphere{m-1}}\|w^\top \matDel^\top \matcir{u}\|_2$. For $w \in \sphere{m}$, the $m = 1$ case of Theorem~\ref{thm:chain_semi_par} with $\matDel \leftarrow \matDel w$  and $V_0 = \kappa^2 \in \R$ implies that with probability at least $1-\delta$,
	\begin{align*}
	\frac{\|w^\top \matDel^\top\matcir{u}\|_2}{\sqrt{w^\top \matDel^\top \matDel w + \kappa^2}} &= \left\|\left(\left(\matDel w\right)^\top \left(\matDel w\right) + \kappa^2 \right)^{-1/2}\left(\matDel w\right)^\top \matcir{u}\right\|_{\op} \\
	&\lesssim T^{1/2}\sqrt{\log(1/\delta) + \log (1 + (\left(\matDel w \right)^\top  \left(\matDel w\right)  \kappa^{-2})^{1/2} + \ptil}\\
	&\le T^{1/2}\sqrt{\log(1/\delta) + \log (1 + \frac{\opnorm{\matDel}^2}{\kappa^2} )^{1/2} + \ptil}\\
	&\le T^{1/2}\sqrt{\log(1/\delta) + \log (1 + \frac{\opnorm{\matDel}}{\kappa} ) + \ptil},
	\end{align*}
	where the last line uses $\sqrt{1 + x^2} \le 1 + x$ for $x \ge 0$. Since $w^\top \matDel^\top \matDel w \le \opnorm{\matDel}^2$, rearranging shows we have that with probability $1- \delta$,
	\begin{align*}
	\|w^\top \matDel^\top\matcir{u}\|_2 &\le \sqrt{\opnorm{\matDel}^2 + \kappa^2}T^{1/2}\sqrt{\log(1/\delta) + \log (1 + \frac{\opnorm{\matDel}}{\kappa} ) + \ptil}\\
	&\le (\opnorm{\matDel} + \kappa ) T^{1/2}\sqrt{\log(1/\delta) + \log (1 + \frac{\opnorm{\matDel}}{\kappa} ) + \ptil}
	\end{align*}
	We now proceed to union bound over $w$. A standard covering argument (see e.g.~\citet[Section 4.2]{vershynin2018high}) shows that if $\calT$ is an $\epsilon$-net of $\sphere{m}$, then $\opnorm{\matDel^\top\matcir{u}} = \sup_{w \in \sphere{m}} \|w^\top \matDel^\top\matcir{u}\|_2 \\\le \tfrac{1}{1-\epsilon} \sup_{w \in \calT} \|w^\top \matDel^\top\matcir{u}\|_2$.
	A standard computation (see e.g.~\citet[Corollary 4.2.13]{vershynin2018high}) lets us choose $|\calT| \le m \log(1 + \frac{2}{\epsilon})$. Setting $\epsilon = 1/2$ and union bounding over $w \in \calT$, we have that with probability at least $1-\delta$,
	\begin{align*}
	\|\matDel^\top\matcir{u}\|_{\op} &\lesssim (\opnorm{\matDel} + \kappa ) T^{1/2}\sqrt{\log(1/\delta) + \log (1 + \frac{\opnorm{\matDel}}{\kappa} ) + m + \ptil}.
	\end{align*}
	Recall the definition $\lil(x) := \log_+(\log_+(x))$. The final bound follows directly from invoking the following lemma with $Z = \|\matDel^\top\matcir{u}\|_{\op}$, $C \lesssim T^{1/2}$, $D = m + \ptil$, and $M = \opnorm{\matDel}$:
	\begin{lem}[Iterated Logarithm Conversion]\label{lem:log_to_lil} Let $Z$ be a random variable and suppose that there exists constants $C,D$ and a random variable $M$ such that, for any $\kappa > 0$, it holds that $\Pr[ Z \ge C(\kappa + M)\sqrt{\log(1/\delta) + D + \log (1 + \frac{M}{\kappa})}] \le \delta$. Then, for any $\beta,\kappa > 0$, one also has
	\begin{align*}
	\Pr\left[ Z \ge C((1+\beta)M + \kappa)\sqrt{\log\tfrac{2}{\delta} +2\lil(\tfrac{\beta M}{\kappa}) + D  + \log \left(1 + \tfrac{e}{\beta}\right)}\right] \le \delta.
	\end{align*}
	In particular, if $D \gtrsim 1$, then by setting $\beta = 1$, $\Pr\left[ Z \gtrsim C(M + \kappa)\sqrt{\log \tfrac{1}{\delta} +\lil(\frac{ M}{\kappa}) + D }\right] \le \delta$.
	\end{lem}
	\begin{proof}
	Define $\kappa_j = e^{j-1} \kappa$ for $j \ge 1$, and $\delta_j = \frac{\delta}{2j^2}$. Then, $\sum_{j \ge 1} \delta_j \le \delta$, and by a union bound, 
	\begin{align*}
	\Pr[ Z \ge \inf_{j \ge 1 } C(\kappa_j + M)\sqrt{\log \tfrac{1}{\delta_j} + D + \log (1 + \tfrac{M}{\kappa})}] \le \delta.
	\end{align*} 
	In particular, choosing $j = \floor{\log \beta M/\kappa}$, we have that $\frac{\beta M}{e} \le \kappa_j \le \max\{\kappa,\beta M\}$, which implies $M + \kappa_j \le M + \max\{\kappa, \beta M\}) \le (1+\beta)M + \kappa$ and $\log(1 + \frac{M}{\kappa_j}) \le \log (1 + \frac{e}{\beta})$. Moreover, $\delta_j  = \frac{\delta}{2j^2}\ge\frac{\delta}{2 \max\{ \log \beta M/\kappa, 1\}^2}$, which implies $\log(1/\delta_j) \le \log(2/\delta) + 2\log ( \max\{ \log (\beta M/\kappa) , 1\}) \le \log(2/\delta) +  2\lil(\beta M/\kappa)$. Hence,
	\begin{align*}
	\Pr\left[ Z \ge C((1+\beta)M + \kappa)\sqrt{\log(2/\delta) +2\lil(\tfrac{\beta M}{\kappa}) + D  + \log (1 + \tfrac{e}{\beta})}\right] \le \delta.
	\end{align*}

	\end{proof}

\subsection{Proof of Lemma~\ref{lem:ominorm_bound}\label{sec:ominorm_lem}}
To lower bound $\ominorm{\vbar}  = \opnorm{\vbartoep} \ge \twonorm \vbar$. observe that the first row of $\vbartoep$ consists of the vector $[v_N^\top,v_{N-1}^\top,\dots,v_1^\top] \in \R^{Tp}$, followed by zeros. Hence, $\opnorm{\vbartoep} \ge \twonorm{[v_N^\top,v_{N-1}^\top,\dots,v_1^\top]} = \twonorm{\vbar}$. To upper bound $\ominorm{\vbar} = \opnorm{\vbartoep}$, observe that by the norm-contraction inequality $\opnorm{\vbartoep}$ is bounded by the operator norm of the Toeplitz matrix $M \in \calR^{(N-N_1 + 1) \times (N - N_0 + 1)}$ (recall $N_0 = N_1 - T$) whose $ij$-th entry is the norm of the $ij$-th vector block of $\vbartoep$. This is an upper triangular Toeplitz matrix with an associated sequence $a_i = \|V_{T - i}\|_2$ for $i \in \{0,\dots,T-1\}$, and $a_i = 0 $ for $i \ge N$. By a standard inequality, the operator norm of $M$ is bounded by the $\ell_1$ norm of  the sequence $a_0,a_1,\dots$, which is $\sum_{i=0}^{T-1}\|V_{N-i}\|$, which in turn is at most $\sqrt{T \sum_{i=0}^{T-1} \|V_{N-i}\|_2^2} = \sqrt{T\twonorm{\vbar}^2}$.

\subsection{Proof of Lemma~\ref{lem:general_martingale}}
The proof is essentially identical to that of Theorem 3 in~\cite{yasin11}, with the exception that the variance-process is matrix-valued, and the noise process is vector-valued. Let $\matS_k := \sum_{t=1}^k \Apast_{t-1}^\top \matu_t$. We begin by constructing a supermartingale for each direction $w \in \R^m$:
\begin{align*}
M_k(w) := \exp\left( \langle w, \matS_k \rangle - \frac{1}{2}\|w\|_{\Vpast_k}^2\right).
\end{align*}
Note that $M_0 = 1$. To verify that $M_k(w)$ is a supermartingale with respect to the filtration $\calF_k$, we see use the fact that $\matS_k = \matS_{k-1} + \Apast_{k-1}\matu_k$ and $\Vpast_{k} = \Apast_{k-1}\Apast_{k-1}^\top + \Vpast_{k-1}$ to write
\begin{align*}
\Exp[M_k(w) | \calF_{k-1}] &= \Exp\left[\exp\left(\langle w, \Apast_{k-1}^\top\matu_k + \matS_{k-1} \rangle - \frac{1}{2}\|w\|^2_{\Vpast_{k-1} + \Apast_{k-1}^\top\Apast_{k-1} } \right)\right]\\
 &= \Exp\left[\exp\left(\langle w, \Apast_{k-1}^\top\matu_k + \matS_{k-1}\rangle - \frac{1}{2}\|w\|^2_{\Vpast_{k-1}} - \frac{1}{2}\|w\|^2_{\Apast_{k-1}^\top\Apast_{k-1}}  \right)\right]\\
  &= \Exp\left[\exp\left(\langle w, \Apast_{k-1}^\top\matu_k\rangle  - \frac{1}{2}\|w\|^2_{\Apast_{k-1}^\top\Apast_{k-1}}\right)  M_{k-1}(w)\right]\\
  &\overset{(i)}{=} M_{k-1}(w) \cdot \Exp\left[\exp\left(\langle w, \Apast_{k-1}^\top\matu_k\rangle - \frac{1}{2}\|w\|^2_{\Apast_{k-1}^\top\Apast_{k-1}}\right) \right]\\
  &= M_{k-1}(w) \cdot \Exp\left[\exp\left(\langle \Apast_{k-1} w, \matu_k\rangle -  \frac{1}{2}\|\Apast_{k-1} w\|^2\right) \right] \overset{(ii)}{\le} M_{k-1},
\end{align*}
where $(i)$ uses the fact that $M_{k-1}$ is $\calF_{k}$ measurable, and $(ii)$ uses that $M_{k-1} \ge 0$ and $\langle\Apast_{k-1} w ,\matu_k \rangle$ is $\|\Apast_{k-1} w\|^2$-subgaussian. Since $M_0 = 1$, we conclude by the optional stopping theorem that for any $k \in \N$ and $w \in \R^m$, $\Exp[M_k(w)] \le 1$. The remainder of the proof follows that of Theorem 3 in~\cite{yasin11} verbatim. Specifically, these steps show that 
\begin{align*}
\|S_{\tau}\|_{\overline{\Vpast}_{\tau}^{-1}}^2 \le  2\log \left(\frac{\det(V_0^{-1/2})\det(\overline{\Vpast}_{\tau}^{1/2})}{\delta}\right),~\text{ where } \overline{\Vpast}_{\tau} = \Vpast_{\tau} + V_0.
\end{align*}
To conclude, we verify that 
\begin{align*}
\det(\overline{\Vpast}_{\tau}^{1/2})\det(V_0^{-1/2}) &= \sqrt{ \det(\overline{\Vpast}_{\tau})\det(V_0^{-1})}~= \sqrt{\det( \overline{\Vpast}_{\tau} \cdot V_0^{-1})}\\
&= \sqrt{\det((V_0  + \Vpast_{\tau})\cdot V_0^{-1}}~= \sqrt{\det(I + \Vpast_{\tau}V_0^{-1})}.
\end{align*}


\section{Chaining for Self-Normalized Tail Inequalities \label{sec:martingale_chaining}}
	In this section, we introduce a generic inequality for martingales. Let's consider the general set up. Let $(\Omega,\calF)$ denote a probability space, $\calX$ denote a separable space with metric $\dist(\cdot,\cdot,)$, $\{\matz_x\}_{x \in \calX}$ denote a real valued random process defined on $(\Omega,\calF)$
	\begin{defn} Let $\matsig,\matmu$ denote random variables taking values in $\R_{\ge 0}$.  We say that a process $\{\matz_x\}_{x \in \calX}$ is a $(\matsig,\matmu)$-offset subgaussian process on $(\calX,\dist)$ if, for any $x,y \in \calX$ and $u > 0$,
	\begin{align*}
	\Pr[\matz_{x} - \matz_{y} \ge \dist(x,y)(\matsig u  + \matmu)] \le \exp(-u^2/2).
	\end{align*}
	\end{defn}
	Note that we do not require $\matsig,\matu,\{\matz_x\}$ to be independent. We now define Talagrand's $\gamma_2$ functional:
	\begin{defn} Let $\calX$ be a separable metric space. Talagrand's $\gamma_2$ function is defined as 
	\begin{align*}
	\gamma_2(\calX,\dist) := \inf_{(\calA_n) : |\calA_n| \le 2^{2^n}} \sup_{x \in \calX} \sum_{n \ge 0} 2^{n/2}\diam(\calA_n(x)),
	\end{align*}
	where the $\inf$ is taken over all sequences of partitions $(\calA_n)$ of $\calX$ of size $|\calA_n| \le 2^{2^n}$ (with the exception of $|\calA_0| = 1$).
	\end{defn}
	Typically, computing the infimum in the definition of $\gamma_2$ may be challenging. Fortunately, there exists a easier-to-manage upper bound on $\gamma_2$ due to Dudley:
	\begin{prop}[Dudley's Bound]  Let $\calX$ be a separable metric space with $\diam(\calX) < \infty$. Then, 
	\begin{align*}
	\gamma_2(\calX,\dist) \le \int_{0}^{\diam(\calX)}\sqrt{\log \calN(\calX,\dist, u)}du,
	\end{align*}
	where $ \calN(\calX,\dist, u)$ denotes the cardinality of the minimal $u$-covering of $\calX$, that is, the cardinality of a minimal subset $\calT$ of $\calX$ satisfying $\sup_{x \in \calX} \inf_{y \in \calT} \dist(x,y) \le u$.
	\end{prop}
	Finally, we introduce the main theorem of this section, which extends the generic chaining applied to typical subgaussian processes to $(\matsig,\matmu)$-offset-subgaussian process:
	\begin{thm}\label{thm:offset_chaining}  Let $\{\matz_x\}_{x \in \calX}$ be a $(\matsig,\matmu)$-offset subgaussian process. Then, there exists universal constants $c$ such that
	\begin{align*}
	\Pr[\sup_{x,y \in \calX} |\matz_x - \matz_{y}| \ge c_1\left(\matsig\gamma_2(\calX,\dist) + \diam(\calX)(\matsig u + \matmu)\right)] \le c_2e^{-u^2}.
	\end{align*}
	If in in addition $\calX$ is normed spaced with $\dist(x,y) = \|x-y\|$, and $\Pr[\matz_{x} \ge \|x\|(\matsig u  + \matmu)] \le \exp(-u^2/2)$, then
	\begin{align*}
	\Pr\left[\sup_{x,y \in \calX} |\matz_x| \ge c_1\left(\matsig\gamma_2(\calX,\dist) + (\min_{x \in \calX}\|x\| + \diam(\calX))(\matsig u + \matmu)\right)\right] \le c_2e^{-u^2}.
	\end{align*}
	\end{thm}

	With Theorem~\ref{thm:offset_chaining} in hand, the proof of Corollary~\ref{cor:chaining_toep} is nearly immediate.
	\begin{proof}
	Let $\calX = \sphere{Tp} \subset \R^{Tp}$, and define the norm $\ominorm{\cdot} := \|\vbartoep\|_{\op}$. Recall the inequalities~\eqref{eq:z_chain_single} and~\eqref{eq:z_chain_pairwise}, restated here for convenience with the $\ominorm{\cdot}$ notation:
	\begin{align*}
		\Pr\left[\matz_{\vbar} \ge \sqrt{2 \log(1/\delta)}\ominorm{\vbar}\matsig + \matmu \right] \le  \delta,
		\end{align*}
		where $\matu = \matmu := \sqrt{2\log \det(I + (\matDel^\top  \matDel  V_0^{-1})^{1/2}}$ and $\matsig = 1$. By setting $u = \sqrt{2\log(1/\delta)}$, we see that $\matz_{x}$ is $(\matsig,\matmu)$-offset normed subgaussian with respect to $\gamma_2(\calX,\ominorm{\cdot})$. The corollary is therefore a direct consequence of Theorem~\ref{thm:offset_chaining}, again substituting in $u = \sqrt{2\log(1/\delta)}$.
	\end{proof}

\subsection{Proof of Proposition~\ref{prop:chaining_toeplitz}\label{sec:chaining_toeplitz_proof}}
	\begin{proof} Our argument follows the proof of Theorem 4.1 in~\cite{krahmer2014suprema}, modifying the argument to ensure tighter dependences on $T$ and $p$. Notably, we remove a dependence on the number of samples $N$ which would arise from invoking~\cite{krahmer2014suprema} without alteration. We shall use Dudley's inequality to bound
	\begin{align*}
	\gamma(\sphere{Tp},\ominorm{\cdot}) &\le \int_{u = 0}^{\diam(\sphere{Tp},\ominorm{\cdot})} \sqrt{\log \calN(\sphere{Tp},\ominorm{\cdot}, u)}du \\
	&=  \int_{u = 0}^{\sqrt{2T}} \sqrt{\log \calN(\sphere{Tp},\ominorm{\cdot}, u)}du \lesssim \sqrt{Tp(\min\{T,\log^2(eT)\log^2(eTp)\})}
	\end{align*}
	Using the fact that $\ominorm{\cdot} \le \sqrt{T}\|\cdot\|_2$, and $\log \calN(\sphere{Tp},\sqrt{T}\|\dot\|, u) \lesssim Tp\log(1 + \frac{\sqrt{T}}{u})$, one can coarsely bound the above by $\sqrt{Tp} \int_{u = 0}^{\sqrt{2T}} \sqrt{\log(1 + \frac{\sqrt{T}}{u}} \lesssim \sqrt{T^2p}\int_{0}^1 \sqrt{\log(1 + \frac{1}{u})}du \lesssim \sqrt{T^2p}$. 

	It remains to prove the more refined bound of $\sqrt{Tp \log^2(eT)\log^2(eTp)}$ To do so, we need to control the associated covering numbers $\log \calN(\sphere{Tp},\ominorm{\cdot}, u)$. To this end, we shall require two ingredients. The first is known as Maurey's Lemma, 
	\begin{lem}\label{lem:maurey}\citet[Lemma 4.2]{krahmer2014suprema} Let $\calU$ denote a finite subset of a normed space $(\calX,\|\cdot\|)$, and suppose that that there exists an $A > 0$ such that, for any $k \in \mathbb{N}$ and sequence $(u_1,\dots,u_k) \in  \calU^k$, $\Exp [ \|\sum_{i=1}^k \mateps_i u_i\|] \le A\sqrt{k} $, where $\mateps_i$ are independent Rademacher random variables. Then,
	\begin{align*}
	\log \calN(\conv(\calU), \|\cdot\|, u ) \le \left(\frac{A}{u}\right)^2 \log |\calU|.
	\end{align*}
	\end{lem}
	In order to apply Maurey's lemma, we shall choose $\calU = \{\sqrt{Tp} \cdot e_i\}_{i=1}^{Tp}$ and $\calX = \R^{Tp}$ with the metric $\ominorm{\cdot}$. Observe that $\sphere{Tp} \subset \conv(\calU)$, and thus  $\log \calN(\sphere{Tp}, \ominorm{\cdot}, u ) \le  \log \calN(\conv(\calU), \ominorm{\cdot}, u )$. To estimate the latter quantity using Lemma~\ref{lem:maurey}, we shall require the following characterization of $\ominorm{\cdot}$:
	\begin{lem}\label{lem:omin_char} There exists a set $\calZ \subset \R^{Tp}$ with $|\calZ| \le 16\pi(Tp)^{3/2}$ such that $\|z\|_{\infty} \le 1$ for all $z \in \calZ$ and $\ominorm{\vbar} \le 4\max_{z \in \calZ} z^\top \vbar$ for all $\vbar \in \R^{Tp}$.
	\end{lem} 
	Lemma~\ref{lem:omin_char} is proven in the subsection below. We stress that it the lemma is crucial to removing the $N$-dependence in our final bound; the proof of Theorem 4.1 in~\cite{krahmer2014suprema} effectively renders $\ominorm{\vbar} \le \max_{z \in \calZ'}  z^\top \vbar$ for some $|\calZ'| \approx N$.

	 To obtain a covering estimate, observe that for any sequence $u_1,\dots,u_k \in \calU^k$, Lemma~\ref{lem:omin_char} implies
	\begin{align*}
	\Exp[\|\sum_{i=1}^k \mateps_i u_i\|] \lesssim \Exp[\max_{z \in \calZ} \langle z, \sum_{i=1}^k \mateps_i u_i\rangle] = \Exp[\max_{z \in \calZ}  \sum_{i=1}^k \mateps_i \cdot \langle z, u_i\rangle]
	\end{align*}
	For any fixed $z \in \calZ$, we have $|\langle z, u_i \rangle| \le \|z\|_{\infty} \|u_i\|_1 \le \sqrt{Tp}$. Thus, by a union bound and Hoeffding's inequality, we have that for any appropriate constant $c_1$,
	\begin{align*}
	\Pr[ \sup_{z \in \calZ} \sum_{i=1}^k \mateps_i \cdot \langle z, u_i\rangle \ge \sqrt{Tp} \cdot \sqrt{k}u] \le \sum_{z \in \calZ}\Pr[\sum_{i=1}^k \mateps_i \cdot \langle z, u_i\rangle \ge \sqrt{Tp} \cdot \sqrt{k}u]  \le |\calZ|e^{-c_1 u^2}.
	\end{align*}
	A standard tail integration argument reveals that $\Exp[ \sup_{z \in \calZ} \sum_{i=1}^k \mateps_i \cdot \langle z, u_i\rangle \ge \sqrt{k}u] \lesssim \sqrt{k} \cdot \sqrt{\log |\calZ|} \cdot \sqrt{Tp}$ for an appropriate constant $c_2$. Thus, $\calU$ and $\ominorm{\cdot}$ satisfy the conditions of Lemma~\ref{lem:maurey} with $A \lesssim \sqrt{Tp \log |\calZ|}$, and thus 
	\begin{align*}
	\log \calN(\sphere{Tp}, \ominorm{\cdot}, u ) \le  \log \calN(\conv(\calU), \ominorm{\cdot}, u ) \lesssim \frac{Tp \log |\calZ|}{u^2} \cdot\log |\calU| \lesssim \frac{Tp\log^2 (eTp)}{u^2},
	\end{align*}
	where we use the bound $|\calU| = Tp$ and $|\calZ| \lesssim (Tp)^{3/2}$.
	For $u \le 1$, this bound is quite loose, and instead we shall use the bound $\ominorm{\cdot} \le T\|\cdot\|_2 $ to obtain a standard covering bound:
	\begin{align*}
	\log \calN(\sphere{Tp}, \ominorm{\cdot}, u ) \le \log \calN(\sphere{Tp}, Tp\|\cdot\|_2, u ) \lesssim Tp \log(\frac{eT}{u}),
	\end{align*}
	Invoking Dudley's bound yields
	\begin{align*}
	\gamma(\sphere{Tp},\ominorm{\cdot})  &\lesssim \int_{1}^{\sqrt{2T}} \sqrt{\frac{Tp \log^2(eTp)}{u^2}}du + \int_{0}^{1} \sqrt{Tp \log\frac{eT}{u}}\\
	&= \sqrt{Tp}\left( \log(eTp) \int_{1}^{\sqrt{2T}} \frac{du}{u} + \int_{0}^{1} \sqrt{\log\tfrac{eT}{u}}\right)\\
	&\le \sqrt{Tp}\left( \log(eTp) \int_{1}^{\sqrt{2T}} \frac{du}{u} + \sqrt{\log eT} + \int_{0}^1\sqrt{\log \tfrac{1}{u}}\right)\\
	&= \sqrt{Tp}\left( \log(eTp) \log(\sqrt{2T}) + \sqrt{\log eT} + \int_{0}^1\sqrt{\log \tfrac{1}{u}} \right)\\
	&\lesssim \sqrt{Tp}\log(eTp) \log(eT),
	\end{align*}
	where we use the fact that $ \int_{0}^1\sqrt{\log \tfrac{1}{u}}$ is at most a universal constant.
	\end{proof}
\subsubsection{Proof of Lemma~\ref{lem:omin_char}}
We can embed $\vbartoep$ as a submatrix an infinite Toeplitz operator where $a_{-i} = \vbar_{i+1}$ for $i \in \{0,\dots,Tp\}$, and $0$ elsewhere. It is well known that such a matrix has operator norm bounded by the $\Hinf$ norm $\max_{\theta \in [0,2\pi]}\left|\sum_{k=1}^{Tp}\vbar[k]e^{i k\theta}\right|$, where $|\cdot|$ denotes complex magnitude. Define $F(\theta) := \left|\sum_{k=1}^{Tp}\vbar[k]e^{i k\theta}\right|$. Our proof will have two ingredients:
\begin{enumerate} 
\item For all $\theta \in [0,2\pi]$, we can bound $F(\theta) \le \sqrt{2}\sup_{\ell \in [4]} \langle z_{\theta,\ell}, \vbar \rangle$, where $\|z_{\theta,\ell}\|_{\infty} \le 1$ for $\ell \in [4]$.
\item We construct a covering $\calT$ of $[0,2\pi]$ such that $\sup_{\theta \in [0,2\pi]} F(\theta) \le 2\max_{\theta \in \calT}F(\theta)$, and $|\calT| \le 4\pi (Tp)^{3/2}$
\end{enumerate}
Together, these two imply that, for $\calZ = \{z_{\theta,\ell}: \theta \in \calT, \ell \in [4]\}$, $\|\vbar\|_{\ominus} \le 2\sqrt{2}\sup_{z \in \calZ}\langle z, \vbar \rangle$,
as well as $\|z\|_{\infty} \le 1$ for all $z \in \calZ$, and $|\calZ| = 4 |\calT| \le 16\pi (Tp)^{3/2}$.
\vspace{.5cm}

\noindent \textbf{Proof of Point 1:}
Given a complex number $\omega$, we have $|\omega| = \sqrt{\Re(\omega)^2 + \Im(\omega)^2} \le \sqrt{2}\max\{|\Re(\omega)|,|\Im(\omega)|\} = \sqrt{2}\max\{\Re(\omega),-\Re(\omega),\Im(\omega), -\Im(\omega)\}$. Defining the complex vector $z_{\theta} = (e^{i k\theta})_{k=1}^{Tp}$, and $z_{\theta,1} = \Re(z_{\theta})$, $z_{\theta,2} = \Im(z_{\theta})$, $z_{\theta,3} = -z_{1,\theta}$, $z_{4,\theta} = -z_{\theta,2}$, we see that $F(\theta) \le \sqrt{2}\max_{\ell \in [4]} \langle z_{\ell,\theta},\vbar\rangle$. Moreover, $\max_{\ell \in 4}\|z_{\ell,\theta}\|_{\infty} \le 1$, as $\max\{|\Re(e^{ik\theta})|,|\Im(e^{ik\theta})|\} \le 1$ for all $k \in \R$.
\vspace{.5cm}

\noindent \textbf{Proof of Point 2:} It suffices to show that $F(\theta)$ is $\|\vbar\|_2\cdot(Tp)^{3/2}$-Lipschitz. Indeed, if this is true then by choosing a $\frac{1}{2 (Tp)^{3/2}}$-net $\calT$ of $[0,2\pi]$, we have
\begin{align*}
\ominorm{\vbar} &= \max_{\theta \in[0,2\pi]} F(\theta) ~\le \max_{\theta \in \calT}  F(\theta)   + \max_{\theta \in[0,2\pi]}\min_{\theta' \in \calT}|F(\theta) - F(\theta')|\\
&\le \max_{\theta \in \calT}  F(\theta)   + \frac{1}{2 (Tp)^{3/2}} \cdot (Tp)^{3/2}\|\vbar\|_2 \le \max_{\theta \in \calT}  F(\theta)  + \frac{1}{2}\ominorm{\vbar},
\end{align*} 
where we used the bound that $\ominorm{\vbar} \ge \|\vbar\|_2$. After rearranging, $\ominorm{\vbar} \le  \max_{\theta \in \calT}  2F(\theta)$. Lastly, we note that we can construct a $\frac{1}{2 (Tp)^{3/2}}$-net of the interval $[0,2\pi]$ of size at most $\floor{4\pi (Tp)^{3/2}} \le 4\pi (Tp)^{3/2}$. It remains to show that $F(\theta)$ is Lipschitz. We can bound  
\begin{align*}
|F(\theta_1) - F(\theta_2)| \le \left|\sum_{k=0}^{Tp - 1}\vbar[k](e^{i k\theta_1} - e^{ik \theta_2})\right| \le \|\vbar\|_2\sqrt{\sum_{k=0}^{Tp - 1}|e^{i k\theta_1} - e^{ik \theta_2}|^2}
\end{align*} 
by Cauchy Schwartz. Geometrically, $|e^{i k\theta_1} - e^{ik \theta_2}|$ is the distance between the point $(\cos k\theta_1, \sin k \theta_1)$ and the point $(\cos k \theta_2, \sin k \theta_2)$ on the unit sphere, which is at most the arc length $k |\theta_2 - \theta_1|$ between the two points. Hence, $\sqrt{\sum_{k=0}^{Tp - 1}|1 - e^{ik (\theta_2 - \theta_1)}|^2|} \le \sqrt{\sum_{k=0}^{Tp - 1} k^2 (\theta_2 - \theta_1)^2 } \le (Tp)^{3/2} |\theta_2 - \theta_1|$.

\subsection{Proof of Theorem~\ref{thm:offset_chaining}} 
The second part of the theorem is a consequence of the first after noting that 
\begin{align*}
\max_{y \in \calX} \matz_y \le \min_{x \in \calX} \left(|\matz_{x}| + \max_{y \in\calX} |\matz_y - \matz_x| \right)\le \min_{x \in \calX} |\matz_{x}| + \max_{x,y \in \calX} |\matz_y - \matz_x|,
\end{align*}
and that $ \max_{x,y \in \calX} |\matz_y - \matz_x|$ can be bounded by the first part of the theorem, whereas $\min_{x \in \calX} \matz_{x}$ can be bounded by the condition $\Pr[\matz_{x} \ge \|x\|(\matsig u  + \matmu)] \le \exp(-u^2/2)$.

The proof of the first part of the theorem is analogous to the proof of Theorem 2.2.27 in~\cite{talagrand2014upper}. With a standard separability argument, we may assume without loss of generality that $\calX$ is finite.  Let $N_n := 2^{2^n}$. By the definition of the $\gamma_2$ functional, we may choose a sequence $(\calA_n)$ of partitions of $\calX$ of size at most $|\calA_n| \le N_n$ satisfying $(\calA_n)_{n \ge 0}$ with $\sup_{x \in \calX} \sum_{n \ge 0} 2^{n/2}\diam(\calA_n(x))  \le 2\gamma_2(\calX)$.

The key observation is that this is chosen to (nearly) minimize a bound involving only distances between elements of $\calX$, and is not chosen based on the random variances $\matsig$ or offsets $\matmu$.

Let $(\calX_n)_{n \ge 0}$ denote any sequence of subsets of $\calX_n$ where $\calX_n$ comprises of exactly one element of each set in $\calA_n$; note then that $|\calX_n| \le N_n$. Define the unioned sets $\calU_n := \bigcup_{q = 0}^{n} \calX_q$. In particular, $\calU_0 = \calX_0$, and $|\calU_n| \le \sum_{q = 0}^n N_q \le 2N_n$. Lastly, we define the event
\begin{align*}
\calE(u) := \{ |\matz_{x} - \matz_{x}| \le 2\dist(x_1,x_2)(\matsig (u + 2^{n/2})  + \matmu ) , \quad \forall n \ge 1, x,y \in \calU_n, \matz_{x} \}.
\end{align*}
Display (2.60) in Talagrand verifies that $\Pr[\calE(u)^c] \le 2\sum_{n \ge 1} |\calU_n|^2\exp( - 2(2^n + u^2)) \le c_2 \exp(-2u^2)$. Therefore, it suffices to show that there exists a constant $c_1$ such that, if for any fixed $\sigma, \mu \ge 0$, then for any deterministic real-valued process $\{z_{x}\}_{x \in X}$ satisfying the condition 
\begin{align}\label{eq:tail_cond}
|z_{x} - z_{y}| \le 2\dist(x,y)(\sigma (u + 2^{n/2})  + \mu ) , \quad \forall n \ge 1, x,y \in \calU_n,
\end{align} 
it holds that $\sup_{x \in \calX} |z_{x_1} - z_y| \le c_1\left(\sigma\gamma_2(\calX) + \diam(\calX)(\sigma u + \mu)\right)$.

First consider the case $\sigma = 0$. For any $x,y \in \calX$, it suffices to show $|z_{x} - z_{y}| \le 2 \mu \diam(\calX)$.  Since $\bigcup_n \calU_n = \calX$ and $\calX$ is finite by assumption, there exists an $n$ large enough for which $x,y \in \calU_n$.  By~\eqref{eq:tail_cond}, $|z_{x} - z_{y}| \le 2\dist(x,y)(\sigma (u + 2^{n/2})  + \mu ) = 2\dist(x,y) \mu \le 2 \mu \diam(\calX)$.

Next, consider the case $\sigma > 0$. Define $z_x' := z_x/\sigma$, and $u' = u + \mu/\sigma$. Then, the process $z_x'$ satisfies the condition:
\begin{align}\label{eq:tail_cond_prime}
|z_{x}' - z_{y}'| \le 2\dist(x,y)(u' + 2^{n/2})  , \quad \forall n \ge 1, x,y \in \calU_n,
\end{align}
The proof of~\citet[Theorem 2.2.27]{talagrand2014upper} shows that~\eqref{eq:tail_cond_prime} implies that $\sup_{x,y \in \calX} |z_x' - z_y'| \le c_1(\gamma_{2}(\calX) + \diam(\calX) u') $. Multiplying both sides by $\sigma$, we have 
\[\sup_{x,y \in \calX} |z_x - z_y| \le c_1(\gamma_{2}(\calX) + \sigma \diam(\calX)  u') = c_1(\gamma_2(\calX) + \diam(\calX) (\sigma u + \mu))\:.\]

\newpage
\part{Prefiltered Least Squares for Linear Dynamical Systems\label{part:LTI}}

\section{Bounds on \texorpdfstring{$\opnorm{\matDel_{\phi}}$}{Delta}~\label{sec:error_calcs}}

In this section, we extend Proposition~\ref{prop:error_bound_stochastic} to the case where $\matx_1 \ne 0$, and to the adversarial noise regime. Throughout, we assume that $\matx_1$ is deterministic (stochastic bounds can be developed by reasoning over the randomness of $\matx_1$). When $\matx_1 \ne 0$, we now have three LTI systems:
\begin{align}
	\Gsf_\phi := \fourtup{\Ast}{\Bst}{C_\phi}{0},\quad \Fsf_\phi := \fourtup{\Ast}{B_w}{C_\phi}{0}, \quad \Hsf_\phi := \fourtup{\Ast}{\matx_1}{C_\phi}{0}\: \label{eq:phi_LTIs}.
\end{align}
Our bound for stochastic noise with general $\matx_1 \ne 0$ is as follows:
\begin{prop}[Stochastic Noise Bound]\label{prop:error_bound_stochastic_matx1}
Consider a filter of the form $\phi = [\Psi_1 | \dots | \Psi_{d}] \in \R^{m \times dm}$ for some $1 \le d \le L$, and suppose that $N \ge Td\max\{m,\log(1/\delta)\}$. Then, under the stochastic noise model of Assumption~\ref{asm:noise}, with probability at least $1-\delta$ we have that the extended filter $\phitil := [\phi \mid | \mathbf{0}_{m\times(L-m)d}] \in \R^{L \times m}$ satisfies
\begin{align*}
    \|\matDel_{\phitil}\|_{\op} &\;\lesssim \sqrt{N}(\|\Markov_{\infty}(\Gsf_\phi)\|_{\op}+\|\Markov_{\infty}(\Fsf_\phi)\|_{\op}) + \Mknorm{\infty}{\Hsf_\phi}\\
    &\; + \sqrt{m+\log(1/\delta)}(\mixnorm{N}{\Gsf_\phi} + \mixnorm{N}{\Fsf_\phi})\\
    &\;+ \sqrt{N} (1+\|\phi\|_{\blockop}) \left(\Mknorm{Td}{\Gsf} + \Mknorm{Td}{\Fsf}+ \opnorm{D_z}\right),
\end{align*}
where we define $\mixnorm{N}{\Gsf} := \min\{\sqrt{N}\Mknorm{\infty}{\Gsf}, \|\Gsf\|_{\Hinf}\}\:.$.
\end{prop}
For adversarial noise, we obtain the following analogous bound:
\begin{prop}[Bound on $\Opt_{\lambda}$ for adversarial noise] \label{prop:error_bound_adversarial} Consider a filter of the form $\phi = [\Psi_1 | \dots | \Psi_{d}]$ for some $1 \le d \le L$.
With probability at least $1-\delta$, if $N \ge Td\max\{m,\log(1/\delta)\}$, then in the adversarial noise model, the extended filter $\phitil := [\phi \mid | \mathbf{0}_{m\times(L-m)d}] \in \R^{L \times m}$ satisfies
\begin{align*}
    \|\matDel_{\phitil}\|_{\op} \;\lesssim & \sqrt{N}\Mknorm{\infty}{\Gsf_\phi}+\sqrt{Nd_w}\mixnorm{N}{\Fsf_\phi} + \Mknorm{\infty}{\Hsf_\phi}\\
    &\quad+   \sqrt{m+\log(1/\delta)}\mixnorm{N}{\Gsf_\phi}\\
    &\quad+  \sqrt{N}(1+\|\phi\|_{\blockop})\left[\Mknorm{Td}{\Gsf} +\sqrt{Tdd_w}\Mknorm{Td}{\Fsf}+ \sqrt{dd_z} \opnorm{D_z}\right].
\end{align*}
\end{prop}

\subsection{Outline of the Proofs}
We shall assume without loss of generality that $d = L$, because the extended filter $\phitil \in \R^{m \times Lm}$ and the original $\phi \in \R^{m \times dm}$ yields the same errors.
We now outline the proofs of Proposition~\ref{prop:error_bound_stochastic} and Proposition~\ref{prop:error_bound_adversarial}. Throughout, we fix a filter $\phi \in \R^{m \times Lm}$. We are aiming to control
\begin{equation*}
    \opnorm{\matDel_{\phi}} = \left\| \Delpast - \matK \phi^\top\right\|_{\op} = \left\| \left[\matdel_{\None} - \phi \cdot \matk_{\None}  | \matdel_{\None+1} - \phi \cdot \matk_{\None+1}  |\ldots | \matdel_{N} - \phi \cdot \matk_{N}    \right]\right\|_{\op}.
\end{equation*}
Concretely, we take our features to be $L$ previous output values equally $T$-spaced, i.e.
\begin{equation*}
    \matk_t := [\maty_{t-T}^\top \mid \maty_{t-2T}^\top \mid \dots\mid\maty_{t- LT}^\top]^\top\:.
\end{equation*}
We express our filter as $\phi = [\Psi_1 | \dots | \Psi_{L}]$, for $\Psi_j \in \R^{m \times m}$. Thus, our prediction takes the form
\begin{equation*}
    \phi \cdot \matk_t = \sum_{\ell = 1}^{L} \Psi_{\ell} \maty_{t-\ell T}\:.
\end{equation*}
Recall from Section~\ref{sec:bounding_opt} the auxiliary signal $\xtil_{n;t}$, associated observation $\ytil_{n;t}$, and auxiliary features $\ktil_t$ defined via  
\begin{align*}
\xtil_{n;t} := &\;\begin{cases} \Ast^{n - (t - L T)}\matx_{t - L T} & n \ge t - L T \\
\matx_{n} & n \le t - L T
\end{cases}\\
\ytil_{n;t} :=  &\;\Cst \xtil_{n;t} \\
\ktil_t := &\;[\ytil_{t-T;t}^\top\mid \ytil_{t-2T;t}^\top\mid\dots \mid\ytil_{t-L T;t}^\top]^\top\:.
\end{align*}
We can now decompose the error term as 
\begin{align*}
    \matdel_{\phi,t} = \matdel_t - \phi \cdot \matk_t = &\; \underbrace{(\tmaty_{t;t} - \phi \cdot \tmatk_t)}_{(\Errone_{\phi,t})} + \underbrace{[(\matdel_t - \tmaty_{t;t}) - \phi \cdot (\matk_t - \ktil_t)]}_{({\Errtwo_{\phi,t}})},
\end{align*}
where we have suppressed the dependence on $\phi$ in $\Errone_t$ and $\Errtwo_t$, as $\phi$ is fixed. We further define the stacked errors\footnote{Here, we reverse the time ordering of the columns (which is norm-preserving) in order to cleanly write many quantities in terms of Toeplitz operators.}
\begin{align*}
\Err^{(j)}_{\phi,t_1:t_2} := \left[\Err^{(j)}_{\phi,t_1}| \Err^{(j)}_{\phi,t-1} |\ldots | \Err^{(j)}_{\phi,t_2}\right]\:,
\end{align*}
and our goal will be to control $\opnorm{\Err_{\phi,N:\None}}\leq\|\Errone_{\phi,N:\None}\|_{\op}  + \|\Errtwo_{\phi,N:\None}\|_{\op}$. Toward bounding these two terms, we outline a general strategy to bound their individual components; the full details can be found in the remainder of the section. Let $\matutotal \in \R^{Np}$, $\matwtotal \in \R^{Nd_w}$, $\matztotal \in \R^{N d_z}$ denote the concatenated (from $N$ down to 1) sequences of input, process noise, and sensor noise vectors. By the linearity of the system, we can express
\begin{align*}
\Err_{\phi,N:\None} = (\Gsfone + \Gsftwo)\matutotal + (\Fsfone + \Fsftwo)\matwtotal + \Gsftwo_{\matz}\matztotal + \Hsfone \matx_1\:.
\end{align*}
where $\Gsfone$ encodes the contribution of $\matutotal$ to $\Err^{(1)}_\phi$, $\Gsftwo$ encodes the contribution of $\matutotal$ to $\Err^{(2)}_\phi$, and so on. For example, $\Gsfone$ is a linear operator that maps $\R^{Np}$ to $\R^{\Nbar \times m}$. Note that $\matztotal$ only contributes to the second error term and $\matx_1$ only contributes to the first.

Then, the key to bounding the error terms is that we can write, for example,
\begin{align*}
v^\top \Gsfone \matutotal = &\; \left(\Gsfonev\matutotal\right)^\top
\end{align*}
for an appropriate block matrix $\Gsfonev \in \R^{\Nbar \times Np}$. Thus,
\begin{align}\label{eq:key_matrix}
\opnorm{\Gsfone \matutotal} = \sup_{v \in \calS^{m-1}} \|\Gsfonev\matutotal\|_2.
\end{align}
Our main tool to control this quantity is the following result, due to~\cite{krahmer2014suprema} (somewhat overloading our notation for $\calM$):\begin{prop}\label{prop:mendelson_tail} Let $\calM \subset \R^{n \times m}$ denote a class of matrices, and let $\matxi \in \R^{m}$ denote a subgaussian vector. Then, 
\begin{align*}
\Pr[\sup_{M \in \calM} \|M\matxi\|_2 \gtrsim \gamma_2(\calM,\|\cdot\|_{\op}) + \diam(\calM,\|\cdot\|_{\fro}) + \sqrt{\log(1/\delta)}\diam(\calM,\|\cdot\|_{\op}) ] \le 1 - \delta\:.
\end{align*}
\end{prop}
We observe that, even with the martingale structure, the vectors $\matutotal$ is a sub-Gaussian vector, as is $\matwtotal$ in the stochastic model. Our strategy is outlined as follows:
\begin{enumerate}
	\item We begin give a detailed proof to bound $\Gsfone$ in Section~\ref{sec:first_term}.
	\item We begin give a detailed proof to bound $\Gsftwo$ in Section~\ref{sec:second_term}.
	\item We explain how to bound the contributions of the remaining terms in Section~\ref{sec:other_error_terms}. Specifically, we address process noise under both stochastic and adversarial models in~\ref{sec:process_noise} and adversarial noise in~\ref{sec:output_noise}. Lastly, we address the contributions of the initial state in~\ref{sec:initial_state_noise}. 
\end{enumerate}

\subsection{Bounding \texorpdfstring{$\Gsfone$}{G1}}\label{sec:first_term}
\paragraph{Step 1: Recognize a block matrix structure}
Recalling that $\Gsf_\phi = (\Ast, \Bst, C_\phi, 0)$, for each vector $v\in \R^{m}$, we define the block Toeplitz matrix
\begin{align*}
\Gsfonevt :=
\begin{bmatrix}
&v^\top \Markov_{N-\Tbar}(\Gsf_\phi)\\
\mathbf{0}_{1\times p}& v^\top \Markov_{N-1-\Tbar}(\Gsf_\phi)\\
\vdots&\vdots\\
\mathbf{0}_{1\times (\Nbar-2) p}& v^\top \Markov_{N_1+1-\Tbar}(\Gsf_\phi)\\
\mathbf{0}_{1\times (\Nbar-1) p}& v^\top \Markov_{N_1-\Tbar}(\Gsf_\phi)\\
\end{bmatrix} \in \R^{\Nbar\times(N-\Tbar-1)p}\:.
\end{align*}
One can verify that $\Gsfonevt$ is the nonzero portion (i.e. submatrix) of $\Gsfonev$, which is sufficient for our goal of bounding matrix norms.
\paragraph{Step 2: Bound the Frobenius and operator norms} 
By considering each row of $\Gsfonevt$, we immediately see that
\begin{align*}
\fronorm{\Gsfonevt}^2 \leq N \fronorm{v^\top \Markov_{N}(\Gsf_\phi)}^2 \leq &\; N \|v\|_2^2 \opnorm{\Markov_{N}(\Gsf_\phi)}^2\:. 
\end{align*}
Since $\Mknorm{k}{\Gsf_\phi}$ is increasing in $k$, we can bound $\Mknorm{N}{\Gsf_\phi} \le \Mknorm{\infty}{\Gsf_\phi}$. However, we would like to note that, in general, $\max_{v\in\sphere{m}} \fronorm{v^\top \Markov_{N}(\Gsf_\phi)}$ could be much less than this quantity, and using it instead would sharpen Proposition~\ref{prop:error_bound_stochastic}.

Now, concerning the operator norm, we have two options. First, we could simply take 
\begin{align*}
\opnorm{\Gsfonevt}\leq \fronorm{\Gsfonevt} \leq \|v\| \sqrt{N}  \Mknorm{\infty}{\Gsf_\phi}\:.
\end{align*}
On the other hand, we see that
\begin{align*}
\opnorm{\Gsfonevt} \stackrel{(i)}{\leq} \|v^\top\Gsf_\phi\|_{\Hinf} =  \sup_{|z|=1}\opnorm{v^\top C_\phi(zI-\Ast)^{-1}\Bst} \leq \twonorm{v} \|\Gsf_\phi\|_{\Hinf}\:,
\end{align*}
where (i) comes from the fact that $\Gsfonevt$ is, in a sense, a ``submatrix" of the infinite-dimensional linear operator $v^\top \Gsf_\phi$ (see e.g.~\cite{tilli1998singular} Corollary 4.2). Thus, recalling that
\begin{align*}
\mixnorm{N}{\Gsf_\phi} = \min\{\sqrt{N}\Mknorm{\infty}{\Gsf_\phi}, \|\Gsf_\phi\|_{\Hinf}\}\:,
\end{align*}
 we have that $\sup_{v \in \sphere{m}}\opnorm{\Gsfonevt} \leq \mixnorm{N}{\Gsf_\phi}$.
\paragraph{Step 3: Bound the $\gamma_2$ functional}
To bound $\gamma_2(\calM,\|\cdot\|_{\op})$ with $\calM = \{\Gsfonev: v \in \sphere{m}\}$, we first note that $\gamma_2(\calM,\|\cdot\|_{\op}) \leq \gamma_2(\sphere{m},\|\Gsf^{(1)}_{(\cdot)}\|_2$), as the linear map $v\to \Gsfonev$ is only injective, in general. Then, one can use Step 2 above to show (see~\cite{talagrand2014upper} Exercise 2.2.23) that
\begin{align*}
\gamma_2(\sphere{m},\|G^{(1)}_{(\cdot)}\|_2) \leq &\; \mixnorm{N}{\Gsf_\phi} \: \gamma_2(\sphere{m},\|\cdot\|_2)\:.
\end{align*}
Finally, upper bounding the right hand side by Dudley's integral and using the standard covering number bound for the sphere gives
\begin{equation*}
\gamma_2(\calM,\|\cdot\|_{\op})\lesssim \mixnorm{N}{\Gsf_\phi} \sqrt{m}\:.
\end{equation*}
We remark that one could instead use the generic chaining (see \cite{talagrand2014upper} Chapter 2) to bound $\gamma_2$ directly and, instead of $\sqrt{m}$, get the \emph{stable rank} of some matrix, but we will be loose in this aspect.

Putting these three steps together with Proposition~\ref{prop:mendelson_tail} shows that with probability at least $1-\delta$,
\begin{align*}
\opnorm{\Gsfone \matutotal} \lesssim &\; \sqrt{N}\Mknorm{\infty}{\Gsf_\phi} + \sqrt{m + \log{(1/\delta)}} \mixnorm{N}{\Gsf_\phi}\:.
\end{align*}

\subsection{Bounding \texorpdfstring{$\Gsftwo$}{G2}}\label{sec:second_term}
	Now, we seek to bound $\opnorm{\Gsftwo \matutotal} = \sup_{v \in \calS^{m-1}} \|\Gsftwov\matutotal\|_2$, and we repeat the previous three steps.
	\paragraph{Step 1} The explicit block matrix structure is a bit cumbersome to write out, but the key is to note that the rows of $\Gsftwov$ are simply shifted versions of $v^\top G^{(2)}$, where
	\begin{align*}
	G^{(2)} := \Markov_{Td}(\Gsf)\cdot\blkdiag(\mathbf{0}_{T\times T},I_{T(L-1)\times T(L-1)})-\sum_{k=1}^L \Psi_k[\mathbf{0}_{m \times Tk} \mid \Markov_{T(L-k)}(\Gsf)]\:.
	\end{align*}
	Note that the slight complication with $\blkdiag$ occurs due to the subtraction of $\Gst \overline \matu_t$ in $\matdel_t$; this $\blkdiag$ term is simply identity in the process noise case of~\ref{sec:process_noise}.
	\paragraph{Step 2}
	By the same argument as for the first error term, we have that $\fronorm{\Gsftwov}\leq \sqrt{N} \twonorm{v} \opnorm{G^{(2)}}$. The operator norm requires slightly more care, and requires the following lemma.
	\begin{lem} We have that $\opnorm{\Gsftwov} \lesssim \sqrt{\Tbar} \twonorm{v} \opnorm{G^{(2)}}$.
	\end{lem}
	\begin{proof}
	As stated before, $\Gsftwov$ is embeddable in a block Toeplitz matrix with shifted versions of $v^\top G^{(2)}$ as its rows, where the nonzero portion of $v^\top G^{(2)}$ has at most $\Tbar$ consecutive blocks. We use the following elementary property of Toeplitz matrices to bound $\opnorm{\Gsftwov}$.
	\begin{lem}\label{lem:toep_conv} Let $X$ and $Y$ be Toeplitz matrices generated by $x \in \R^{a}$ and $y \in \R^{b}$, where ``generated" can include padding $x$ and $y$ with zeros on either end. Then,
	\begin{align*}
	\opnorm{XY}\leq \|x * y\|_{1} \leq (a+b-1)\twonorm{x}\twonorm{y}\:,
	\end{align*}
	where $*$ denotes convolution,
	\end{lem}
	Taking $X=Y^\top=\Gsftwov$ (and making the appropriate adjustment for the vector case), we see that
	\begin{align*}
	\opnorm{\Gsftwov}^2 = \opnorm{\Gsftwov {\Gsftwov}^\top} \lesssim &\; \Tbar \twonorm{v^\top G^{(2)}}^2\leq \Tbar\twonorm{v}^2 \opnorm{G^{(2)}}^2\:.
	\end{align*}
	\end{proof}
	Furthermore, by H\"older's inequality (with the convention that $\Psi_0:=1$), 
	\begin{align*}
	\opnorm{G^{(2)}} &\; \leq \opnorm{\sum_{k=0}^L \Psi_k[\mathbf{0}_{m \times Tk} \mid \Markov_{T(L-k)}(\Gsf)]} \\
	\leq &\;  \left(1+\sum_{k=1}^L \opnorm{\Psi_k}\right) \max_{j\in[L]} \opnorm{[\mathbf{0}_{m \times Tj} \mid \Markov_{T(L-j)}(\Gsf)]} \\
	\leq &\; (1+\|\phi\|_\blockop) \Markov_{TL}(\Gsf)\:.
	\end{align*}
	Thus, we have shown that $\opnorm{\Gsftwov} \lesssim \sqrt{\Tbar} \twonorm{v}(1+\|\phi\|_\blockop) \Markov_{TL}(\Gsf)$.
	\paragraph{Step 3} As with the first term, we upper bound Talagrand's $\gamma_2$ functional using the operator norm calculation from Step 2.

	Again putting the three steps together and appealing to Proposition~\ref{prop:mendelson_tail}, we see that with probability at least $1-\delta$,
	\begin{align*}
	\opnorm{\Gsftwo \matutotal} \lesssim &\; \left(\sqrt{N} + \sqrt{\Tbar m}  + \sqrt{\Tbar \log(1/\delta)}\right)(1+\|\phi\|_\blockop) \Markov_{TL}(\Gsf)\\
	&\; \lesssim \sqrt{N}(1+\|\phi\|_\blockop) \Markov_{TL}(\Gsf)\:,
	\end{align*}
	by the assumption on $N$.

\subsection{Additional terms \label{sec:other_error_terms}}
\subsubsection{Process noise\label{sec:process_noise}}
In the stochastic noise model, the previous calculations for $\Gsfone$ and $\Gsftwo$ apply directly as the martingale structure allows us to again use Proposition~\ref{prop:mendelson_tail}; namely, the vectors $\matw_{s:t}$ are $1$-subgaussian.

In the adversarial noise model, we can bound~\eqref{eq:key_matrix} directly. First, we see that
\begin{align*}
\opnorm{\Fsfone \matwtotal} = \; \sup_{v \in \calS^{m-1}} \|\Fsfonev\matwtotal\|_2 
\leq &\; \sqrt{N d_w} \sup_{v \in \calS^{m-1}} \opnorm{\Fsfonev} 
\leq \; \sqrt{Nd_w}\mixnorm{N}{\Fsf_\phi}\:,
\end{align*}
by the same arguments as in the preceding section. 
Second, the same strategy outlined for $\Gsftwo$ establishes that 
\begin{align*} \opnorm{\Fsftwo \matwtotal} \leq \sqrt{N \Tbar d_w} (1+\|\phi\|_\blockop) \Mknorm{\Tbar}{\Fst}.
\end{align*}

\subsubsection{Output noise\label{sec:output_noise}}
	Let us first start with the stochastic noise model. The same calculations as in Section~\ref{sec:second_term} hold, except now we are concerned with the matrix
	\begin{align*}
	G^{(2)}_z &:= \sum_{k = 0}^{L} \Psi_{k} \begin{bmatrix}\mathbf{0}_{m \times Tk} \mid  D_z \mid \mathbf{0}_{m \times T(L-k)-1}\end{bmatrix}\:.
	\end{align*}
	As such, we see that $\fronorm{\Gsftwozv} \leq \sqrt{N}\twonorm{v}\opnorm{G^{(2)}_z}$ and $\opnorm{\Gsftwozv}\lesssim \sqrt{L}\twonorm{v}\opnorm{G^{(2)}_z}$. Then,
	\begin{align*}
	\opnorm{G^{(2)}_z} = \opnorm{\sum_{k=0}^L \Psi_k\begin{bmatrix}\mathbf{0}_{m \times Tk} \mid  D_z \mid \mathbf{0}_{m \times T(L-k)-1}\end{bmatrix}} \leq(1+\|\phi\|_{\blockop}) \opnorm{D_z},
	\end{align*}
	so by Proposition~\ref{prop:mendelson_tail} we have
	\begin{align*}
	\opnorm{\Gsftwoz \matztotal} \lesssim &\; \left(\sqrt{N} + \sqrt{Lm}  + \sqrt{L\log(1/\delta)}\right)(1+\|\phi\|_{\blockop}) \opnorm{D_z}\\
	\lesssim &\; \sqrt{N}(1+\|\phi\|_{\blockop}) \opnorm{D_z}\:.
	\end{align*}

	For the adversarial case, we instead have 
	\begin{align*}
	\opnorm{\Gsftwoz \matztotal} \lesssim \; \sqrt{N L d_z}(1+\|\phi\|_{\blockop}) \opnorm{D_z}\:.
	\end{align*}

\subsubsection{Contribution of the initial state\label{sec:initial_state_noise}}
We see that the nonzero contribution of $\matx_1$ to $\matDel_{\phi}$ is given by 
\begin{align*}
\Hsfone \matx_1 =  [C_\phi A^{N-TL-1}\matx_1 \mid C_\phi A^{N-TL-2}\matx_1 \mid \dots \mid C_\phi A^{\None-TL-1}\matx_1]\:.
\end{align*}
Thus, $\opnorm{\Hsfone \matx_1} \leq \Mknorm{N}{\Hsf_{\phi}}$.


\section{Definition of \texorpdfstring{$\Mbarstoch$, $\Mbaradv$}{Mtoch, Madv}, and Proof of Proposition~\ref{prop:opt_to_final_lti}\label{sec:opt_to_final_lti}}

\subsection{Notation\label{sec:Mnotation}} Note that in the body of the text we assumed $\matx_1 = 0$. Here we allow $\matx_1 \neq 0$, and opt for the more general definitions of $\Mbarstoch$. Specifically, we define the constants
\begin{align}
M_0&:= \|S^{-1}\matx_1\|  \nonumber\\
M_{B}(t)&:= \|S^{-1}\Bst\| + \sqrt{t}\|S^{-1}B_w\|, \quad M_B := M_B(1) \nonumber\\
M_{C}&:= \|\Cst S\|_{\op} \nonumber\\
M_{D}(t)&:= \|\Dst\|_{\op} + \sqrt{t}\|D_z\|_{\op},\quad M_{D} := M_D(1)\label{eq:Mconstants}
\end{align}
which inherently reflect the conditioning of the chosen realization of $\Gst$. We define the general version of $\Mbarstoch$ and $\Mbaradv$ that take $\matx_1$ into account.
\begin{defn}\label{defn:magnitude_bound_general} Let $\Ast = S \Jst S^{-1}$ denote the Jordan-normal decomposition of $\Ast$. We let
\begin{align*}
\Mbarstoch &:=  (N^{-1/2}M_0 + M_B)M_C + M_D\\
\Mbaradv &:= (N^{-1/2}M_0 + M_B(Tdd_w))M_C + M_D(dd_z)\:.
\end{align*}
\end{defn}
Note that the above definition reduces to the quantities used in the body of the paper in the case $\matx_1 = 0$. We shall require the following bound, which is a corollary of Proposition~\ref{prop:markovbound} and Lemma~\ref{lem:worst_case_k_bounds} in the next section:
\begin{restatable}[Concrete Markov Bounds]{cor}{corconcretemarkov}\label{cor:concrete_markov} Let $\Ast$ have maximum Jordan block size $k$. Then, for any $n \ge 1$, 
\begin{align*}
\Mknorm{n}{\Gsfst} \lesssim \|\Dst\|_{\op} + \|S^{-1}\Bst\|_{\op}\|\Cst S\|_{\op}\cdot k^{1/2}n^{k - \frac{1}{2}}\:,
\end{align*}where an analogous bound holds for $\Fsfst$ and $\Hsfst$ replacing $S^{-1}\Bst$ by $S^{-1}B_w$ and $S^{-1}\matx_1$, respectively.
\end{restatable}

\subsection{Proof of Proposition~\ref{prop:opt_to_final_lti}} We first begin with the following lemma, which gives a generic bound on $\deff$ and $\Ovfit_{\mu}$ in terms of the quantity
\begin{align*}
\Optil := \opnorm{\matYplus} + \opnorm{\Gst}, \text{ where } \matYplus := [\maty_N | \dots | \maty_1].
\end{align*}
We have the following bound, in terms of an intermediate error quantity $\dbar_{\Optil}$.
\begin{lem}\label{lem:deff_bound} On $\eventU$, we have
\begin{align*}
\deff(\Opt_{\mu} + \Ovfit_{\mu},Lm,\regu) &\lesssim \dbar_{\Optil} := \ptil + Lm\log_+(\Optil) + Lm\log_+(\tfrac{\sqrt{NL}}{\regu^2 }) \\
\Ovfit_{\mu} &\lesssim \min\{N,\sqrt{T(\dbar_{\Optil} + \log \tfrac{1}{\delta})}\}\opnorm{\Gst}\:.
\end{align*}
\end{lem}
We defer the proof of the above lemma until Section~\ref{sec:lem:deff_bound}. Now, recall that by Theorem~\ref{thm:ph_vr_oracle} with $\Lbar \leftarrow Lm$, on the event $\eventU$,
\begin{align*}
\opnorm{\Gph  - \Gst}  \lesssim \frac{N^{-1/2}(\Opt_{\regu} + \Ovfit_{\regu}(\delta) + \mu)}{\sqrt{N}}  \cdot \sqrt{ T\left(\log \tfrac{1}{\delta} + \deff(\Opt_{\regu}+\Ovfit_{\regu}(\delta),\Lbar,\regu) \right)}
\end{align*}
with probability at least $1 - \delta - \delU$. Combining with Lemma~\ref{lem:deff_bound}, we have
\begin{align*}
\opnorm{\Gph  - \Gst}  \lesssim \frac{(\Opt_{\regu} + \sqrt{T(\dbar_{\Optil} + \log \tfrac{1}{\delta})}\opnorm{\Gst} + \mu)}{\sqrt{N}}  \cdot \sqrt{ \frac{T\left(\log \tfrac{1}{\delta} + \dbar_{\Optil} \right)}{N}}~.
\end{align*}
Hence, it suffices to show that with probability $1 - \delta$ (absorbing union bounds into $\log(1/\delta)$), that when $\mu \ge 1$,
\begin{align*}
\dbar_{\Optil} \lesssim \dbar := \ptil +  Lm\left(\log_+ \Mbarstoch + k\log_+ N\right),
\end{align*}
where $k$ is the largest Jordan block of $\Ast$.
For $\mu \ge 1$, using the bound $N \ge L$, we have 
\begin{align*}
\dbar_{\Optil} &= \ptil + Lm\log_+(\Optil) + Lm\log_+(\tfrac{\sqrt{NL}}{\regu^2 }) \\
 &\;\lesssim  \ptil + Lm\log_+(\Optil) + Lm\log_+N.
\end{align*}
Thus, it remains to establish the bound
\begin{align*}
\log_+ \Optil \lesssim \begin{cases} \log_+ \Mbarstoch + k\log_+ N & \text{(stochastic noise)}\\
\log_+ \Mbaradv + k\log_+ N & \text{(adversarial noise)}\:.
\end{cases}
\end{align*}
 We begin this task by bounding the random part of $\Optil$, $\|\matYplus \|_{\op}$. We will do so in terms of the block Toeplitz matrix
 \begingroup
    \renewcommand{\arraystretch}{2.5}
\begin{align*}
    T_N(\Gsfst) &= \begin{bmatrix} \Dst & \Cst \Bst & \Cst \Ast \Bst & \dots & \Cst \Ast^{N-2} \Bst  \\
    \mathbf{0}_{d \times d} &  \Dst & \Cst\Bst  & \ddots & \Cst \Ast^{N-3} \Bst \\
    \vdots &  \ddots & \ddots & \ddots  & \vdots \\
    \vdots  &   & \ddots & \Dst & \Cst\Bst\\
    \mathbf{0}_{d \times d}  &  \cdots & \cdots & \mathbf{0}_{d \times d}  & \Dst
    \end{bmatrix} \in \R^{Nm \times Np}
\end{align*}
\endgroup
and its analog for $T_N(\Fsfst)$, as well as $\Markov_N(\Hsfst)$.
\begin{prop}\label{prop:yplus_bound}
Suppose that $N \ge \max\{m,\log(1/\delta)\}$. With probability at least $1-\delta$, we have that 
\begin{enumerate}[(i)]
    \item (Stochastic model)
    \begin{align*}
        \|\matYplus \|_{\op}  \lesssim &\; \sqrt{N}\left(\opnorm{T_N(\Gsfst)} + \opnorm{T_N(\Fsfst)} + \opnorm{D_z}\right) +\Mknorm{N+1}{\Hsfst}\\
        \le &\; \sqrt{N}\left(\sqrt{N}\Mknorm{N}{\Gsfst} + \sqrt{N}\Mknorm{N}{\Fsfst} + \opnorm{D_z}\right) + \Mknorm{N+1}{\Hsfst}
    \end{align*}
    \item (Adversarial model) $\|\matYplus \|_{\op}  $ is bounded as 
    \begin{align*}
         \lesssim &\; \sqrt{N}\left(\opnorm{T_N(\Gsfst)} +   \sqrt{d_w}\opnorm{T_N(\Fsfst)} +  \sqrt{d_z}\opnorm{D_z}\right) + \Mknorm{N+1}{\Hsfst}\\
         \le &\; \sqrt{N}\left(\sqrt{N}\Mknorm{N}{\Gsfst} +   \sqrt{d_w}\sqrt{N}\Mknorm{N}{\Fsfst}+  \sqrt{d_z}\opnorm{D_z}\right) + \Mknorm{N+1}{\Hsfst}
    \end{align*}
\end{enumerate}
\end{prop}
\begin{proof} The argument mirrors those in Appendix~\ref{sec:error_calcs}.  The dependence on $\Hsfst$ is through the Markov operator $\Mknorm{N}{\Hsfst}$ due to the argument given in Section~\ref{sec:initial_state_noise}. For the other terms, we bound the $\Gsf$ term as a representative example. Letting $\Ysfu$ being the operator that maps $\matutotal$ to $\matYplus$, the key is to again note that $\Ysfuv = (I\otimes v^\top) T_N(\Gsfst)$. Thus, by the now-standard arguments, one can show that
\begin{align*}
\sup_{v\in\sphere{m}}\opnorm{\Ysfuv} \leq \opnorm{T_N(\Gsfst)}\quad\text{and}\quad\sup_{v\in\sphere{m}}\fronorm{\Ysfuv} \leq \sqrt{N}\opnorm{T_N(\Gsfst)}\:.
\end{align*}
Proposition~\ref{prop:mendelson_tail} along with the simplification $N \ge \max\{m,\log(1/\delta)\}$ then gives the desired bounds. Bounding $\opnorm{T_N(\Gsfst)} \le \sqrt{N}\Mknorm{N}{\Gsfst}$ follows by considering each row of $T_N$ separately and then noting that each row of $T_N$ is a submatrix of the first.

\end{proof}
We now bound $ \Optil = \|\matY_+ \|_{\op} + \sqrt{N}\|\Gst\|_{\op}$ in the stochastic case using Proposition~\ref{prop:yplus_bound} and Corollary~\ref{cor:concrete_markov}:
\begin{align*}
\|\matY_+ \|_{\op} +&\; \sqrt{N}\|\Gst\|_{\op} \\
\overset{\mathclap{\text{Prop}.~\ref{prop:yplus_bound}}}{\lesssim} &\;\; 
\sqrt{N}\left(\sqrt{N}\opnorm{\Markov_N(\Gsfst)} +  \sqrt{N}\opnorm{\Markov_N(\Fsfst)} + \opnorm{D_z} + \|\Gst\|_{\op}\right) + \Mknorm{N+1}{\Hsfst} \\
\overset{(i)}{\leq} &\;\;  \sqrt{N}\left(2\sqrt{N}\opnorm{\Markov_N(\Gsfst)}  + \sqrt{N}\opnorm{\Markov_N(\Fsfst)} + \opnorm{D_z}\right) + \Mknorm{N+1}{\Hsfst} \\
\le &\;\;  2N\left(\opnorm{\Markov_N(\Gsfst)}  + \sqrt{N}\opnorm{\Markov_N(\Fsfst)} + N^{-1/2}\Mknorm{N+1}{\Hsfst}+  \opnorm{D_z}\right) \\
\overset{\mathclap{\text{Cor.}~\ref{cor:concrete_markov}}}{\lesssim} &\;\; 2N\left(M_CM_B N^{k-1/2} + N^{-1/2}\cdot M_0 M_C N^{k-1/2} + M_D\right)\\
\le &\;\; 2N^{k+1/2}\left(M_CM_B +   N^{-1/2} M_0  +  M_D\right)
\quad \overset{\mathcal{\text{Defn.}}~\ref{defn:magnitude_bound_general}}{=} 2N^{k+1/2}\Mbarstoch\;, 
 \end{align*}
where $(i)$ uses $\opnorm{\Gst} = \Mknorm{T}{\Gsfst}\leq \Mknorm{N}{\Gsfst}$ as $N \geq T$, and where in the penultimate, last line $M_1$, $M_0$, and $M_D$ are the constants in~\eqref{eq:Mconstants}. Hence, we have that
\begin{align*}
\log_+ \Optil &\lesssim \log_+ 2N^{2k+1}\Mbarstoch \lesssim k \log N + \log_+ \Mbarstoch
\end{align*}
The adversarial case is analogous, where we replace $M_B \leftarrow M_B(d_w)$ and $M_D \leftarrow  M_D(d_z)$, which yield the extra factors of $\sqrt{d_w}$ and $\sqrt{d_z}$.

\subsection{Proof of Lemma~\ref{lem:deff_bound}\label{sec:lem:deff_bound}}
\begin{proof}
Recall the definition
\begin{align*}
\deff(\Opt,\Lbar,\mu) &:= \ptil + m+ \lil \tfrac{{\Opt}}{\mu} + \Lbar\log_+(\Opt + \tfrac{\sqrt{N}\|\matK\|_{\op}}{\regu^2 }),
\end{align*}
which we bound with $\Opt \leftarrow \Opt_{\mu} + \Ovfit_{\mu}$ and $\Lbar \leftarrow Lm$. Since $\Ovfit_{\mu} \le \sqrt{N} \opnorm{\Gst}$, we can write
\begin{align*}
\Opt_{\mu} + \Ovfit_{\mu} &\le \min_{\phi} \|\matDel - \phi \matK\|_{\op} + \mu \|\phi\|_{\op} + \sqrt{N}\opnorm{\Gst} \\
& \leq \|\matDel \|_{\op}  + \sqrt{N}\opnorm{\Gst} \quad (\text{taking}~\phi = 0)\\
&\le \|\matDel + \Gst\matUbar \|_{\op}  + \opnorm{\Gst} \opnorm{\matUbar} + \sqrt{N}\opnorm{\Gst}\\
&\lesssim \|\matY \|_{\op}+ \sqrt{N}\opnorm{\Gst}  \quad (\text{on } \eventU)\\
&\le \|\matYplus \|_{\op} + \sqrt{N}\opnorm{\Gst} := \Optil.
\end{align*}
Moreover, since $\matK$ consists of $L$ submatrices, each of which is a submatrix of $\matYplus$, we see that $\|\matK\|_{\op} \le \sqrt{L}\|\matYplus \|_{\op}$. Lastly, recall $\log_+ x \ge 1$.  Therefore, we can bound
\begin{align*}
\deff(\Opt_{\mu} + \Ovfit_{\mu},Lm,\mu) 
&\lesssim \ptil + \lil \tfrac{{\Optil}}{\mu} + Lm\log_+(\Optil + \tfrac{\sqrt{N}\|\matK\|_{\op}}{\regu^2 }) \quad (\text{absorbing } m)\\
&\le \ptil +  \lil \tfrac{{\Optil}}{\mu} + Lm\log_+(\Optil + \tfrac{\sqrt{NL}\|\matYplus \|_{\op}}{\regu^2 })\\
&\le \ptil + \lil \tfrac{{\Optil}}{\mu} + Lm\log_+(\Optil (1 + \tfrac{\sqrt{NL}}{\regu^2 }))\\
&\lesssim \ptil + \lil \tfrac{{\Optil}}{\mu} + Lm\log_+(\Optil) + Lm\log_+(\tfrac{\sqrt{NL}}{\regu^2 }).
\end{align*}
Lastly, we observe that since $\lil$ is submultiplicative, $\lil \tfrac{{\Optil}}{\mu} \lesssim \lil \Optil + \lil \frac{1}{\mu} \leq \log_+ \Optil + \log_+ \frac{1}{\mu} \leq Lm(\log_+ \Optil + \log_+ \frac{1}{\mu} )$. Partially absorbing this $\lil$ term, we find that
\begin{align*}
\deff(\Opt_{\mu} + \Ovfit_{\mu},Lm,\mu) \lesssim \ptil + Lm\log_+(\Optil) + Lm\log_+(\tfrac{\sqrt{NL}}{\regu^2 }) + Lm\log_+ \frac{1}{\regu}.
\end{align*}
Now, note that if $\regu \ge 1$, $Lm\log_+ \frac{1}{\regu} = Lm \le Lm\log_+(\tfrac{\sqrt{NL}}{\regu^2 })$. On the other hand, if $\regu \le 1$, then $\tfrac{\sqrt{NL}}{\regu^2 } \ge \frac{1}{\regu}$, and thus $Lm\log_+ \frac{1}{\regu} \le Lm\log_+(\tfrac{\sqrt{NL}}{\regu^2 })$. In either case, $Lm\log_+ \frac{1}{\regu} \leq Lm\log_+ \frac{\sqrt{NL}}{\regu^2}$, so
\begin{align*}
\deff(\Opt_{\mu} + \Ovfit_{\mu},Lm,\mu) \lesssim \ptil + Lm\log_+(\Optil) + Lm\log_+(\tfrac{\sqrt{NL}}{\regu^2 })\:.
\end{align*}

For the second bound of Lemma~\ref{lem:deff_bound}, we use the fact that if $\Lambda$ is a matrix of rank $Lm$, 
\begin{align*}
\log \det( I + \Lambda)^{1/2} = \sum_{i=1}^{Lm} \log \sqrt{1 +  \lambda_i(\Lambda) } &\le \sum_{i=1}^{Lm} \log \sqrt{1 +  \opnorm{\Lambda} } \lesssim Lm \log_+ \opnorm{\Lambda}. 
\end{align*}
Applying the above equation with $\Lambda = \regu^{-2}\matK \matK^{\top}$, we find that
\begin{align*}
\Ovfit_{\regu}(\delta) &:= \opnorm{\Gst} \cdot \min\left\{\sqrt{N},\sqrt{T}\sqrt{\log\tfrac{1}{\delta} + \ptil + \log \det(I + \regu^{-2}\matK \matK^{\top})^{1/2}  }\right\} \\
&\lesssim \opnorm{\Gst} \cdot \min\left\{\sqrt{N},\sqrt{T}\sqrt{\log\tfrac{1}{\delta} + \ptil + Lm\log_+ \frac{\opnorm{\matK}^2}{\regu^2}}\right\}\\
&\le \opnorm{\Gst} \cdot \min\left\{\sqrt{N},\sqrt{T}\sqrt{\log\tfrac{1}{\delta} + \ptil + Lm\log_+ \frac{L\opnorm{\matYplus}^2}{\regu^2}} \right \}\\
&\le \opnorm{\Gst} \cdot \min\left\{\sqrt{N},\sqrt{T}\sqrt{\log\tfrac{1}{\delta} + \ptil + Lm\log_+\opnorm{\matYplus}^2 + Lm\log_+\tfrac{L}{\regu^{2}} }\right\},
\end{align*}
from which the result follows by taking $Lm\log_+\opnorm{\matYplus}^2 \lesssim Lm \log_+ \opnorm{\matYplus} \le Lm \log_+ \Optil $ and $ Lm\log_+\frac{L}{\regu^{2}} \le Lm \log_+ \frac{\sqrt{NL}}{\regu^2}$, as $N \ge L$.
\end{proof}

\subsection{Selecting the parameter \texorpdfstring{$L$}{L}}\label{app:selecting_L}
In this section, we give an informal discussion of how to select the parameter $L$. 
Observe that the confidence bounds from Theorem~\ref{thm:ph_vr_oracle} are \emph{almost} data-dependent, but in fact depends on the quantity $\Opt_{\mu} + \Ovfit_{\mu}$, which is not known to the learner. In order to select $L$, one shall need to replace these quantities with data-dependent ones, and then use a standard procedure (e.g. structural risk minimization) to tune $L$. First, considering $L$ fixed, define the following empirical proxy for $\Opt_{\mu} + \Ovfit_{\mu}$,
\begin{align*}
\Opthat_{\mu} := \opnorm{\matY - \matK \predrid} + \matu\opnorm{\predrid}.
\end{align*}
Our first main result of this section is that
\begin{align*}
\Opthat_{\mu} \gtrsim  \opnorm{\Gst}\rootn + \Opt_{\mu}
\end{align*}
with high probability. Formally, our guarantee is
\begin{prop}\label{prop:L_oracle} For constants $C_1,C_2$, suppose that there exists an $K  \in \{e,e^2,\dots\}$ such that $L$ and $N$ satisfy
\begin{enumerate}[(1)]
	\item $N  \ge C_1T(\ptil +  m + Lm K + \log 1/\delta) $
	\item $K \ge \log \regu\opnorm{\predrid}$,  and either 
	\item $\text{(a):}~~K \ge \log \frac{C_2\opnorm{\matK}}{\regu \rootn\opnorm{\Gst}} \quad \text{or} \quad \text{(b):}~~ K  \ge \log \frac{C_2\opnorm{\matK}}{\regu \Opthat}$\:.
\end{enumerate}
Then, with probability $1-2\delta$, whenever $\eventU$ holds,
\begin{align*}
\Opthat_{\mu} \gtrsim  \Opt_{\mu} + \opnorm{\Gst}\rootn\:.
\end{align*}
\end{prop}
Note that condition (3a) gives a condition which is more amenable to analysis (requiring only a lower bound on $\opnorm{\Gst}$), whereas condition (3b) can be evaluated by the learner, as $\matK$ and $\Opthat$ are empirical quantities.
We sketch the proof of the above result in the following subsection. Now, note that by the triangle inequality, we also have
\begin{align*}
\Opthat_{\mu} \lesssim  \opnorm{\Gst}\rootn + \Opt_{\mu},
\end{align*}
and moreover, $ \opnorm{\Gst}\rootn  \le \Ovfit_{\mu}$ by definition. Hence we can use structural risk minimization~\cite{shawe1998structural} to select $L$. 

To sketch this approach, let $\Opthat_{\mu;L}$, $\matK_{L}$, $\Opt_{\mu}^L$, $\Gph^L$, and $\Nmin(L)$ be the corresponding quantities defined for a given $L$. We define the set
\begin{align*}
\calS(\delta) := \left\{L \in \N: C_1T(\ptil +  m + Lm \left(\log_+ \tfrac{C_2\opnorm{\matK_{L}}}{\regu \Opthat_{L}} + \log_+ \regu\opnorm{\predrid^{L}}  \right) + \log \tfrac{L}{\delta}) \le N\right\},
\end{align*}
which represents the set of admissible lengths $L$ for which  Proposition~\ref{prop:L_oracle} guarantees that $\Opthat_{\mu;L}$ is a good proxy for $\Opt_{\mu} + \Ovfit_{\mu}$. We may then select $\Lhat$ as 
\begin{align*}
\Lhat \in \argmin\{ \Conf(L,\mu) : L \in \calS(\delta)\},
\end{align*}
where we have defined the upper confidence bound
 \begin{align*}
\Conf(L,\mu) := \frac{N^{-1/2}(\Opthat_{\mu;L} + \mu)}{\sqrt{N}}  \cdot \sqrt{ T\left(\log \tfrac{1}{\delta} + \deff(\Opthat_{\mu;L},Lm,\regu) \right)}.
\end{align*}
We can briefly analyze the outcome using a sketch of arguments similar to those in Section~\ref{sec:opt_to_final_lti}. Denote the $L$-indexed dimension quantity from Proposition~\ref{prop:opt_to_final_lti},
\begin{align*}
\dbar(L) := \ptil +  Lm\left(\log_+ \Mbarstoch + k\log_+ N + \log_+\right) = \BigOhTil{p + Lmk}.
\end{align*}
It can be shown that for an appropriate constant $C_2$, with probability $1-\delta$ on $\eventU$, it holds that
\begin{align*}
\left\{ L : N \ge C_2(T(\dbar(L) + Lm \log_+\opnorm{\Gst}^{-1} + \log\tfrac{1}{\delta}))\right\} \subseteq  \calS(\delta). 
\end{align*}
This can in turn be used to establish the following analogue of Proposition~\ref{prop:opt_to_final_lti}, whose proof we omit.
\begin{prop} Fix a $\delta \in (0,1)$, and $T,~\Lmax \in\N$.
Suppose that $N \ge \Nmin(\Lmax)$, $N_1 = T\Lmax$, $\rho(\Ast) \le 1$, and that the largest Jordan block of $\Ast$ is of size $k$. Then, once
\begin{align*}
 N \ge C_2(T(\dbar(\Lmax) + \Lmax m \log_+\opnorm{\Gst}^{-1} + \log\tfrac{1}{\delta}),
\end{align*}
the estimator $\Gph^{\Lhat}$, where $\Lhat$ is selected in the manner described above, satisfies with probability at least $1-\delta - (2Np)^{-\log^2 (2Tp)\log^2(2Np)}$ in the stochastic noise model that
\begin{align*}
\opnorm{\Gph^{\Lhat}  - \Gst}  \lesssim \min_{L \in \{0,1,\dots,\Lmax\}}   \left(\frac{\Opt_{\regu}^L +\rootn \opnorm{\Gst}  + \mu}{\sqrt{N}}\right) \cdot \sqrt{ \frac{T(\dbar(L) +   \log \tfrac{L}{\delta})}{N} }.
\end{align*}
In the adversarial noise model, we instead take $\dbar(L) := \ptil +  Lm\left(\log_+ \Mbarstoch  + \log_+ (d_z + d_w) + k\log_+ N\right)$.
\end{prop}
We remark that the parameter $\Lmax$ in the above proposition appears merely in the analysis. Moreover, one can also search $L$ in powers binary powers $2^i$ for added computational efficiency.

\subsubsection{Proof Sketch of Proposition~\ref{prop:L_oracle}}
By the reverse triangle inequality,
\begin{align*}
\Opthat_{\mu}  ~=~ \|\matY - \matK \predhat \|_{\op} + \mu \|\predhat\|_{\op} ~\ge~ \|\matDel - \matK \predhat \|_{\op} + \mu \|\predhat\|_{\op} -\opnorm{\Gst} \ge \Opt_{\mu} -\opnorm{\Gst}.
\end{align*}
Thus if $\Opthat_{\mu} \le  \Opt_{\mu} + \opnorm{\Gst}\rootn$, then $\Opt_{\mu} \lesssim \opnorm{\Gst}\rootn$. Hence, it suffices to show  that $\Opthat_{\mu} \gtrsim \opnorm{\Gst}\rootn$ with probability $1 - \delta$ on $\eventU$. Moreover, it suffices to prove the theorem if (3a) holds; indeed, if (3b) holds, then either $\Opthat \le \opnorm{\Gst}\rootn$, it which case (3a) holds, or $\Opthat \ge \opnorm{\Gst}\rootn$ as desired.

Fix a $v \in \sphere{m}$ for which $\opnorm{\Gst} = \twonorm{v^\top \Gst}$. In this simplified setting, we show the following lemma.
\begin{lem}\label{case:misone} Fix $v \in \sphere{m}$. Then with,  probability $1-\delta$, for any $j \in \N$  satisfying
\begin{align*}
e^j \ge \log\left(\frac{8\opnorm{\matK}}{\regu \twonorm{\Gst v}}\right) \quad \text{and}\quad  e^j \ge \log \regu\twonorm{\predrid v },
\end{align*}
where $N  \gtrsim T(\ptil +  m + Lm e^j + \log 1/\delta) $, then $
\| (\matY -  \matK\predrid)v\|_2  \ge \frac{1}{4}\twonorm{v^\top\Gst  }. $
\end{lem}
This implies Proposition~\ref{prop:L_oracle}, since $e^j \ge \log \regu\opnorm{\predrid}$ implies $e^j \ge \log \regu\twonorm{\predrid^\top v }$, and $v$ was chosen so that $\twonorm{v^\top\Gst  } = \opnorm{\Gst}$.
We now prove the above lemma:

\begin{proof}
We shall use the following intermediate lemma, which we prove following the proof of Lemma~\ref{case:misone}. 
\begin{lem}\label{lem:casemisonelem} Suppose $m = 1$. Then for any fixed $\pred$ and any $N \gtrsim T(\ptil + \log \tfrac{1}{\delta})$, it holds with probability at least $1-\delta$ on $\eventU$ that $\| (\matY -  \matK\pred)v\|_2  \ge \frac{1}{2\sqrt{2}}\|\Gst\|\sqrt{N}$.
\end{lem}
Next, we mirror the proof of Theorem~\ref{thm:ph_vr_general}, For $j \ge 1$, let $c_j = e^{e^{j}}$ and let $\calT_{j}$ denote a $1/c_j$ net of the set $c_j \ball{\R^{\Leff}}/\regu$ in the norm $\|\cdot\|_2$. Following the computations in that proof, it holds with probability at least $1-\delta$ that, for all $j$ which satisfy
\begin{align*}
N \gtrsim T(\ptil + Lm e^j + \log 1/\delta),
\end{align*}
it holds that for all $\predtil \in \calT_j$ that
\begin{align*}
\|\matY - \matK \predtil^\top\| \ge \frac{1}{2\sqrt{2}}\|\Gst\|\sqrt{N}.
\end{align*}
Hence, if for some $\Delta > 0$ we have
\begin{align*}
e^j \ge \log\left(\frac{\opnorm{\mu^{-1}\matK}}{\Delta}\right),
\end{align*}
and if $\predrid \in c_j \ball{\R^{\Leff}}/\regu$, we have with probability $1 - \delta$ on $\eventU$,
\begin{align*}
\|\matY - \matK \predtil\| &\ge~ \|\matY - \matK \predtil\| - \opnorm{\matK}\|\predrid - \predtil\|_{2} \\
&\ge  \|\matY - \matK \predtil^\top\| - \Delta ~\ge ~ \frac{\|\Gst\|\sqrt{N}}{2\sqrt{2}}  - \Delta.
\end{align*}
Finally, setting $\Delta = \frac{\|\Gst\|\sqrt{N}}{8}$ yields the desired bound of Proposition~\ref{prop:L_oracle}.
\end{proof}
\begin{proof} [Proof of Lemma~\ref{lem:casemisonelem}]
Since $m = 1$, we work with the $2$-norm. Here,
\begin{align*}
\|\matY - \matK_{L}\pred^\top \|_{2}^2 &=  \|\matDel + \matUbar\Gst^\top- \matK\pred^\top\|_{2}^2 \\
&=  \|\matUbar\Gst^\top\|_{2}^2 + \|\matDel - \matK\pred^\top\|_{2}^2  + 2\langle  \matUbar\Gst^\top, \matDel_{\phi} \rangle^2\\
&= \frac{1}{2}\| \matU\Gst^\top\|_{2}^2 + \frac{1}{2}\|\matDel_{\phi}\|_{2}^2 \quad (:= \frac{1}{2}T_1)\\
&\; \quad+ \frac{1}{2}\| \matUbar\Gst^\top\|_{2}^2 + \frac{1}{2}\| \matDel_{\phi}\|_{2}^2 + 2\langle  \matUbar\Gst^\top, \matDel_{\phi} \rangle^2 \quad(:= T_2).
\end{align*}
We shall show that with probability $1 - \delta$ on $\eventU$, the term $T_2$ is nonnegative. This suffices since on $\eventU$, $\frac{1}{2}T_1 \ge \frac{1}{4}(\|\Gst\|_2\rootn)^2$. To show that $T_2$ is nonnegative with high probability,
we may assume that $ \Gst \ne 0$, for otherwise this holds trivially. On $\eventU$, we have that $\| \matUbar\Gst^\top\|_{2} \ge \frac{1}{\sqrt{2}}\sqrt{N}\| \Gst\| := \alpha$. Now, set $\beta = \| \matDel_{\phi}\|_{2}$. We may also assume that $\beta \le 16 \alpha$, since otherwise 
\begin{align*}
& \frac{1}{2}\| \matUbar\Gst^\top\|_{2}^2 + \frac{1}{2}\| \matDel_{\phi}\|_{2}^2 + 2\langle \matU\Gst^\top, \matDel_{\phi} \rangle^2 \ge \frac{\alpha^2}{2} + \frac{\beta^2}{2} - \beta\| \Gst\|\|\matU\|_{\op} ~\ge~ \frac{\alpha^2}{2} + \frac{\beta^2}{2} - 2\beta \alpha \: (\text{on }\eventU),
\end{align*}
which is nonnegative for $\beta \ge 16 \alpha$. Next, by Theorem~\ref{thm:chain_semi_par} with $\kappa \leftarrow \alpha$, $\matDel\leftarrow \matDel_{\phi}$, $\opnorm{\matDel_{\phi}} \leftarrow \beta$, and $m = 1$, we have with probability $1 - \delta$ that
\begin{align*}
 \twonorm{\matDel^\top\matU} \lesssim T^{1/2}(\beta + \alpha)\sqrt{\ptil + \log \tfrac{1}{\delta} + \lil (\tfrac{\beta}{\alpha})}\:.
\end{align*}
Using the fact that $\beta \le 16 \alpha$, that the above bound is at most $ CT^{1/2}\alpha \sqrt{\ptil + \log \tfrac{1}{\delta} }$ for a universal constant $C$. Noting that $\|\Gst\| = \sqrt{2/N}\alpha$, we find that with probability $1-\delta$ on $\eventU$,
\begin{align*}
T_2(v) &\ge \frac{\alpha^2}{2} -  2\langle  \matU\Gst^\top,\matDel_{\phi} \rangle^2 ~\ge~ \frac{\alpha^2}{2} -  \twonorm{\Gst} \twonorm{\matDel_{\phi}^\top \matU}\\
&\ge \frac{\alpha^2}{2} -  \sqrt{2/N}\alpha \cdot CT^{1/2}\alpha \sqrt{\ptil + \log \tfrac{1}{\delta} } = \alpha^2 \left(\frac{1}{2} - C\sqrt{\frac{2T(\ptil + \log \tfrac{1}{\delta})}{N}}  \right),
\end{align*}
which is nonnegative as soon as $N \ge C'T(\ptil + \log \tfrac{1}{\delta})$ for some universal constant $C'$.
\end{proof}


\section{Polynomial Approximations and Phase Rank\label{sec:poly_approx_main}}

In this section, we present demonstrate how to bound $\Opt_{\mu}$ using the $(\alpha,T)$ phase rank of $\Ast$. Our bounds will be in terms of the $M_{(\cdot)}$-constants defined in Section~\ref{sec:Mnotation} above. 
\begin{enumerate}
	\item Section~\ref{sec:poly_approx_results} presents our main findings in terms of two types quantities: $K_1(\cdots)$ captures the ``complexity'' of a polynomial required to cancel out large dynamical modes of $\Ast$ (formalized in Proposition~\ref{prop:main_approx_K_theorem}), and $K_2(\cdot)$ describes the growth rate of finite-length Markov parameter matrices (Proposition~\ref{prop:markovbound}). We also present findings based on worst-case upper bounds on $K_1(\cdots)$ and $K_2(\cdot)$, via Lemma~\ref{lem:worst_case_k_bounds}. Specifically,
		\begin{enumerate}
			\item Theorems~\ref{thm:big_stoch} and ~\ref{thm:big_adv} present refined bounds for the stochastic and adversarial noise models respectively, which are simplified into Proposition~\ref{prop:phase_rank_intro} for stochastic noise, and Corollary~\ref{cor:simple_adv} for adversarial noise. 
			\item We also derive Theorem~\ref{thm:big_stoch_disent}, and a simplified consequence Corollary~\ref{cor:stoch_disentangle}, under the condition that the modes of $\Ast$ can be ``disentangled'', even when observed through the matrix $\Cst$.
		\end{enumerate}

	\item Section~\ref{sec:poly_approx} introduces the main technical tools which . The idea is to show the existence of filters $\phi$ for control-theoretic norms of the systems $\Gsf_{\phi},\Fsf_{\phi},\Hsf_{\phi}$ defined in \eqref{eq:phi_LTIs} can be bounded by a quantity $\complexq_f(\dots)$, which roughly describes how well a polynomial of bounded degree and coefficient magnitude can ``cover'' a certain set of poles in the complex plane. We focus on scalar filters (Section~\ref{sec:block_scalar_predictors}), and discuss possibly sharper bounds for richer, non-scalar filters (Section~\ref{sec:poly_approx_no_scalar_pred}). Section~\ref{sec:block_scalar_predictors} also includes Proposition~\ref{prop:H_bound}, which bounds the covering-like quantity $H(\dots)$ with the more transparent $K_1(\alpha,d,T)$.
	\item Finally, Section~\ref{sec:proof_big_stoch} gives a proof of Theorem~\ref{thm:big_stoch} for stochastic noise; this proof only depends on results stated in the first section of this appendix,~\ref{sec:poly_approx_results}.
	\end{enumerate}
	With the exception of the proof Theorem~\ref{thm:big_stoch}, all further proofs are deferred to Appendix~\ref{sec:supporting_proofs_poly}. 

\subsection{Main Results\label{sec:poly_approx_results}}
We start by presenting our main results for the stochastic case. Let $\Ast = S\Jst S^{-1}$ denote the Jordan decomposition of $\Ast$. Since $\algspec(\Ast) = \algspec(\Jst)$, we shall use the two interchangeably. 

\subsubsection{\texorpdfstring{$K_1$ and $K_2$}{K1 and K2}: Controlling Poles and Markov Operator Norms\label{sec:kone_ktwo}}

	We begin by introducing two central quantities. First, we introduce a term $K_1(d,T,\alpha,q)$ which reflects how well a $d$-length linear filter can predict observations of the Jordan-normal linear system $\Jst$ when it has $(\alpha,T)$ phase rank $d$. Here, prediction is defined by the $\Mknorm{\infty}{\cdot}$ and $\Hinf$ norms, indexed by $q \in \{2,\infty\}$, respectively. Formally, we define
	\begin{align*}
	K_1(d,T,\alpha,q) &:= \max_{(\lambda,k)\in \algspec(\Ast)} k^2\chq \begin{cases} 0 & |\lambda| = 1\\
	(T(1+\alpha))^{k-\frac{\I(q = 2)}{2}}  2^{d-k} & |\lambda| \in (1 - \frac{1}{T(1+\alpha)},1] \\
	 \frac{2^d}{(1 - |\lambda|)^{k-\frac{\I(q = 2)}{2}}} & |\lambda| <  1 - \frac{1}{T(1+\alpha)} 
	\end{cases},\\
	&\text{where}\quad \chinf :=1 \quad \text{and } \quad \chtwo := \sqrt{1 + \frac{2}{\pi}}.\notag
	\end{align*}
	Taking $q = \infty$ correspond to the $\Hinf$-norm, whereas taking $q = 2$ corresponds to the norm $\opnorm{\Markov_{\infty}(\cdot)\}}$. This notation is because $\Mknorm{\infty}{\Gsfst}$ can be rendered as a 
	norm we call $\Htwoop$, defined in ~\eqref{eq:Htwoopdef}, on the transfer function $\Gsf(z)$. This norm is similar to but slightly sharper than the standard $\Htwo$-norm in control theory; see Chapter 4 of~\cite{zhou1996robust} for a discussion on transfer function norms, including $\Htwo$. In Section~\ref{sec:block_scalar_predictors}, we then show the following bound.
	\begin{prop}\label{prop:main_approx_K_theorem} Suppose $\Ast$ has $(\alpha,T)$-phase rank $d$. Then, there exists a filter $\phi \in \R^{m \times dm}$ with $1+\|\phi\|_{\blockop} \le 2^d$ such that 
	\begin{align*}
	\|\Markov_{\infty}(\Gsf_{\phi})\|_{\op} &\le \|S^{-1}\Bst\|_{\op}\cdot  \|\Cst S\|_{\op}K_1(d,T,\alpha,2)\\
	\|\Gsf_{\phi}\|_{\Hinf} &\le \|S^{-1}\Bst\|_{\op}\cdot  \|\Cst S\|_{\op}K_1(d,T,\alpha,\infty).
	\end{align*}
	Analogous bounds hold for $\Fsf_{\phi}$ and $\Hsf_{\phi}$  where $ \|S^{-1}\Bst\|_{\op}$ is replaced by $\|S^{-1}B_w\|_{\op}$ and $\|S^{-1}\matx_1\|_{\op}$, respectively.
	\end{prop}

	Second, we (somewhat tediously) define a term $K_2(N)$,
	\begin{align}
	\widetilde{\Mcomplextwo}(k,N):= &\; \begin{cases} N^{1/2} & k = 1\\
	N^{k - 1/2}\left(\frac{e}{k-1}\right)^{k-1}  & 2 \le k \le N + 1 \\
	N^{1/2}2^N & k \geq N + 1
	\end{cases}\label{eq:Mcomplextwotil}\\
	\Mcomplextwo(k,\lambda,N) = &\; \begin{cases}
	\frac{k}{(1-|\lambda| )^{k-\frac{1}{2}} }  \wedge \widetilde \Mcomplextwo(k,N)& 0 \le |\lambda| < 1\\
	\widetilde \Mcomplextwo(k,N)& |\lambda| = 1\:,
	\end{cases} \label{eq:Mcomplextwo}.\\
	\Ktwotwo(N) &:= \max_{(\lambda,k)\in \algspec(\Ast)} \Mcomplextwo(k,\lambda,N) \label{eq:K_2(N)}.
	\end{align}
	The term $K_2(N)$ describes the entire magnitude of an length-$N$ trajectory generated by the Jordan-normal linear system $\Jst$. Indeed, in Section~\ref{sec:prop:finite_length}, we prove
	\begin{restatable}[Bound on Magnitude of Markov Parameters]{prop}{propmarkovbound} \label{prop:markovbound}
	Consider a dynamical system of the form $\Gsf = (\Ast,B,C,D)$, where $ \Ast = S\Jst S^{-1}$ is in Jordan normal form. 
	Then, for all $n \ge 1$,
	\begin{align*}
	\opnorm{\Markov_{n}(\Gsf)} \le \opnorm{\Markov_{n+1}(\Gsf)} \le \opnorm{D} + \opnorm{S^{-1}B}\opnorm{CS}\Ktwotwo(n).
	\end{align*}
	\end{restatable}
	It is immediate to then check that $K_1$ and $K_2$ admit the following worst case bounds:
	\begin{restatable}[Worst-Case Bounds on $K_1,K_2$]{lem}{worstcaseKlem}\label{lem:worst_case_k_bounds} Suppose that $\Ast$ has $\rho(\Ast) \le 1$,  and has largest Jordan block of size $k$, and has $(\alpha,T)$-phase rank at most $d \ge k$ for some $\alpha \ge 1$. Then, for all $n \geq 1$,
	\begin{align*}
	K_1(d,T,\alpha,q) \lesssim &\; k^2 (T(1+\alpha))^{k - \frac{\I(q = 2)}{2}} 2^{d}\\
	K_2(n) \le &\; e n^{k - \frac{1}{2}}\:.
	\end{align*}
	\end{restatable}

\subsubsection{Main Results: Stochastic Noise with Block-Scalar Filters \label{sec:main_results_stochastic}}

	We begin by presenting bounds on $\Opt_{\mu}$ for stochastic noise that arise from considering block-scalar filters of the form
	\begin{align*}
	\phitil = [\phi \mid \mathbf{0}] \in \R^{m \times Lm}, \text{where }\phi = [f_1 I_m \mid f_2 I_m \mid \dots f_d  I_m]\in \R^{m \times dm}.
	\end{align*}
	We state two bounds: first, a theorem in terms of the more precise bounds $K_1(N)$ and $K_2(N)$, and then a corollary which applies the bounds from Lemma~\ref{lem:worst_case_k_bounds}, which is proved in Section~\ref{sec:proof_big_stoch}:
	\begin{thm}[Bounds for Stochastic Noise]\label{thm:big_stoch} Suppose that $\Ast$ has $(\alpha,T)$ phase rank at most $1 \le d \le L$.  Then, for any $\delta \in (0,1)$ and $N \geq Td\max\{m,\log(1/\delta)\}$, it holds with probability $1-\delta$ that
	\begin{align*}
	N^{-1/2}\Opt_{\mu} \lesssim &\; M_C(M_B + N^{-1/2}M_0)  K_1(d,T,\alpha,2)\\
	&\; + \min\{ K_1(d,T,\alpha,2),  N^{-1/2}K_1(d,T,\alpha,\infty)\}M_C M_B\sqrt{m+\log(1/\delta)}\\
	&\; +  2^d \left(M_CM_B K_2(Td)+ M_D + \mu N^{-1/2}\right)\:.
	\end{align*}
	\end{thm}
	By replacing the above bounds with worst case bounds from Lemma~\ref{lem:worst_case_k_bounds}, we obtain the bound
	\begin{align*}
	N^{-1/2}\Opt_{\mu} \lesssim &\; M_C(M_B + N^{-1/2}M_0)   k^2 (T(1+\alpha))^{k - \frac{1}{2}} 2^{d}\\
	&\; +  N^{-1/2} k^2 (T(1+\alpha))^{k} 2^{d} M_C M_B\sqrt{m+\log(1/\delta)}\\
	&\; +  2^d \left(M_CM_B (Td)^{k- \frac{1}{2}}+ M_D + \mu N^{-1/2}\right).
	\end{align*}
	For $N \ge T(1+\alpha)\max\{m,\log(1/\delta)\}$, we can absorb the second line into the first term. This yields that $N^{-1/2}\Opt_{\mu} $ is bounded by $\lesssim$
	\begin{align*}
	&M_C(M_B + N^{-1/2}M_0)   k^2 (T(1+\alpha))^{k - \frac{1}{2}} 2^{d} +   2^d \left(M_CM_B (Td)^{k- \frac{1}{2}}+ M_D + \mu N^{-1/2}\right)\\
	&\le 2^{d}T^{k - \frac{1}{2}} \left( k^2(1+\alpha)^{k-\frac{1}{2}} + d^{k- \frac{1}{2}}\right) (\Mbarstoch  + \mu N^{-1/2}),
	\end{align*}
	from which we directly obtain Proposition~\ref{prop:phase_rank_intro} as stated in the body of the paper, which we restate here for convenience.
	
	\propstochphase*

\subsubsection{Results for Adversarial Noise\label{sec:main_results_adversarial}}
 
We now present the analogue of Theorem \ref{thm:big_stoch} for adversarial noise; the proof is essentially identical, and omitted in the interest of brevity:
\begin{thm}[Bounds for Adversarial Noise]\label{thm:big_adv} In the setting of Theorem~\ref{thm:big_stoch} (with the adversarial noise model), we have that
\begin{align*}
	N^{-1/2}\Opt_{\mu} &\;\lesssim M_C (\opnorm{S^{-1}\Bst} + N^{-1/2}M_0)K_1(d,T,\alpha,2)\\
	&\;+ M_C \opnorm{S^{-1}\Bst} \cdot \min\{ K_1(d,T,\alpha,2),  N^{-1/2}K_1(d,T,\alpha,\infty)\}\\
	&\;+ M_C \opnorm{S^{-1}B_w}K_1(d,T,\alpha,\infty)\sqrt{d_w} \\
	&\; + 2^d \left(M_CM_B(Tdd_w) K_2(Td)+ M_D(dd_z)  + \mu N^{-1/2}\right)\:,
	\end{align*}
	where $M_B(\cdot)$ and $M_D(\cdot)$ are as in~\eqref{eq:Mconstants}. 
\end{thm}
Recalling the definition of  definition of $\Mbaradv := (N^{-1/2}M_0 + M_B(Tdd_w))M_C + M_D(dd_z)$ from Definition~\ref{defn:magnitude_bound_general}, we obtain the following analogue of Proposition~\ref{prop:phase_rank_intro}:
\begin{cor}\label{cor:simple_adv} In the setting of the previous theorem, where $\Ast$ has $(\alpha,T)$ phase rank $d$, and maximum Jordan block size $k$, we have that 
\begin{align*}
N^{-1/2}\Opt_{\mu} \le &\; (\Mbaradv + \mu N^{-1/2}) \cdot  T^{k}C^{\adv}_{\alpha,d,k},\quad \text{where}\\
C_{\alpha,d,k}^{\adv} := &\; 2^d\left(\frac{k^2}{2^k}(1+\alpha)^{k} + d^k\right) \:.
\end{align*}
\end{cor}

\subsubsection{Bounds for Disentangling Filters \label{sec:main_results_disent}}
For the case when, after a similarity transformation, the invariant subspaces of $\Ast$ can be decomposed onto the rows of $\Cst$, we can construct individual filters for each element in the decomposition. This is a generalization of the often-studied case (e.g. \cite{sarkar2018fast}) of full-state observation: after transformation, we can observe each mode of $\Ast$ directly. We begin by describing partitions of its associated Jordan matrix $\Jst$ into invariant subspaces.

\begin{defn}[Admissible Spectral Partition] Let $J \in \R^{n \times n}$ be a matrix in Jordan normal form. We say that a set $\calS_{1:{\partsize}} := \{\calS_1,\dots,\calS_{\partsize}\} \subset [n]$ is an \emph{admissible spectral partition} if, for each $i \in [{\partsize}]$, the matrix $J(\calS_i) := (J_{ab})_{ab \in \calS_i \times \calS_i} \in \C^{|\calS_i|\times |\calS_i|}$ is a Jordan matrix.
\end{defn}
In other words, $\calS_{1:{\partsize}} \subset [n]$ is an admissible spectral partition if each $\calS_i$ corresponds to coordinates indexing a $J$-invariant subspace of $\C^n$. 

Next, we introduce a notion under which an admissible spectral partition can be ``disentangled'' by a transformation $V$, such that subsets of rows of $V\Cst$ are supported on invariant subspaces of $\Ast$ corresponding to the partition $\{\calS_1,\dots,\calS_{\partsize}\}$.
\begin{defn}[Disentangling Matrix] Let $\Ast = S\Jst S^{-1}$. We see that an invertible matrix $V \in \R^{m \times m}$ \emph{disentangles} an admissible spectral partition $\calS_{1:{\partsize}}$ of $\Jst$ if we we have the decomposition $V\Cst S = [C_1^\top | C_2^\top | \dots | C_q^\top]^\top$, where each matrix $C_i$ is supported on entries in $\calS_i$. We let $\cond(V)$ denote the condition number of $V$, and denote the associated quantity
\begin{align*}
M_C(\calS_{1:{\partsize}};V) := \opnorm{V^{-1}} \max_{v = (v_1,\dots,v_{\partsize})\in \sphere{m}} \left(\sum_{i=1}^{\partsize} \|v_i^\top C_i\|_2\right)\:.
\end{align*}
\end{defn}
Finally, we let $K_1(d,T,\alpha,q;\calS_i)$ denote the analogue of $K_1$ restricted to pairs $(\lambda,k) \in \algspec(J(\calS_i))$, and note that $K_1(d,T,\alpha,q;\calS_i)$ also satisfies the bound in Lemma~\ref{lem:worst_case_k_bounds}. With these  definitions in place, we have the following analogue of Proposition~\ref{prop:main_approx_K_theorem}, the result motivating the definition of $K_1$.
\begin{prop}\label{prop:disent_approx_K_theorem} Suppose $\{\calS_1,\dots,\calS_{\partsize}\} \subset [n]$ is an admissible partition of $\Jst$, disentangled by a matrix $V$,  and that each $J(\calS_i)$ has $(\alpha_i,T)$ phase rank at most $d$. Then, there exists a filter $\phi \in \R^{m \times dm}$ with $\|\phi\|_{\blockop} \le \kappa(V) \min({\partsize},d)2^d$ such that
\begin{align*}
\|\Gsf_{\phi}\|_{\Htwoop}  &\le \opnorm{S^{-1}\Bst} M_C(\calS_{1:{\partsize}};V) \cdot \max_{i}K_1(d,T,\alpha_i,2;\calS_i)   \\
\|\Gsf_{\phi}\|_{\Hinf}  &\le \opnorm{S^{-1}\Bst} M_C(\calS_{1:{\partsize}};V)  \cdot \max_{i}K_1(d,T,\alpha,\infty;\calS_i),
\end{align*}
with analogous bounds for $\Fsf_{\phi}$ and $\Hsf_{\phi}$. 
\end{prop}
Following along the lines of the proof of Theorem~\ref{thm:big_stoch}, we have the following bound for stochastic noise (we omit adversarial noise for brevity).
\begin{thm}[Bounds for Stochastic Noise with Disentangling Predictors]\label{thm:big_stoch_disent} 
Suppose $\{\calS_1,\dots,\calS_{\partsize}\} \subset [n]$ is an admissible partition of $\Jst$, disentangled by a matrix $V$,  and that each $J(\calS_i)$ has $(\alpha_i,T)$ phase rank at most $1 \le d \le L$. Introduce the shorthand
\begin{align*}
\widetilde{K}_1(q) := \max_{i \in [{\partsize}]}K_1(d,T,\alpha_i,q;\calS_i)
\end{align*}
Then, for any $\delta \in (0,1)$ and $N \ge Td\max\{m,\log(1/\delta)\}$, it holds with probability $1-\delta$ that
\begin{align*}
N^{-1/2}\Opt_{\mu} &\;\lesssim M_C(\calS_{1:{\partsize}};V)(M_B + N^{-1/2}M_0) \widetilde{K}_1(2)\\
&\; \le \min\{\widetilde{K}_1(2),  N^{-1/2}\widetilde{K}_1(\infty)\}M_C(\calS_{1:{\partsize}};V) M_B\sqrt{m+\log(1/\delta)}\\
&\; \le \cond(V) \min({\partsize},d)2^d\left(M_CM_B K_2(N)+ M_D + \mu N^{-1/2} \right)\:.\numberthis \label{eq:opt_MK_bound}
\end{align*}
\end{thm}

\begin{cor}\label{cor:stoch_disentangle} 
Letting $\alpha_{\max} := \max_{i \in [{\partsize}]}\{\alpha_i\}$, $k$ denote the size of the largest Jordan block of $\Ast$ and supposing $N \ge T(1+\alpha_{\max})$, \eqref{eq:opt_MK_bound} can be bounded by
\begin{align*}
N^{-1/2}\Opt_{\mu} \lesssim&\;T^{k - 1/2}\biggl[\left(k^2\left(\frac{1+\alpha_{\max}}{2}\right)^k M_C(\calS_{1:{\partsize}};V)(M_B + \frac{M_0}{\sqrt{N}})\right)\\
 &\;+\left(\cond(V) ({\partsize}\wedge d)2^d\left(M_CM_B K_2(N)+ M_D+ \mu N^{-1/2}\right)\right)\biggr].
\end{align*}
\end{cor}

\subsection{Polynomial Approximations for Linear Dynamical Systems \label{sec:poly_approx}}
In this section, we present bounds on the terms $\|\Gsf_{\phi}\|_{\Hinf}$, $\|\Fsf_{\phi}\|_{\Hinf}$, and $\|\Hsf_{\phi}\|_{\Hinf}$. Our strategy is to relate these quantities to how well polynomials can approximate a set of complex numbers. 
To begin, we define $\Mon(L,B)$ as the set of degree-$L$ monic polynomials on $\C$,
\begin{align*}
f(z) = z^{L} + f_1 z^{L-1} + \dots + f_{L}\:,
\end{align*} with real coefficients and $\ell_1$-norm at most $B$, i.e. $\|f\|_{L_1} := 1 + \sum_{i=1}^L |f_i| \leq B$. Furthermore, for a finite set $\calC \subset \Disk \times \N$ (usually $\calC = \algspec(\Ast)$), we define the following complexity terms, corresponding to $\Htwo$ and $\Hinf$. For $q \in \{2,\infty\}$, and constants $\chq$ defined above, we define the complexity terms 
\begin{align*}
\complexq_f(\calC,T) &\;:=  \max_{(\lambda,k) \in \calC}  \begin{cases}  
  0 & f(z) \text{ has root order} \ge k \text{ at } \lambda \\
  \infty & |\lambda| = 1, f(z) \text{ has root order} < k \text{ at } \lambda\\
  \chq\max\limits_{z:|z - \lambda| \le 1 - |\lambda|}  \frac{ k^2 |f(z)|}{(1 - |\lambda|)^{k - \frac{\I(q=2)}{2}}} & \text{otherwise.} 
\end{cases}
\end{align*}

The term $\complexq_f(\algspec(\Ast), T)$ roughly describes how effectively a polynomial $f$ cancels the poles in $\Ast$. Due to the $T$-step subsampling, we shall typically be interested in $\complexq_g$ for polynomials of the form $g(z) = f(z^T)$.
We recall the definition of the $\Hinf$-norm for a real rational transfer function $\Gsf(z): \C \to \C^{m \times p}$:
 \begin{align*}
 \|\Gsf\|_{\Hinf} &:= \sup_{z \in \Torus} \left\|\Gsf(z) \right\|_{\op}\:.
 \end{align*}
 If the poles of $\Ast$ are all strictly inside $\Disk$, this quantity is finite. Above, we use the operator norm on $\mathbb{C}^p \to\mathbb{C}^m$. Now, we define the $\Htwoop$ norm for such a transfer function via
 \begin{align*}
  \|\Gsf\|_{\Htwoop} :=  & \max_{v \in \sphere{m}}\sqrt{\frac{1}{2\pi}\int_{z\in \Torus} \|v^\top\Gsf(z)\|_2^2} \numberthis \label{eq:Htwoopdef}\\
  = &\;  \max_{v \in \sphere{m}}\sqrt{\frac{1}{2\pi}\int_{z\in \Torus} \tr[(v^\top\Gsf(z))^*(v^\top\Gsf(z))]}\\
  = &\; \max_{v \in \sphere{m}}\|v^\top\Gsf(z)\|_{\Htwo},
 \end{align*}
 where again we use the standard $\ell_2$-norm on $\mathbb{C}^p$ and the definition of the canonical $\Htwo$-norm (see Section 4.3 of \cite{zhou1996robust} for both the frequency-domain and time-domain definitions). Crucially, $\Htwoop$ is equal to the operator norm of the infinite-horizon Markov ``matrix".
 \begin{lem}[Equivalence of $\Htwoop$ and $\opnorm{\Markov_{\infty}(\cdot)}$]\label{lem:Htwoopequiv} Let $\Gsf = (A,B,C,D)$, and suppose $\rho(A) < 1$. Then,
 \begin{align*}
 \opnorm{\Markov_{\infty}(\Gsf)} = \|\Gsf\|_{\Htwoop}\:.
 \end{align*}
 \end{lem}
 \begin{proof} 
 Using $(i)$ to denote block indexing, we see that 
 \begin{align*}
\twonorm{v^\top\Markov_k(\Gsf)}^2  = &\;\tr\left[\sum_{i=0}^{k-2}\Markov_k^{(i)}(v^\top\Gsf)\Markov_k^{(i)}(v^\top\Gsf)^*\right] \\
\implies \twonorm{v^\top\Markov_\infty(\Gsf)}^2 = &\; \tr\left[\sum_{i=0}^{\infty}\Markov_\infty^{(i)}(v^\top\Gsf)\Markov_\infty^{(i)}(v^\top\Gsf)^*\right] \\
\overset{(*)}{=} &\; \|v^\top\Gsf\|_{\Htwo}^2 \\
\implies \opnorm{\Markov_{\infty}(\Gsf)}^2 = &\;  \|\Gsf\|_{\Htwoop}^2\:,
 \end{align*}
 where the limiting step holds as $\rho(A) < 1$, and $(*)$ comes from the time-domain characterization of the $\Htwo$ norm.
 \end{proof}
At the center of our analysis is the following proposition, that demonstrates that $\complexq_f$ does, in fact, describe the ability of polynomials $f$ to cancel poles.
 \begin{restatable}[Polynomial Approximation of Jordan Blocks]{prop}{prophinfapprox}
\label{prop:hinf_approx}
Let $f: \C \to \C$ be an analytic function and $J \in \R^{n \times n}$ be a Jordan block matrix. Then $\|f(J) (zI - J)^{-1}\|_{\Hinf} \le \complexinf_f(\calC, 1)$ and $\|f(J) (zI - J)^{-1}\|_{\Htwoop} \le \complextwo_f(\calC, 1)$. Moreover, if $f$ is a polynomial of degree at most $d < k$, then the factor of $k^2$ in $\complexq_f$ can be replaced with $k(d+1)$.
\end{restatable}
The above proposition is proved in Appendix~\ref{sec:hinf_approx_proof}.
 Note that each of the $\Htwoop$- and $\Hinf$-norms is finite as long as its argument has all of its poles strictly inside the unit disk. Therefore, some poles of $\Ast$ with modulus $1$ may need to be canceled in order to achieve a finite $\Htwoop$ or $\Hinf$-norm; we shall use the (standard) convention that the argument of the norm, as a real rational function of $z$, should be ``evaluated'' before computing the norm.

\subsubsection{Approximations Using Block-Scalar Filters: Proof of Proposition~\ref{prop:main_approx_K_theorem} \label{sec:block_scalar_predictors}} 

Our first theorem bounds the $\Hinf$- and $\Htwoop$-norms of $\Fsf_{\phi}$, $\Gsf_{\phi}$ and $\Hsf_{\phi}$ in terms of the quantity $\complexq_f(\algspec(\Ast),T)$ by considering simple, block-weighted identity filters of the form
\begin{align}\label{eq:phi_form}
\pred = -\begin{bmatrix} f_1 I_{m \times m} |  f_2 I_{m \times m} | \dots | f_L I_{m \times m} 
\end{bmatrix} \in \R^{m \times Lm},
\end{align}
 where $f_1,\dots,f_L$ correspond to the coefficients of a polynomial $f \in \Mon(L,B)$.
\begin{thm}\label{thm:poly_approx_simple} Let $(\Ast,\Bst,\Cst)$ be a dynamical system, where $\Ast = S\Jst S^{-1}$ denotes the Jordan decomposition of $\Ast$.  Then, for any $f \in \Mon(L,B)$ , the filter $\phi \in \R^{m \times Lm}$ from~\eqref{eq:phi_form} satisfies $1+\|\phi\|_{\blockop} \le B$ and
\begin{align*}
\|\Gsf_{\phi}\|_{\Htwoop} &\le \|S^{-1}\Bst\|\cdot  \|\Cst S\|_{\op}\complextwo_f(\algspec(\Jst),T)\\
\|\Gsf_{\phi}\|_{\Hinf} &\le \|S^{-1}\Bst\|\cdot  \|\Cst S\|_{\op}\complexinf_f(\algspec(\Jst),T),
\end{align*}
and similarly for $\Fsf_{\phi}$ and $\Hsf_{\phi}$, where $\Bst$ is replaced by $B_w$ and $\matx_1$, respectively.  
\end{thm}
The theorem above is proven in Section~\ref{sec:thm:poly_approx_simple}.
Note that this theorem does not preclude the case where $H_f(\algspec(\Jst),T) = \infty$, and thus the polynomial $f$ must be chosen appropriately. In particular, by choosing $f(z) = \prod_{i=1}^d (z - \mu_i^T)$, where $\mu_1,\dots,\mu_d$ are the complex numbers which witness the $(\alpha,T)$-phase rank condition, we show in Section~\ref{sec:prop:H_bound} that for systems of bounded phase rank and Jordan block size, $H_f(\algspec(\Jst),T),$ is bounded. This is summarized in the following proposition.

\begin{prop}\label{prop:H_bound} Suppose that $\Jst$ has $(\alpha,T)$-phase rank $d$. Then there exists a polynomial $f \in \Mon(d,2^d)$ such that $\complexq_f(\algspec(\Jst),T) \le  K_1(d,T,\alpha,q)$.
\end{prop}
Proposition~\ref{prop:main_approx_K_theorem} is now a direct consequence of combining Proposition~\ref{prop:H_bound} with Theorem~\ref{thm:poly_approx_simple} (with $L = d$).

\subsubsection{Approximations using Disentangling Filters \label{sec:poly_approx_no_scalar_pred}}

	We present an analogue of Theorem~\ref{thm:poly_approx_simple} for disentangling filters.
	\begin{thm}\label{main:Hinf_theorem_partition} Let $\Gst = (\Ast,\Bst,\Cst)$, and let $\Ast = S\Jst S^{-1}$ denote the Jordan decomposition of $\Ast$. Suppose $V$ disentangles a spectral partition $\calS_{1:{\partsize}}$ of $\Jst$. Then, for any polynomials $f^{(1)},\dots,f^{({\partsize})} \in \Mon(L,B)$,  there exists a filter $\phi \in \R^{m \times Lm}$ with $\|\phi\|_{\blockop} \le \cond(V) \min({\partsize},L)B$ satisfying
	 \begin{align*}
	 \|\Gsf_{\phi}\|_{\Hinf}  &\le  \opnorm{V^{-1}}\opnorm{S^{-1}\Bst}  \max_{v = (v_1,\dots,v_{\partsize})\in \sphere{m}} \sum_{i=1}^{\partsize} \|v_i^\top C_i\|_2\left(\complexinf_{f^{(i)}}(\algspec(\Jst(\calS_i)),T) \right)
	\end{align*}
	and similarly for $\Fsf_{\phi}$ and $\Hsf_{\phi}$, where $\Bst$ is replaced by $B_w$ and $\matx_1$, respectively. This also holds for the $\Htwoop$ analogues, replacing $\complexinf$ by $\complextwo$. Here $\cond(V)$ denotes the condition number.
	\end{thm}
	The proof is given in Section~\ref{sec:main:Hinf_theorem_partition}. Proposition~\ref{prop:disent_approx_K_theorem} is now a corollary of this theorem and Proposition~\ref{prop:H_bound}.

\subsection{Proof of Theorem~\ref{thm:big_stoch}\label{sec:proof_big_stoch}}

We shall prove the stochastic case; the adversarial case follows from essentially the same arguments. Let $\Ast$ have $\rho(\Ast) \le 1$ and $(\alpha,T)$ phase rank $d$, and consider the filter $\phi$ from Proposition~\ref{prop:main_approx_K_theorem}, which satisfies $1+\|\phi\| \le 2^d$ and 
\begin{align*}
\|\Markov_{\infty}(\Gsf_{\phi})\|_{\op} &\le \|S^{-1}\Bst\|_{\op}\cdot  \|\Cst S\|_{\op}K_1(d,T,\alpha,2)\numberthis \label{eq:markov_gphi}\\
\|\Gsf_{\phi}\|_{\Hinf} &\le \|S^{-1}\Bst\|_{\op}\cdot  \|\Cst S\|_{\op}K_1(d,T,\alpha,\infty),
\end{align*}
and analogously for $\Fsf_{\phi}$ and $\Hsf_{\phi}$. Consider the extended filter $\phitil = [\phi | \mathbf{0}]$ obtained by embedding $\phi$ in $\R^{m \times Lm}$. Then, $\opnorm{\phitil} = \opnorm{\phi}$, and thus
\begin{align*}
N^{-1/2}\Opt_{\mu} \le N^{-1/2}(\opnorm{\matDel_{\phitil}} +  \mu \opnorm{\phi}).
\end{align*}
Therefore, by the assumption $N \ge Td \max{\log(1/\delta),m}$, bounding $N^{-1/2}\opnorm{\matDel_{\phitil}}$ with Proposition~\ref{prop:error_bound_stochastic} implies
\begin{align*}
N^{-1/2}\Opt_{\mu}  &\; \lesssim \underbrace{\|\Markov_{\infty}(\Gsf_\phi)\|_{\op}+\|\Markov_{\infty}(\Fsf_\phi)\|_{\op} + N^{-1/2}\Mknorm{\infty}{\Hsf_\phi}}_{(a)}\\
&\; + \underbrace{\sqrt{\frac{m+\log(1/\delta)}{N}}(\mixnorm{N}{\Gsf_\phi} + \mixnorm{N}{\Fsf_\phi})}_{(b)}\\
&\;+ \underbrace{(1+\|\phi\|_{\blockop}) \left(\Mknorm{Td}{\Gsf} + \Mknorm{Td}{\Fsf}+ \opnorm{D_z}\right)+ N^{-1/2}\mu \opnorm{\phi})}_{(c)}\:.
\end{align*}
For term $(a)$, we have that 
\begin{align*}
(\|\Markov_{\infty}(\Gsf_\phi)\|_{\op}&\;+\|\Markov_{\infty}(\Fsf_\phi)\|_{\op}) + N^{-1/2}\Mknorm{\infty}{\Hsf_\phi} \\
&\; \overset{\mathclap{\eqref{eq:markov_gphi}}}{\lesssim} \|\Cst S\|_{\op}(N^{-1/2}\|S^{-1}\matx_1\|_{\op} + \|S^{-1}\Bst\|_{\op}+ \|S^{-1}B_w\|)  K_1(d,T,\alpha,2)\\
&\; \le M_C(M_B + N^{-1/2}M_0)  K_1(d,T,\alpha,2)\:.
\end{align*}
Similarly, for the term $(b)$, recalling $\Gamma_N(\Gsf) = \min\{\sqrt{N}\Markov_N(\Gsf),\|\Gsf\|_{\Hinf}\}$,
\begin{align*}
&\;\sqrt{\frac{m+\log(1/\delta)}{N}}(\mixnorm{N}{\Gsf_\phi} + \mixnorm{N}{\Fsf_\phi}) \\
&\; \le \min\{ K_1(d,T,\alpha,2),  N^{-1/2}K_1(d,T,\alpha,\infty)\}\cdot \|\Cst S\|_{\op}(\|S^{-1}B_w\| + \|S^{-1}\Bst\|_{\op})\sqrt{m+\log\tfrac{1}{\delta}}\\
&\; = \min\{ K_1(d,T,\alpha,2),  N^{-1/2}K_1(d,T,\alpha,\infty)\}M_BM_C\sqrt{m+\log \tfrac{1}{\delta}},
\end{align*}
where we take $\Mknorm{N}{\Gsf_\phi} \leq \Mknorm{\infty}{\Gsf_\phi}$ to use~\eqref{eq:markov_gphi}. For term $(c)$, we have $1+\|\phi\|_{\blockop} \le 2^d $ and $\opnorm{\phi}  \le \|\phi\|_{\blockop} \le 2^d$, so that
\begin{align*}
 (1+&\|\phi\|_{\blockop})  \left(\Mknorm{Td}{\Gsf} + \Mknorm{Td}{\Fsf}+ \opnorm{D_z}\right)+ N^{-1/2}\mu\|\phi\|_{\op}\\
&\;\le 2^d \left(\Mknorm{Td}{\Gsf} + \Mknorm{Td}{\Fsf}+ \opnorm{D_z} + N^{-1/2}\mu \right)\\
&\;\lesssim 2^d \left(M_BM_C K_2(Td)+\opnorm{\Dst} +\opnorm{D_z} + N^{-1/2}\mu\right)\\
&\;\le  2^d \left(M_BM_C K_2(Td)+ M_D + + N^{-1/2}\mu\right),
\end{align*}
where the bound $\Mknorm{Td}{\Gsf} + \Mknorm{Td}{\Fsf} \lesssim M_BM_C K_2(Td)$ follows from Proposition~\ref{prop:markovbound}. Combining parts $(a)$, $(b)$, and $(c)$ then yields Theorem~\ref{thm:big_stoch}.

\section{Supporting Proofs\label{sec:supporting_proofs_poly}}

\subsection{Proofs for Section~\ref{sec:poly_approx}}
\subsubsection{Proof of Theorem~\ref{thm:poly_approx_simple}\label{sec:thm:poly_approx_simple}}

	We prove the bound for $\Gsf_{\phi}$ without loss of generality. Let $f \in \Mon(L,B)$, and define the corresponding filter 
	\begin{align*}
	\pred = -\begin{bmatrix} f_1 I_{m \times m} |  f_2 I_{m \times m} | \dots | f_L I_{m \times m}
	\end{bmatrix} \in \R^{m \times Lm}\:.
	\end{align*}
	Since $f \in \Mon(L,B)$, we have that
	\begin{align*}
	1+\|\pred\|_{\blockop} = 1+ \sum_{\ell = 1}^L \opnorm{f_l I_{m \times m}} = 1+ \sum_{\ell = 1}^L |f_l| = \|f\|_1 \le B.
	\end{align*}
	With this filter, the corresponding observation matrix $C_{\phi}$ is given by (with a reasonable abuse of notation)
	\begin{align*}
	C_{\phi} = \Cst \Ast^{LT} + \Cst f_1 \Ast^{(L-1)T} + \Cst f_2 \Ast^{(L-2)T} + \dots + \Cst f_{L} = \Cst f(\Ast^T). 
	\end{align*}
	Therefore, 
	\begin{align*}
	\Gsf_{\phi}(z) = \Cst f(\Ast^T) (zI - \Ast)^{-1}\Bst = \Cst S f(\Jst^T) (zI - \Jst)^{-1}S^{-1}\Bst. 
	\end{align*}
	We can now bound the $\Hinf$-norm of $\Gsf_{\phi}$ in terms of $f$:
	\begin{align*}
	\|\Gsf_{\phi}\|_{\Hinf} &= \max_{z \in \Torus} \|\Cst S f(\Jst^T) (zI - \Jst)^{-1}S^{-1}\Bst\|_{\op}\\
	&\le \|\Cst S\|_{\op}\|S^{-1}\Bst\|_{\op} \max_{z \in \Torus}\|f(\Jst^T) (zI - \Jst)^{-1}\|_{\op}\\
	&= \|\Cst S\|_{\op}\|S^{-1}\Bst\|_{\op} \|f(\Jst^T) (zI - \Jst)^{-1}\|_{\Hinf}.
	\end{align*}
	A similar argument shows that 
	\begin{align*}
	\|\Gsf_{\phi}\|_{\Htwoop} \le |\Cst S\|_{\op}\|S^{-1}\Bst\|_{\op} \|f(\Jst^T) (zI - \Jst)^{-1}\|_{\Htwoop}.
	\end{align*}
	To control the $\Hinf$ and $\Htwoop$ terms at the heart of the above bound, we will recall the following proposition, proved in Section~\ref{sec:hinf_approx_proof} below.
\prophinfapprox*
	In particular, we apply Proposition~\ref{prop:hinf_approx} with the analytic function $\widetilde{f}(z) := f(z^T)$. This implies that $\|f(\Jst^T) (zI - \Jst)^{-1}\|_{\Hinf} \leq \complexinf_f(\calC,T)$. Hence, we find that
	\begin{align*}
	\|\Gsf_{\phi}\|_{\Hinf} \le \|\Cst S\|_{\op}\|S^{-1}\Bst\|_{\op}\complexinf_f(\calC,T),
	\end{align*}
	and similarly for $\complextwo_f$. 
\subsubsection{Proof of Proposition~\ref{prop:H_bound}\label{sec:prop:H_bound}}
	 Let $\mu_1,\dots,\mu_d$ witness  the $(\alpha,T)$-phase rank condition of $\algspec(\Ast)$. We now consider the corresponding polynomial $f(z) := \prod_{i=1}^d (z - \mu_i^T)$. Note that $f$ is monic and has degree $d$; thus, the fact that $f \in \Mon(d,2^d)$ follow from the following bound on its $\ell_1$-norm. 
	\begin{lem}\label{lem:fnorm} Let $f$ be a degree-$d$ polynomial whose roots all lie in $\Disk$. Then $\max_{z \in \Disk} |f(z)| \le \|f\|_{\ell_1} \le 2^d$.
	\end{lem}
	\begin{proof} The bound $\max_{z \in \Disk} |f(z)| \le \|f\|_{\ell_1}$ holds for any polynomial by the triangle inequality. Then, since $|\mu_i| \le 1$, $\|f\|_{\ell_1} = \sum_{i=0}^d |\sum_{\calS \in \binom{[d]}{i}}\prod_{i \in \calS}\mu_i| \le \sum_{i=0}^d\binom{d}{i} = 2^d$.
	\end{proof}
	Next, for any $(\lambda,k) \in \calC$, we shall bound each term in the maximum of $H_f(\calC,T)$.

	For $|\lambda| = 1$, the approximate phase rank condition implies that there are at least $k$ elements $\mu_i$ such that $\lambda^T = \mu_i^T$. Thus, $f(z^T)$ has a root of order $\ge k$ at $\lambda$, and the corresponding term in $\complexq_f(\calC,T)$ evaluates to zero. We shall therefore show that for $|\lambda| < 1$, one has 
	\begin{align}
	\max_{z:|z - \lambda| \le 1 - |\lambda|}  \frac{k^2 |f(z^T)|}{(1 - |\lambda|)^{k}} \le \begin{cases} (1+\alpha)^k T^k 2^{d-k}k & |\lambda| \in [1 - \frac{1}{T(1+\alpha)},1) \\
 	\frac{2^d}{(1 - |\lambda|)^k} & |\lambda| <  1 - \frac{1}{T(1+\alpha)} 
	\end{cases}\:.\label{eq:complex_inf_wts}
	\end{align}
	This immediately implies the desired bound on $\complexinf_f(\algspec(\Ast),T)$ in Proposition~\ref{prop:H_bound}. For the correct bound for $q = 2$, we note that since we need only consider $|\lambda| < 1$, by factoring out a $(1 - |\lambda|)^{1/2}$ we have
	\begin{align*}
	\chtwo^{-1}\complextwo_f(\algspec(\Ast),T) &= \max_{(\lambda,k)\in\algspec(\Ast): |\lambda| <1 } \max_{z:|z - \lambda| \le 1 - |\lambda|}  \frac{k^2 |f(z^T)|}{(1 - |\lambda|)^{k - 1/2}}\\
	&\overset{\mathclap{\eqref{eq:complex_inf_wts}}}{\le} (1 - |\lambda|)^{1/2}  \begin{cases} (1+\alpha)^k T^k 2^{d-k}k & |\lambda| \in [1 - \frac{1}{T(1+\alpha)},1) \\
 	\frac{2^d}{(1 - |\lambda|)^k} & |\lambda| <  1 - \frac{1}{T(1+\alpha)} 
	\end{cases}\\
	&=   \begin{cases} (1 - |\lambda|)^{1/2} (1+\alpha)^k T^k 2^{d-k}k & |\lambda| \in [1 - \frac{1}{T(1+\alpha)},1) \\
 	\frac{2^d}{(1 - |\lambda|)^{k-1/2}} & |\lambda| <  1 - \frac{1}{T(1+\alpha)} 
	\end{cases}\\
	&\le   \begin{cases}  (1+\alpha)^{k-1/2} T^{k-1/2} 2^{d-k}k & |\lambda| \in [1 - \frac{1}{T(1+\alpha)},1) \\
 	\frac{2^d}{(1 - |\lambda|)^{k-1/2}} & |\lambda| <  1 - \frac{1}{T(1+\alpha)} 
	\end{cases},
	\end{align*}
	where the last line uses that $(1 - |\lambda|)^{1/2} \le (1+\alpha)^{-1/2} T^{-1/2}$ for $ |\lambda| \in [1 - \frac{1}{T(1+\alpha)},1)$.

	We now turn our attention to the proof of~\eqref{eq:complex_inf_wts}.
	For $\lambda < 1 - \frac{1}{(1+\alpha)T}$ and $z \in \C$ such that $|z-\lambda| \le 1 - |\lambda| $, we note that $z \in \Disk$. Thus, $z^T \in \Disk$ implies $f(z^T) \le 2^d$ by Lemma~\ref{lem:fnorm}. Hence,
	\begin{align*}
	k^2  \max_{z:|z - \lambda| \le 1 - |\lambda|}  \frac{|f(z^T)|}{(1 - |\lambda|)^{k}} \le k^2 \frac{2^d}{(1-|\lambda|)^k}.
	\end{align*}

	Lastly, we consider the case $\lambda \in [1 - \frac{1}{(1+\alpha)T},1)$. Since $(\lambda,k) \in \calC$ and $\{\mu_1,\dots,\mu_d\}$ witnesses the $(\alpha,T)$-phase rank condition, without loss of generality (by permuting labels) we have $\max_{j \in [k]}\min_{\mutil: \mutil^T = \mu_{i}^T}|\lambda - \mutil| \le \alpha \left(1 - |\lambda|\right)$. Letting $\mutil_i$ denote a complex number satisfying $\mutil^T = \mu_i^T$ which minimizes $|\lambda - \mutil|$ (breaking ties arbitrarily), it then follows that 
	\begin{align*}
	\max_{i \in [k]}|\lambda - \mutil_i| \le \alpha \left(1 - |\lambda|\right)\:.
	\end{align*}
 	We shall then factor
	\begin{align*}
	f(z^T) = \left(\prod_{i=1}^k (z - \mutil_i)\right) \cdot F(z)
	\end{align*}
	where
	\begin{align*}
	F(z) := \left(\prod_{i=1}^k \frac{z^T - \mu_i^T}{z - \mutil_i} \right) \cdot  \left(\prod_{i = k+1}^{d} (z^T - \mu_i^T) \right).
	\end{align*}
	We first estimate the magnitude of $F(z)$.
	\begin{lem} $F(z)\le T^{k}2^{d-k}$ for any $z:  |z - \lambda|\le 1 - |\lambda|$.
	\end{lem}
	\begin{proof} First, observe that any $z:  |z- \lambda|\le 1 - |\lambda|$ lies in $\Disk$. Hence, 
	\begin{align*}
	|F(z)| &\le \left|\prod_{i=1}^k \frac{z^T - \mu_i^T}{z - \mutil_i} \right| \cdot  \left|\prod_{i = k+1}^{d} |z^T| + |\mu^T_i|  \right| \le 2^{d - k}\left|\prod_{i=1}^k \frac{z^T - \mu_i^T}{z - \mutil_i} \right|. 
	\end{align*}
	Next, we observe that that for $i \in [k]$, we have $\mutil_i^T = \mu_i^T$. Thus, 
	\begin{align*}
	\left|\frac{z^T - \mu_i^T}{z - \mutil_i}\right| =  \left|\frac{z^T - \mutil_i^T}{z - \mutil_i}\right| = \left|z^{T-1} + z^{T-2} \mutil_i + \dots + \mutil^{T-1}_i\right| \le T. 
	\end{align*} 
	Hence $|F(z)| \le  2^{d - k}T^k$. 
	\end{proof}
	Next, we we estimate the contribution of $\left|\prod_{i=1}^k (z - \mutil_i)\right|$.
	\begin{lem} Let $z$ satisfy $|z - \lambda| \le 1 - |\lambda|$. Then, $\frac{\left|\prod_{i=1}^k (z - \mutil_i)\right|}{(1 - |\lambda|)^{k}}| \le (1 + \alpha)^k$.
	\end{lem}
	\begin{proof} For $z: |z - \lambda| \le 1 - |\lambda|$, we have
	\begin{align*}
	\left|\prod_{i=1}^k (z - \mutil_i)\right| &\le \prod_{i=1}^k \left(|z - \lambda| + |\lambda - \mutil_i|\right) \\
	&\le \prod_{i=1}^k ( (1 - |\lambda|) + \alpha (1 - |\lambda|)) = (1 + \alpha)^k (1 - |\lambda|)^k.
	\end{align*}
	\end{proof}
	Combining these two estimates, we find that 
	\begin{align*}
	& k^2 \frac{\max_{z:|z - \lambda| \le (1 - |\lambda|)}|f(z^T)|}{(1 - |\lambda|)^{k}} \le k^2(1+\alpha)^k T^k 2^{d-k}.
	\end{align*}
	In summary, we have shown that for any $(\lambda,k) \in \calC$ with $|\lambda| < 1$, one has 
	\begin{align*}
	 \max_{z:|z - \lambda| \le 1 - |\lambda|}  \frac{k^2|f(z^T)|}{(1 - |\lambda|)^{k}}  \le k^2 \begin{cases} (1+\alpha)^k T^k 2^{d-k} & \lambda > 1 - \frac{1}{T(1+\alpha)}\\
	\frac{2^d}{(1-|\lambda|)^k} & \lambda \le \frac{1}{T(1+\alpha)}
	 \end{cases}.
	\end{align*}

\subsubsection{Proof of Theorem~\ref{main:Hinf_theorem_partition}\label{sec:main:Hinf_theorem_partition}}
 
	We shall prove the bound for $\complexinf$; the bound for $\complextwo$ is similar.
	As in the statement of the theorem, let $\Ast = S \Jst S^{-1}$, let $\calS := \algspec(\Ast)$, and suppose there exists subsets $\calS_1,\dots,\calS_\partsize \subset \calS$ and an invertible transformation $V$ for which $V\Cst S = [C_1^\top | C_2^\top | \dots | C_\partsize^\top]^\top$, where $C_i$ is supported on entries corresponding to $\calS_i$, and $C_i \in \R^{m_i \times n}$. Finally, let $f^{(1)},\dots,f^{(\partsize)}$ denote the polynomials in $\Mon(L,B)$,  and define the block diagonal matrix
	\begin{align*}
	 X_{\ell} := \blkdiag(f_{\ell}^{(1)} I_{m_1 \times m_1},f_{\ell}^{(2)} I_{m_2 \times m_2},\ldots,f_{\ell}^{(\partsize)} I_{m_\partsize \times m_\partsize})\:,
	 \end{align*}
	 and let 
	 \begin{align*}
	 \pred = -\begin{bmatrix} V^{-1} X_{1} V |  V^{-1} X_{2} V |  \cdots  | V^{-1} X_{L} V \end{bmatrix}.
	 \end{align*}
	 We can then compute that
	 \begin{align*}
	 C_{\phi} &= \Cst \Ast^{TL} + V^{-1} X_{1} V \Cst \Ast^{(L-1)T} + \dots V^{-1} X_{L} V \Cst \\
	 &= V^{-1}\left( V \Cst S S^{-1}\Ast^{TL}S + X_{1} V \Cst S S^{-1}\Ast^{(L-1)T}S + \cdots + V^{-1} X_{L} V \Cst S\right)S^{-1}\\
	 &= V^{-1}\left( V \Cst S \Jst^{TL} + X_{1} V \Cst S\Jst^{T(L-1)} + \cdots +  V^{-1} X_{L} V \Cst S\right )S^{-1}\\
	 &= V^{-1}\left( \begin{bmatrix} C_1\Jst^{TL}\\ 
	 C_2\Jst^{TL}\\
	 \vdots \\
	 C_\partsize\Jst^{TL}
	 \end{bmatrix} +  \begin{bmatrix} f_1^{(1)} C_1\Jst^{T(L-1)} \\ 
	 f_1^{(2)} C_2\Jst^{T(L-1)}  \\
	 \vdots \\
	 f_1^{(\partsize)}C_\partsize \Jst^{T(L-1)} 
	 \end{bmatrix}  + \dots  + \begin{bmatrix} f^{(1)}_L C_1\\ 
	 f^{(2)}_L  C_2\\
	 \vdots \\
	 f^{(\partsize)}_L  C_\partsize
	 \end{bmatrix}  \right)S^{-1}\\
	 &= V^{-1}\begin{bmatrix} C_1f^{(1)}( \Jst ) \\ 
	C_2f^{(2)}( \Jst ) \\
	 \vdots \\
	C_{\partsize} f^{(\partsize)}( \Jst )
	 \end{bmatrix}S^{-1}\:.
	 \end{align*}
	 Hence, 
	 \begin{align*}
	 \|\Gsf_{\phi}\|_{\op} &= \sup_{z\in\Torus}\left\| V^{-1}\begin{bmatrix} C_1f^{(1)}( \Jst ) \\ 
	C_2f^{(2)}( \Jst ) \\
	 \vdots \\
	C_\partsize f^{(\partsize)}( \Jst )
	 \end{bmatrix}S^{-1} (zI - \Ast)^{-1} \Bst\right\|_{\op} \\
	 &= \sup_{z\in\Torus}\left\| V^{-1}\begin{bmatrix} C_1f^{(1)}( \Jst ) \\ 
	C_2f^{(2)}( \Jst ) \\
	 \vdots \\
	C_\partsize f^{(\partsize)}( \Jst )
	 \end{bmatrix}(zI - \Jst)^{-1} S^{-1}\Bst\right\|_{\op} \\
	 &\le \opnorm{V^{-1}}\opnorm{S^{-1}\Bst} \sup_{z\in\Torus}\left\| \begin{bmatrix} C_1 f^{(1)}( \Jst )(zI - \Jst)^{-1} \\ 
	C_2 f^{(2)}( \Jst )(zI - \Jst)^{-1} \\
	 \vdots \\
	C_\partsize f^{(\partsize)}( \Jst )(zI - \Jst)^{-1}
	 \end{bmatrix} \right\|_{\op}\\
	 &= \opnorm{V^{-1}}\opnorm{S^{-1}\Bst} \sup_{z\in\Torus}\max_{v = (v_1,\dots,v_{\partsize})\in \sphere{m}} \|\sum_{i=1}^{\partsize} v_i^\top C_i f^{(i)}( \Jst )(zI - \Jst)^{-1}\|_{2}. \\
	 &\le \opnorm{V^{-1}}\opnorm{S^{-1}\Bst} \sup_{z\in\Torus}\max_{v = (v_1,\dots,v_{\partsize})\in \sphere{m}} \sum_{i=1}^{\partsize} \|v_i^\top C_i f^{(i)}( \Jst )(zI - \Jst)^{-1}\|_{2}. 
	 \end{align*}
	 Note that the sizes of $v_i$ are given by the admissible spectral partition of $\Ast$. By assumption, for each $i$, $C_i$ is supported on coordinates in $\calS_{i}$. Hence, we see that
	 \begin{align*}
	 \|v_i^\top C_i f^{(i)}( \Jst )(zI - \Jst)^{-1}\|_{2} = \|v_i^\top C_i f^{(i)}( J(\calS_i) )(zI - J(\calS_i))^{-1}\|_{2},
	 \end{align*}
	 since $J(\calS_i)$ is $\Jst$ supported on $\calS_i$. Thus, for any $z \in \Disk$, 
	 \begin{align*}
	 \sup_{z\in\Torus}&\max_{v = (v_1,\dots,v_{\partsize})\in \sphere{m}} \sum_{i=1}^{\partsize} \|v_i^\top C_i f^{(i)}( \Jst )(zI - \Jst)^{-1}\|_{2}\\
	 \leq &\; \max_{v = (v_1,\dots,v_{\partsize})\in \sphere{m}} \sum_{i=1}^{\partsize} \twonorm{v_i^\top C_i} \|\sup_{z\in\Torus}\|f^{(i)}( J(\calS_i) )(zI - J(\calS_i))^{-1}\|_{\op}\\
	 \le &\; \max_{v = (v_1,\dots,v_{\partsize})\in \sphere{m}} \sum_{i=1}^{\partsize}\|v_i^\top C_i\|_2  \complexinf_{f^{(i)}}(\algspec(\Jst(\calS_i)),T) ,
	 \end{align*}
	 where for each term we invoke Proposition~\ref{prop:hinf_approx} and argue as in  Theorem~\ref{thm:poly_approx_simple}. 
	 Therefore, 
	 \begin{align*}
	 \|\Gsf_{\phi}\|_{\Hinf}  &\le  \opnorm{V^{-1}}\opnorm{S^{-1}\Bst} 
	 \cdot \max_{v = (v_1,\dots,v_{\partsize})\in \sphere{m}} \sum_{i=1}^{\partsize} \|v_i^\top C_i\|_2\complexinf_{f^{(i)}}(\algspec(\Jst(\calS_i)),T) \:.
	 \end{align*}
	 Moreover, by the construction in Theorem~\ref{thm:poly_approx_simple}, we see that
	 \begin{align*}
	 \|\phi\|_{\blockop} \leq \cond(V)\sum_{\ell=1}^L\opnorm{X_\ell} = \cond(V)\sum_{\ell=1}^L \max_{i\in[{\partsize}]} |f_\ell^{(i)}| \leq \cond(V) \min({\partsize},L)B\:.
	 \end{align*}

\subsubsection{Proof of Proposition~\ref{prop:hinf_approx}\label{sec:hinf_approx_proof}}
	We wish to show that as long as $f(z)$ has a root of order at least $k$ at each $(\lambda,k) \in \algspec$ with $|\lambda| = 1$, then
	\begin{align*}
 	\|f(\Jst)(zI - \Jst)^{-1}\|_{\Hinf} \le &\;\chinf \max\limits_{w:|w - \lambda| \le 1 - |\lambda|}  \frac{ k^2 |f(w)|}{(1 - |\lambda|)^{k }}\:,\quad\text{and} \\
	 \|f(\Jst)(zI - \Jst)^{-1}\|_{\Htwoop} \le&\; \chtwo \max\limits_{w:|w - \lambda| \le 1 - |\lambda|}  \frac{ k^2 |f(w)|}{(1 - |\lambda|)^{k - \tfrac{1}{2}}}\:,
	\end{align*}
	where $k^2$ can be replaced by $k(d+1)$ if $f$ is given by a polynomial of degree $d$. Note that the right-hand sides of the above displays use argument $w$ instead of $z$ to avoid confusion with the parameter $z$ on the right-hand side. We start with following lemma:
	\begin{lem} Let $B(z) \in \R^{n \times n}$ be a block-diagonal transfer function with blocks $B_1(z),\dots,B_{\partsize}(z)$. Then, $\|B(z)\|_{\Hinf} = \max_{i \in [{\partsize}]}\|B(z)\|_{\Hinf}$ and $\|B(z)\|_{\Htwoop} = \max_{i \in [{\partsize}]}\|B(z)\|_{\Htwoop}$\:.
	\end{lem}
	\begin{proof} 

	For $\Hinf$, 
	\begin{align*}
	\|B(z)\|_{\Hinf} &= \max_{z \in \Torus} \opnorm{B(z)} ~= \max_{z \in \Torus} \max_{j \in [{\partsize}]} \opnorm{B_j (z)} \\
	&= \max_{j \in [{\partsize}]} \max_{z \in \Torus} \opnorm{B_j (z)} = \max_{j \in [{\partsize}]} \|B_j(z)\|_{\Hinf}\end{align*}

	For $\Htwoop$, let $v = (v_1,\dots,v_{\partsize})$ be a decomposition of $v$ along the blocks $B_j \in \R^{n_j \times n_j}$. Then:
	\begin{align*}
	\|B(z)\|_{\Htwoop}^2 &= \max_{v \in \sphere{n}} \frac{1}{2\pi} \int_{0}^{2\pi} \tr(v^\top B(e^{i\theta })^{*}B(e^{i \theta})v)d\theta\\
	&= \max_{v \in \sphere{n}} \sum_{j=1}^{\partsize}  \frac{1}{2\pi} \int_{0}^{2\pi} \tr(v_j^\top B_j(e^{i\theta })^{*}B_j(e^{i \theta})v_j)\\
	&\le \max_{v \in \sphere{n}} \|v_j\|_2^2 \max_{w_j \in \sphere{n_j}} \frac{1}{2\pi}  \int_{0}^{2\pi} \sum_{j=1}^{\partsize}  \tr(w_j^\top B_j(e^{i\theta })^{*}B_j(e^{i \theta})w_j)\\
	&= \max_{v \in \sphere{n}} \|v_j\|_2^2 \|B_j(z)\|_{\Htwoop}^2 = \max_{j \in [{\partsize}]} \|B_j(z)\|_{\Htwoop}^2,
	\end{align*}
	where the last holds since $\sum_j \|v_j\|_2^2 = \|v\|^2 = 1$. 
	To see the converse holds, one may choose $v$ to be supported on the coordinates of one block $B_j$. 
	\end{proof}

	We now return to the proof of Proposition~\ref{prop:hinf_approx}. Let $f$ be an analytic function of degree $d$, where we take $d = \infty$ if $f$ is not a finite-length polynomial.
    Since $f(J)(zI-J)^{-1}$ has the same Jordan block structure as $J$, the above lemma lets us assume without loss of generality that $J$ consists of a single Jordan block corresponding to an eigenvalue $\lambda$ of order $k$. If $f$ has a zero of order $k$ at $\lambda$, then $f$ is divisible by $(z - \lambda)^k$, which by Cayley-Hamilton implies that $f(J) = \mathbf{0}$ and thus $\|f(J)(zI - J)^{-1}\|_{\Htwoop} \le \|f(J)(zI - J)^{-1}\|_{\Hinf}=0$.
   We shall now show that for any $|\lambda| < 1$ and $z \in \Torus$,
   \begin{align}\label{eq:complex_opbound_wts}
   \opnorm{f(J)h_z(J)}  \le \frac{1}{|z - \lambda|}\max_{w:|w - \lambda| \le 1-|\lambda|} \frac{|f(w)|k(k\wedge d+1)}{(1-|\lambda|)^{k-1}}\:.
   \end{align}

   We first show how to conclude the proof assumption the above display~\eqref{eq:complex_opbound_wts}, and then turn to establishing the inequality.

   \textbf{1. Concluding the proof from~\eqref{eq:complex_opbound_wts}:}   To bound $\Hinf$, we see that 
   \begin{align*} \|f(J)h_z(J)\|_{\Hinf} = \max_{z \in \Torus} \opnorm{f(J)h_z(J)}  \le  \max_{w:|w - \lambda| \le 1-|\lambda|} \frac{|f(w)|k(k \wedge d+1)}{(1-|\lambda|)^{k}},
   \end{align*}
   since we have $|z - \lambda| \ge 1 - |\lambda|$ for all $z \in \Torus$ and $\lambda \in \Disk$. Noting that $\chinf = 1$, the above display is precisely the quantity corresponding to $\complexinf_f$ by taking $k(k\wedge (d+1)) \le k^2$, but it allows one to replace $k^2$ by $k(d+1)$ for $d < k$.

   To bound $\Htwoop$, we use the following lemma, proved at the end of the section.
   \begin{lem}[$\Htwo$ integration]\label{lem:htwo_integral} For $\chtwo = \sqrt{1 + \frac{2}{\pi}}$,  $\sqrt{\frac{1}{2\pi}\int_{0}^{2\pi} \frac{1}{|e^{i \theta} - \lambda|^2}d\theta} \le \chtwo \sqrt{\frac{1}{1 - |\lambda|}}$ for all $\lambda \in \Disk: |\lambda| < 1$.
   \end{lem}
   We note that in our context, the variable over which we integrate is the subscript $z$ in $f(J)h_z(J)$. Using the integral computation,
   \begin{align*}
    \|f(J)h_z(J)\|_{\Htwoop}^2 &=  \max_{v \in \sphere{n}}  \frac{1}{2\pi} \int_{0}^{2\pi} \tr(v^\top (f(J)h_{e^{i\theta}}(J))^{*}(f(J)h_{e^{i\theta}}(J))v\\
    &\le   \frac{1}{2\pi} \int_{0}^{2\pi} \opnorm{ (f(J)h_{e^{i\theta}}(J)}^2\\
    &\overset{\mathclap{\eqref{eq:complex_opbound_wts}}}{\le} \quad  \left(\max_{w:|w - \lambda| \le 1-|\lambda|} \frac{|f(w)|k(k \wedge (d+1))}{(1-|\lambda|)^{k-1}}\right)^2 \frac{1}{2\pi} \int_{0}^{2\pi}\frac{1}{|z - \lambda|^2}\\
    &\overset{\mathclap{\text{Lem. \ref{lem:htwo_integral}}}}{\le} \quad  \left(\max_{w:|w - \lambda| \le 1-|\lambda|} \frac{|f(w)|k(k \wedge (d+1))}{(1-|\lambda|)^{k-1}}\right)^2 \chtwo^2 \frac{1}{1-|\lambda|}\\
    &=   \left(\max_{w:|w - \lambda| \le 1-|\lambda|} \frac{|f(w)|k(k \wedge (d+1))}{(1-|\lambda|)^{k-1/2}}\right)^2 \chtwo^2 
   \end{align*}
   as needed.

   	\textbf{2. Proving~\eqref{eq:complex_opbound_wts}:}
	Since $|\lambda| < 1$, the function $h_z(\lambda) := \frac{1}{z - \lambda}$ is analytic on $\Torus$, and thus we can write $f(J)(zI - J)^{-1} = f(J)h_z(J)$. By Lemma~\ref{lem:toep_conv} and the formula for functions of Jordan block matrices, we see that (dropping the $\lambda$ argument for brevity)for any $z \in \Torus$,
	\begin{align*}
	\opnorm{f(J)h_z(J)} \leq &\;   \sum_{\ell =0}^{k-1}\sum_{j = 0}^{\ell}\left|\frac{h_{z}^{(\ell-j)}}{(\ell-j)!}\right| \left|\frac{f^{(j)}}{j!}\right|\\
	= &\;   \sum_{j =0}^{k-1}\left|\frac{f^{(j)}}{j!}\right|\sum_{\ell=j}^{k-1}\left|\frac{h_{z}^{(\ell-j)}}{(\ell-j)!}\right| \\
	= &\;   \sum_{j =0}^{k-1}\left|\frac{f^{(j)}}{j!}\right|\sum_{\ell=1}^{k-j}\left|\frac{h_{z}^{(\ell)}}{\ell!}\right| \\
	= &\; \sum_{j =0}^{k-1}\left|\frac{f^{(j)}}{j!}\right|\sum_{\ell=1}^{k-j}\left|z-\lambda\right|^{-\ell}\:\\
	= &\; \sum_{j =0}^{k-1 \wedge d}\left|\frac{f^{(j)}}{j!}\right|\sum_{\ell=1}^{k-j}\left|z-\lambda\right|^{-\ell} \numberthis\label{eq:controlnorm_leftoff}\:,
	\end{align*}
	where the last line uses the fact that if $f$ is a degree $d$ polynomial, all $f^{(j)}$ vanish for $j > d$.
	Next, we use Cauchy's integral formula to bound the magnitudes of the terms $\left|\frac{f^{(j)}}{j!}\right|$:
	\begin{lem}[Cauchy's Integral Formula, see e.g.~\cite{stein2003princeton}] Let $f: \C \to \C$ be an analytic function. Then $f^{(n)}(\lambda) \le \frac{n!}{r^n}\max_{a:|\lambda - a|=r} |f(a)|$. 
	\end{lem}
	By setting $r = 1 - |\lambda|$, we have that
	\begin{align*}
	\left|\frac{f^{(j)}(\lambda)}{j!}\right|&\le\max_{w:|w - \lambda| \le 1-|\lambda|}\frac{|f(w)|}{(1 - |\lambda|)^{j}}\:.
	\end{align*}
	Furthermore, if $f$ is a polynomial of degree $d$, moreover, then $f^{(j)}$ vanishes for all $j > d$. Therefore, picking up where we left off from our bound on $\opnorm{f(J)h_z(J)}$ in Equation~\eqref{eq:controlnorm_leftoff}, we bound
	\begin{align*}
	\sum_{j =0}^{d \wedge k-1}\left|\frac{f^{(j)}}{j!}\right|\sum_{\ell=1}^{k-j}|z-\lambda|^{-\ell} \leq&\;
	\max_{w:|w - \lambda| \le 1-|\lambda|}|f(w)|\sum_{j=0}^{d \wedge k-1}\sum_{\ell=1}^{k-j}|z - \lambda|^{-\ell}(1 - |\lambda|)^{-j}\\
	=&\;
	\frac{1}{|z - \lambda|} \max_{w:|w - \lambda| \le 1-|\lambda|}|f(w)|\sum_{j=0}^{k-1 \wedge d}\sum_{\ell=1}^{k-j}|z - \lambda|^{1-\ell}(1 - |\lambda|)^{-j}\\
	\overset{(i)}{\le}&\; \frac{1}{|z - \lambda|} \max_{w:|w - \lambda| \le 1-|\lambda|}|f(w)|\sum_{j=0}^{k-1 \wedge d}\sum_{\ell=1}^{k-j}(1 - |\lambda|)^{1-\ell - j}\\
	\leq&\; \frac{1}{|z - \lambda|}\max_{w:|w - \lambda| \le 1-|\lambda|} \frac{|f(w)|k(k \wedge d+1)}{(1-|\lambda|)^{k-1}}\:,
	\end{align*}
	where $(i)$ follows since follows since $\min_{z \in \Torus}|z-\lambda| = 1-|\lambda|$ for $\lambda \in \Disk$, and the last inequality comes from taking the maximum over the at most $\frac{k(k \wedge d+1)}{2}\leq k(k\wedge d+1)$ terms in the double sum. 

	Lastly, we complete the argument by turning to the proof of Lemma~\ref{lem:htwo_integral}.

	\begin{proof}[Proof of Lemma~\ref{lem:htwo_integral}] By rotation invariance of the integral, we may assume $\lambda$ is real and non-negative. Fix a $\theta_0  \in (0,\frac{\pi}{2}]$ to be chosen later. Then, we decompose our integral as
   \begin{align*}
   \int_{0}^{2\pi} \frac{1}{|e^{i \theta} - \lambda|^2}d\theta &= \int_{0}^{2\pi}\frac{1}{\sin^2 \theta + (1 - \lambda \cos \theta)^2 }d\theta\\
   &\overset{(i)}{=} 2\int_{0}^\pi \frac{1}{\sin^2 \theta + (1 - \lambda \cos \theta)^2 } d\theta \\
   &\overset{(ii)}{\le} 2\left(\frac{\theta_0}{(1 - \lambda)^2} + \left(\int_{\theta_0}^{\pi/2} \frac{1}{\sin^2 \theta}d\theta\right) + \pi\right)\\
    &\overset{(iii)}{=} 2\left(\frac{\theta_0}{(1 - \lambda)^2} + \cot \theta_0 + \pi\right),
   \end{align*}
   where $(i)$ uses the symmetry of the integral, and $(ii)$ breaks the integral into $[0,\theta_0]$, $[\theta_0,\pi]$, and $[\pi,2\pi]$, bounding the integrand above by $\frac{1}{1 - \lambda}$, $\frac{1}{\sin^2 \theta}$, and $1$ on each respective portion. 

   Now, setting $\theta_0 = \arcsin(1-\lambda) \in (0,\frac{\pi}{2}]$, we have that $\cot \theta_0 \le 1/\sin\theta_0 = 1/(1-\lambda)$. Moreover, $\sin \theta_0 = 1 - \lambda$, so
   \begin{align*}
	\frac{\theta_0}{(1 - \lambda)^2} = \frac{1}{1-\lambda} \cdot \frac{\theta_0}{\sin \theta_0} \le \frac{1}{1-\lambda},
   \end{align*}
 	since $\sin x \le x$. Combining the above bounds,
 	\begin{align*}
 	\frac{1}{2\pi}\int_{0}^{2\pi} \frac{1}{|e^{i \theta} - \lambda|^2}d\theta \le \frac{2}{\pi(1-\lambda)} + 1 \le\frac{1}{\lambda}(1 + \frac{2}{\pi}),
 	\end{align*}
 	as needed.
   \end{proof}

\subsection{Bounds on Finite System Norms: Proof of Proposition~\ref{prop:markovbound} \label{sec:prop:finite_length}}
Here we prove a bound on the operator norm of a Markov matrix, in terms of the terms $\Mcomplextwo,\widetilde\Mcomplextwo,\Ktwotwo$ in \eqref{eq:Mcomplextwo}, \eqref{eq:Mcomplextwotil}, and \eqref{eq:K_2(N)} respectively. We recall the proposition we aim to prove:
\propmarkovbound*
\begin{proof}
 We have that 
\begin{align*}
\opnorm{\Markov_{n+1}(\Gsf)}&= \opnorm{\begin{bmatrix} D \mid C B \mid C \Ast B \mid \dots \mid C \Ast^{n-1} B\end{bmatrix}}\\
&\le \opnorm{D} +  \opnorm{\begin{bmatrix} C B \mid C \Ast B \mid \dots \mid C \Ast^{n-1} B \end{bmatrix}}\\
&= \opnorm{D} +  \opnorm{\begin{bmatrix} C S S^{-1}\Bst \mid C S \Jst S^{-1}B \mid \dots \mid C S \Jst^{n-1} S^{-1}\Bst\end{bmatrix}}\\
&= \opnorm{D} +  \opnorm{C S \cdot \begin{bmatrix} I \mid   \Jst  \mid \dots \mid \Jst^{n-1} \end{bmatrix} \cdot( I_{n \times n} \otimes S^{-1} B) }\\
&\le \opnorm{D} +  \opnorm{C S} \cdot \opnorm{\begin{bmatrix} I \mid   \Jst  \mid \dots \mid \Jst^{n-1}\end{bmatrix}  }\opnorm{I_{n \times n} \otimes S^{-1} B}\\
&= \opnorm{D} +  \opnorm{C S}\opnorm{S^{-1}B} \cdot \opnorm{\begin{bmatrix} I \mid   \Jst  \mid \dots \mid \Jst^{n-1} \end{bmatrix}}\:.
\end{align*}
Next, we see that since $\Jst^t$ is block diagonal, the operator norm of $\begin{bmatrix} I \mid   \Jst  \mid \dots \mid \Jst^{n-1} \end{bmatrix}$ is equal to the largest operator norm of a block row corresponding to one the blocks of $\Jst$\footnote{Indeed, the operator norm is invariant under permutations of rows and columns, and since $\Jst^t$ is block diagonal, one can permute the columns of $\begin{bmatrix} I \mid   \Jst  \mid \dots \mid \Jst^{n-1} \end{bmatrix}$ to render it a block diagonal (rectangular) matrix. It is then well known that the operator norm rectangular block diagonal operators is equal to the operator norm of its largest block.}
Consequently, it suffices to prove that $\Jst = J$ consists of a single Jordan block $(\lambda,k) \in \algspec(\Ast)$. We shall first consider a bound that holds for $|\lambda| < 1$, and then a general bound for arbitrary $\lambda$. For $|\lambda| < 1$, by Lemma~\ref{lem:Htwoopequiv} we have
\begin{align*}
\opnorm{\begin{bmatrix} I \mid   J  \mid \dots \mid J^{n-1} \end{bmatrix}} \le \lim_{n \to \infty} \opnorm{\begin{bmatrix} I \mid   J  \mid \dots \mid J^{n-1} \end{bmatrix}}  = \|(zI - J)^{-1}\|_{\Htwoop}\:.
\end{align*}
Applying Proposition~\ref{prop:hinf_approx} with the trivial polynomial $f(z) = 1$, which has degree $d = 0$, gives
\begin{align*}
\|(zI - J)^{-1}\|_{\Htwoop} \le  \chtwo \max_{z:|z - \lambda| \le 1-|\lambda|} \frac{|f(z)|k(k \wedge d + 1)}{(1-|\lambda|)^{k-1/2}} = \frac{k}{(1-|\lambda|)^{k-1/2}}.
\end{align*}
Now for $\lambda \in [0,1]$, we need to bound
\begin{align*}
\opnorm{\begin{bmatrix} I \mid   J  \mid \dots \mid J^{n-1} \end{bmatrix}}^2 \le  \widetilde{\Mcomplextwo}(k,n)^2,
\end{align*}
where we recall the definition
\begin{align*}
\widetilde{\Mcomplextwo}(k,n):= &\; \begin{cases} n^{1/2} & k = 1\\
	n^{k - 1/2}\left(\frac{e}{k-1}\right)^{k-1}  & 2 \le k \le n + 1 \\
	n^{1/2}2^n & k \geq n + 1
	\end{cases}\:.
\end{align*}
Recalling the formula for the powers of Jordan blocks, we have that
	\begingroup
	\renewcommand{\arraystretch}{2.5}
	\begin{align}
	J^t &= \begin{bmatrix} \lambda^t & \binom{t}{1}\lambda^{t-1} & \binom{t}{2}\lambda^{t-2} & \dots & \binom{t}{k-1} \lambda^{t - (k-1)} \\
	0 &  \lambda^t & \binom{t}{1}\lambda^{t-1}  & \cdots & \binom{t}{k-2} \lambda^{t - (k-2)}\\
	\vdots & \ddots  & \ddots & \ddots & \vdots \\
	\vdots &   & \ddots & \lambda^t  & \binom{t}{1}\lambda^{t-1} \\
	0 & \cdots  & \cdots & 0 & \lambda^t 
	\end{bmatrix}, \label{eq:Jordan_power_formula}
	\end{align}
	\endgroup
	where we use the convention $\binom{t}{j} = 0$ for $j > t$. Bounding $\opnorm{J^t}$ by the $\ell_1$ norm of its first row gives
	\begin{align}\label{eq:Jordan_block_opnorm}
	\opnorm{J^t} \le \sum_{j=0}^{k-1} \binom{t}{j} |\lambda|^{t - j} \I(t \ge j) = \sum_{j=0}^{k-1} \alpha_{j,t},
	\end{align}
 where $\alpha_{t,j} = \binom{t}{j} |\lambda|^{t - j} \I(t \ge j) \le \binom{t}{j} \I(t \ge j) $. Since $\alpha_{t,j}$ is increasing in $t$, we can use the crude bound
\begin{align*}
\opnorm{\begin{bmatrix} I \mid   J  \mid \dots \mid J^{n-2} \end{bmatrix}}^2 &\le \sum_{t = 0}^{n-1} \opnorm{J^t}^2\le \sum_{t = 0}^{n-1}(\sum_{j=0}^{k-1}\alpha_{t,j})^2  \\
&\le n\left(\sum_{j=0}^{k-1}\alpha_{n,j}\right)^2 = n\left(\sum_{j=0}^{k-1 \wedge n} \binom{n}{j}\right)^2.
\end{align*}
For $k \ge n + 1$, $\sum_{j=0}^{k-1} \binom{n}{j} = 2^n$, yielding a bound of $n(2^n)^2$. For $k = 1$, the above sum is $n$, and for $ 2 \le k\le n + 1$, we have the standard bound bound $\sum_{j=0}^{k-1} \binom{n}{j} \le (\frac{e n}{k-1})^{k-1}$, yielding a bound of $(n^{k-\frac{1}{2}}(\frac{e}{k-1})^{k-1})^2$. Taking a square root of each of the three cases concludes the proof.
\end{proof}


\section{Bounds under Strong Observability\label{app:strong_observability}}
In this section, we formally define a notion called \emph{strong observability}, inspired by the control theory community, which describes how difficult it is to estimate the hidden state in a linear system with \emph{known} dynamics. We then use this notion to develop a bound on $\Opt_{\mu}$ in terms of the quantities introduced. 
We begin by defining the $d$-step observability matrix 
\begin{align*}
\calO_d(A,C) := [C^\top | (CA)^\top | \dots  |(CA^{d-1})^\top]^\top \in \R^{md \times n}
\end{align*} 
for conforming $A,C$. Furthermore, we introduce the definition of an invariant decomposition:
\begin{defn}[Invariant Decomposition] We say that $(A_+,A_-,C_+,C_-)$ is an \emph{invariant decomposition}  of $(\Ast,\Cst)$ if $\Ast= A_+ + A_-$, $\Cst = C_+ + C_-$, $A_+A_- = A_-A_+ = 0$, and $\rowspace(C_+) \subset \rowspace(A_+)$ and $C_- A_+$. 
\end{defn}
In other words, $A_+$ and $A_-$ contain complementary invariant subspaces of $\Ast$, $C_+$  provides information only about $A_+$, and $C_-$ provides no information about $A_+$.\footnote{We note that if $A_+$ and $A_-$ satisfies   $A_+A_- = A_-A_+ = 0$ and $\Ast = A_- + A_+$, then we can obtain an invariant decomposition $(A_+,A_-,C_+,C_-)$ by letting $C_+ = \Proj_{\Ast} C$ is the matrix obtained by projecting $C$ onto the rowspace of $A_+$, and and $C_- = \Cst - C_+$  the projection onto its complement.}
One should associate $A_+$ with large dynamical modes we need to filter and $A_-$ with smaller modes we can disregard. Strong observability is then defined as follows.
\begin{defn}[Strong Observability]\label{def:strong_obser} Given a pair $(A_+,C_+)$ with $\rowspace(C_+) \subseteq \rowspace(A_+)$, we say that $(A_+,C_+)$ is $(\sigma,T,d)$-strongly observable if $\sigma_{n_+}(\calO_d(A_+^T,C_+)) \ge \sigma > 0$, with $n_+ := \rank(A_+)$.
\end{defn}
Here, $\sigma_k(\cdot)$ denotes the $k$-th largest singular value.
Strong observability states that given $d$ observations sampled every $T$ time steps, one can reconstruct the hidden state $\matx_t$ in a numerically stable fashion. Restricted to the pair $(A_+,C_+)$, strong observability is a quantitative version of a fundamental \emph{observability} condition in control theory, and state estimation in particular (see, e.g. ~\cite{hautus1983strong}). As an example, one can show that the transition matrix $\Ast = [\begin{smallmatrix} 1 & \Delta \\
0 & 1 \end{smallmatrix}]$ and observation matrix $\Cst = \SmallMatrix{ 1 & 0 }$, which correspond to Newton's equation $F = m \ddot{x}$ when the position $x$ is observed, satisfies $(\sigma, T,2)$-strong observability whenever $T\Delta$ is bounded away from zero. 

We begin by stating a simplified bound on $\Opt_{\mu}$ under the strong observability condition, in terms of the  control-theoretic norm $\|\Markov_{\infty}(\cdot)\|_{\op}$.
\begin{prop}[Bounds for Strong Observability]\label{prop:strong_obs} Let $d \le L$, $N \ge Ld \log(1/\delta)$ and $(A_+,A_-,C_+,C_-)$ be an invariant decomposition of $(\Ast,\Cst)$, with $\rho(A_-) < 1$. 
Define the systems 
\begin{align*}
\Gsf_{-} = (A_-,\Bst,C_-,0), \quad \text{and} \quad \Fsf_{-} = (A_-,B_w,C_-,0).
\end{align*} 
Then if $(A_+,A_-)$ is $(\sigma,T,d)$-strongly observable for $(A_+,A_-)$, then with probability at least $1-\delta$,
\begin{align*}
&N^{-1/2}\Opt_{\mu} \lesssim d\left(1+\frac{\opnorm{C_+ A_+^{T}}}{\sigma}\right) \left(\sqrt{(m+\log\tfrac{1}{\delta})}\Const_{A_-} + \Const_{Td} + N^{-1/2}\mu\right), \text{ where }\\
&\Const_{A_-} :=  \|\Markov_{\infty}(\Gsf_-)\|_{\op} + \|\Markov_{\infty}(\Fsf_-)\|_{\op}\:\: \text{and}\:\:
\Const_{Td} := \opnorm{\Markov_{Td}(\Gsfst)} + \opnorm{\Markov_{Td}(\Fsfst)} + \opnorm{D_z}.
\end{align*}
If $L$ is greater than the degree $d_+$ of the minimal polynomial of $A_+$, then the above bound also holds with $d(1+\opnorm{C_+ A_+^{T}}/\sigma)$ (resp. $d$) replaced by $2^{d_+}$ (resp. $d_+$), even if strongly observability fails.
\end{prop}
We note that by choosing the invariant partition $(A_+,A_-)$ to ensure that $A_-$ is stable (i.e., placing all unstable modes into $A_+$), then the operator norms of the infinite-horizon quantities $\Gsf_-$ and $\Fsf_-$ are finite; moreover, by placing near-unstable modes into $A_+$, one can obviate the dependence on instability in these terms as well.

In the following subsection, we shall a state more precise variant of the above bound, including analogues for adversarial noise. Subsequent subsections contain the deferred proofs.  

\subsection{Granular Bounds for Strong Observability}

In this section, we present bounds under the strong observability criterion, Definition~\ref{def:strong_obser}.
We define the corresponding \emph{stochastic observability error} term,
\begin{align*}
\oestoch(\delta,m,N) := &\; \sqrt{N}(\|\Markov_{\infty}(\Gsf_-)\|_{\op}+\|\Markov_{\infty}(\Fsf_-)\|_{\op}) + \Mknorm{\infty}{\Hsf_-}\\
    &\; + \sqrt{m+\log(1/\delta)}(\mixnorm{N}{\Gsf_-} + \mixnorm{N}{\Fsf_-})\\
    &\;+ \sqrt{N}  \left(\Mknorm{Td}{\Gsfst} + \Mknorm{Td}{\Fsfst}+ \opnorm{D_z}\right),
\end{align*}
and the  \emph{adversarial observability error} term,
\begin{align*}
\oeadv(\delta,m,N) := &\; \sqrt{N}\Mknorm{\infty}{\Gsf_-}+\sqrt{Nd_w}\mixnorm{N}{\Fsf_-} + \Mknorm{\infty}{\Hsf_-}\\
    &\;+   \sqrt{m+\log(1/\delta)}\mixnorm{N}{\Gsf_-}\\
    &\;+  \sqrt{N}\left[\Mknorm{Td}{\Gsfst} +\sqrt{Tdd_w}\Mknorm{Td}{\Fsfst}+ \sqrt{dd_z} \opnorm{D_z}\right],
\end{align*}
where we recall the definition
\begin{align*}
\mixnorm{N}{\Gsf} := \min\{\sqrt{N}\Mknorm{\infty}{\Gsf},\|\Gsf\|_{\Hinf}\}
\end{align*}
Our main theorem is as follows:
\begin{thm}\label{thm:main_observability} Suppose that $N \ge Td\max\{m,\log(1/\delta)\}$. Then, if $(A_+,C_+)$ is $(\sigma,T,d)$-strongly  observable for $d \le L$, then with probability $1 - \delta$, we have that in the stochastic model,
\begin{align*}
\Opt_{\mu} &\le (\oestoch(\delta,m,N) + \mu)\left(1 + \frac{d\opnorm{C_+A_+^d}}{\sigma}\right).
\end{align*}
where the analogous bound holds with $\oeadv$ under adversarial noise. 

In general, even if $(A_+,C_+)$ may not be $(\sigma,d)$-strongly observable, we argue as follows. Let $f^+$ denote the minimal polynomial of $A_+$. If $L \ge \deg (f^+)$, then 
\begin{align*}
\Opt_{\mu} &\le (\oestoch(\delta,m,N) + \mu)\|f^+\|_{L_1} \\
&\le (\oestoch(\delta,m,N) + \mu)2^{\deg(f_+)} \le (\oestoch(\delta,m,N) + \mu)2^{\rank(A_+)},
\end{align*}
and analogously for adversarial noise. 
\end{thm}
Proposition~\ref{prop:strong_obs} follows directly by bounding $\mixnorm{N}{\Gsf}\le\sqrt{N}\Mknorm{\infty}{\Gsf}$ and dropping the $\Hsf_-$ term under the assumption made in the body that $\matx_1 = 0$. We now turn the proof of the theorem. The above theorem is prove in the following subsection.

\subsection{Proof of Theorem~\ref{thm:main_observability}\label{sec:main_obs_proof}}
Let $\phi = [\Psi_1 | \dots | \Psi_d | \veczero]$;  in view of the discussion in Section~\ref{sec:error_calcs}, we can assume $L = d$ for our analysis. 
Since $A_+A_- = A_-A_+ = 0$, it follows that for any power $k \ge 1$, $\Ast^k = A_+^k + A_-^k$. Moreover, since $\rowspace(C_+) \subset \rowspace(A_+)  \in \ker(A_-)^\top$ (and similarly when the signs are swapped), it follows that $C_+ A_- = C_-A_+ = 0$. It then follows that
\begin{align*}
\Cst\Ast^k = C_+A_+^k + C_-A_-^k, \quad \text{ for all } k \ge 0.
\end{align*}
This allows us to decompose $C_{\phi}$ as $C_{\phi} := C_{\phi}^+ + C_{\phi}^-$, where for $\sigma \in \{+,-\}$
\begin{align*}
 C_{\phi}^{\sigma} := C_{\sigma} A_{\sigma}^{LT} - \sum_{\ell=1}^L \Psi_{\ell}C_{\sigma} A_{\sigma}^{(L-\ell)T} \in \R^{m \times n}\:.
\end{align*}
The following proposition gives us a control on $ \|\matDel_{\phi}\|_{\op}$ in terms of the systems $C_{\phi}^{\sigma}$. Its proof is deferred to the following subsection:
\begin{prop}\label{prop:observable_error} Suppose that $\phi \in \R^{m \times md}$ satisfies $C_{\phi}[A_+] = 0$ and that $N \ge Td\max\{m,\log(1/\delta)\}$.  Then, in the stochastic noise model,
\begin{align*}
    \|\matDel_{\phi}\|_{\op} &\;\lesssim(1+\blockopnorm{\phi})\oestoch(\delta,m,N)\:,
\end{align*}
and in the adversarial noise model,
\begin{align*}
    \|\matDel_{\phi}\|_{\op} &\;\lesssim(1+\blockopnorm{\phi})\oeadv(\delta,m,N)\:.
\end{align*} 
\end{prop}
Hence,  for stochastic noise, we get that with probability at least $1-\delta$,
\begin{align*}
\Opt_{\mu} \le \mu\opnorm{\pred} + \|\matDel_{\phi}\|_{\op} \le (1 + \blockopnorm{\phi})(\mu + \oestoch(\delta,m,N)),
\end{align*}
and similarly for adversarial noise.
To conclude, it remains to bound $\blockopnorm{\phi}$. Considering an invariant partition $\Ast = A_+ + A_-$, we invoke the following lemma, which is a consequence of the Moore-Penrose pseudoinverse.
\begin{restatable}{lem}{observablelem}\label{lem:strong_obs_inverse} If $(A_+,C_+)$ is $(\sigma,T,d)$-strongly observable, then there exists a matrix $\phi = [\Psi_1 | \dots | \Psi_d] \in \R^{m \times md}$ satisfying $C_+A_+^{Td} - \sum_{\ell=1}^d \Psi_\ell C_+A_+^{T\ell} = 0 $ and $\opnorm{\phi} \le \frac{\opnorm{C_+A_+^{Td}}}{\sigma}.$
\end{restatable}
Using this lemma (proved in Appendix~\ref{sec:observable_error_control_bounds}) and the error bounds from Proposition~\ref{prop:error_bound_stochastic} below,  we conclude the proof of Proposition~\ref{prop:strong_obs} in Appendix~\ref{sec:observable_error_proof}. This lemma directly yields the first part of our theorem, since 
\begin{align*}
1 + \blockopnorm{\phi} \le (1 + d\opnorm{\phi}) \le 1 + \frac{d\opnorm{C_+A_+^d}}{\sigma}.
\end{align*}
	\begin{proof}[Proof of Lemma~\ref{lem:strong_obs_inverse}] Recall the observability matrix
	\begin{align*}
	\calO_d(A_+,C_+) = \begin{bmatrix} C_+ \\
	C_+ A_+\\
	\vdots\\
	C_+ A_+^{d-1}\\
	\end{bmatrix}\:.
	\end{align*}
	Observe that since $\rowspace(C_+) \subset \rowspace(A_+)$, $\calO_d(A_+,\Cst)$ has rank at most $\dim(\rowspace(A_+)) = \rank(A_+) = n_+$.  By assumption, $\calO_d(A_+,C_+)$ has rank at least $n_+$ as well. It follows that $\rowspace(\calO_d(A_+,C_+)) = \rowspace(A_+)$, and therefore $C_+ A_+^d \in \mathrm{range}(\calO_d(A_+,C_+))$. This implies that for the filter
	\begin{align*}
	\phi = C_+ A_+^{d}\calO_d(A_+,C_+)^{\dagger} \in \R^{m \times dm},
	\end{align*}
	we have
	\begin{align*}
	\phi \calO_d(A_+,C_+)  = C_+ A_+^{d}\calO_d(A_+,C_+)^{\dagger}\calO_d(A_+,C_+) = C_+A_+^d\:.
	\end{align*}
	The above then implies that
	\begin{align*}
	C_{\phi}^+ = C_+A_+^d - \phi\cdot\calO_d(A_+,C_+) = 0.
	\end{align*}
	Moreover, as we have established that $\rank(\calO_d(A_+,\Cst)) = n_+$, we have
	\begin{align*}
	\opnorm{\phi} &=  \opnorm{C_+ A_+^{d}\calO_d(A_+,C_+)^{\dagger}} \\
	&\le \opnorm{C_+ A_+^{d}}\opnorm{\calO_d(A_+,C_+)^{\dagger}} ~\le~ \opnorm{C_+ A_+^{d}}\sigma_{n_+}(\calO_d(A_+,C_+))^{-1}\:.
	\end{align*}
	Recalling that our strong observability assumption implies that $\sigma_{n_+}(\calO_d(A_+,C_+))^{-1} \ge \sigma$, we find that the above quantity is at most $\opnorm{C_+ A_+^{d}}/\sigma$ by assumption.
	\end{proof}
The second part of our theorem follows by constructing $\phi$ to use the minimal polynomial of $A_+$.
\begin{lem} There exists a $\phi$ with $1 + \blockopnorm{\phi} = \|f^+\|_{\ell_1} \le 2^{\deg(f_+)} \le 2^{n_+}$ and $C_{\phi}^+=0$, where $f^+$ is the the minimal polynomial of $A_+$.
\end{lem}
\begin{proof}
Let $f^+$ denote the minimal polynomial of $A_+$, and let 
\begin{align*}
\phi = [f_1^+ I_m| f_2^+I_m | \cdots | f_{\deg(f_+)}^+I_m|\veczero].
\end{align*} 
Then, $C_{\phi}^+ = C_+ f_+(A_+) = 0$.
Since $f^+$ has all of its roots in the complex disk (as the spectrum of $A_+$ is a subset of the spectrum of $\Ast$), Lemma~\ref{lem:fnorm} implies that $1 + \blockopnorm{\phi} = \|f^+\|_{\ell_1} \le 2^{\deg(f_+)} \le 2^{n_+}$. 
\end{proof}
These two possible methods of bounding $\|\phi\|_{\blockop}$ conclude the proof.  

\subsection{Proof of Proposition~\ref{prop:observable_error}}
	Again, we assume $L = d$. Using the properties of an invariant decomposition, we can modify our error calculations as follows; the proof in stated in Section~\ref{sec:observable_error_proof}.
	\begin{lem}\label{sec:lem_observable_error} Suppose that $C_{\phi}[A_+] = 0$. Then, the conclusions of Propositions~\ref{prop:error_bound_stochastic_matx1} and~\ref{prop:error_bound_adversarial} hold with $\Gsf_{\phi},\Fsf_{\phi}$ and $\Hsf_{\phi}$ replaced by the following systems:
	\begin{align*}
		\Gsf_\phi^- := \fourtup{A_-}{\Bst}{C_{\phi}^-}{0},\quad \Fsf^-_\phi := \fourtup{A_-}{B_w}{C_{\phi}^-}{0}, \quad \Hsf_\phi^- := \fourtup{A_-}{\matx_1}{C_{\phi}^-}{0}\:.
	\end{align*}
	\end{lem}
	We can now bound the control norms of these ``minus''-systems (proof in Section~\ref{sec:observable_error_control_bounds}):
	\begin{lem}\label{lem:observ_error_control_bounds} $\Mknorm{\infty}{\Gsf_\phi^-} \le (1+\|\phi\|_{\blockop})\Mknorm{\infty}{\Gsf_-}$ and $\|(\Gsf_\phi)^-\|_{\Hinf} \le (1+\|\phi\|_{\blockop})\|\Gsf_-\|_{\Hinf}$, and similarly for the corresponding $\Fsf$ and $\Hsf$ systems. 
	\end{lem}
	Now bounding $\mixnorm{N}{\Gsf} \le \sqrt{N}\Mknorm{N}{\Gsf}$, we can simplify the bound in the stochastic model to obtain
	\begin{align*}
	    \frac{\|\matDel_{\phi}\|_{\op}}{1+\blockopnorm{\phi}} \lesssim &\; \sqrt{N(m+\log\tfrac{1}{\delta})}(\|\Markov_{\infty}(\Gsf_-)\|_{\op}+\|\Markov_{\infty}(\Fsf_-)\|_{\op}) \\
	    &\;+ \sqrt{N}  \left(\Mknorm{Td}{\Gsfst} + \Mknorm{Td}{\Fsfst}+ \opnorm{D_z}\right)\:.
	\end{align*}

\subsubsection{Proof of Lemma~\ref{lem:observ_error_control_bounds} \label{sec:observable_error_control_bounds}}
	 Let us take $\Gsf_\phi^-$ as a representative example. Using the formula 
	 \begin{equation*}
	 C_{\phi}^- := C_- A_-^{LT} - \sum_{\ell=1}^{L} \Psi_{\ell}C_- A_-^{(L-\ell)T}\:,
	 \end{equation*}
	 we have (with the convention $\Psi_{0} = I_m$),
	\begin{align*}
	\Mknorm{\infty}{\Gsf_\phi^-} &= \lim_{n \to \infty} \Mknorm{n}{\Gsf_\phi^-}\\
	&= \lim_{n \to \infty} \sum_{\ell = 0}^L\Mknorm{n}{\fourtup{A_-}{\Bst}{\Psi_{L - \ell} C_- A_-^{\ell}}{0}}\\
	&= \lim_{n \to \infty} \sum_{\ell = 0}^L\opnorm{\Psi_{L-\ell}}\Mknorm{n}{\fourtup{A_-}{\Bst}{C_- A_-^{\ell}}{0}}\\
	&\overset{(i)}{\le} \lim_{n \to \infty} \sum_{\ell = 0}^L\opnorm{\Psi_{L-\ell}}\Mknorm{n+\ell}{\fourtup{A_-}{\Bst}{C_-}{0}}\\
	&= \lim_{n \to \infty} \sum_{\ell = 0}^L \opnorm{\Psi_{L-\ell}}\Mknorm{n+\ell}{\Gsf_-}\\
	&= \sum_{\ell = 0}^L\opnorm{\Psi_{L-\ell}}\lim_{n \to \infty} \Mknorm{n+\ell}{\Gsf_-}\\
	&=  \Mknorm{\infty}{\Gsf_-} \cdot\sum_{\ell = 0}^L\opnorm{\Psi_{L-\ell}} = \Mknorm{\infty}{\Gsf_-} (1 + \|\phi\|_{\blockop}),
	\end{align*}
	where in $(i)$ we have used the fact that $\Markov_{n}(\fourtup{A_-}{\Bst}{C_- A_-^{\ell}}{0})$ is a submatrix of the matrix\\$\Markov_{n+\ell}(\fourtup{A_-}{\Bst}{C_- A_-^{\ell}}{0})$. This argument can be applied for the $\Hinf$-norm, viewed as the asymptotic limit of the operator norm of the associated Toeplitz operator (see e.g.~\cite{tilli1998singular} Corollary 4.2). 

\subsubsection{Proof of Lemma~\ref{sec:lem_observable_error}\label{sec:observable_error_proof}}
	Examining the arguments in Section~\ref{sec:error_calcs}, it suffices to modify the control of the term $\Errone$. Recall the shut-off sequence
	\begin{align*}
	\xtil_{n;t} := &\;\begin{cases} \Ast^{n - (t - L T)}\matx_{t - L T} & n \ge t - L T \\
	\matx_{n} & n \le t - L T
	\end{cases}\\
	\ytil_{n;t} :=  &\;\Cst \xtil_{n;t} \\
	\ktil_t := &\;[\ytil_{t-T;t}^\top\mid \ytil_{t-2T;t}^\top\mid\dots \mid\ytil_{t-L T;t}^\top]^\top\:,
	\end{align*}
	Defining the terms $\boldgam_t := B_w\matw_{t} + \Bst\matu_{t}$, since $C_{\phi}[\Ast] = C_{\phi}^-$ we have that 
	\begin{align*}
	\Err^{(1)}_t &= \ytil_{t;t} - \phi \cdot \ktil_t  ~= C_{\phi}[\Ast]\xtil_{t-TL} = C_{\phi}^-\xtil_{t-TL}\\
	&= C_{\phi}^-\left(\sum_{i = 0}^{t-LT - 2} \Ast^{i}\boldgam_{t-1-i} + \Ast^{t - TL - 1}\matx_1\right)\\
	&=  C_{\phi}^- \left(\boldgam_{t-1} + \sum_{i = 1}^{t-LT - 2} \Ast^{i}\boldgam_{t-1-i} + \Ast^{t - TL - 1}\matx_1\right)\\
	&=  C_{\phi}^- \left(\boldgam_{t-1} + \sum_{\sigma \in \{+,-\}} \sum_{i = 1}^{t-LT - 2} (A^{i}_{\sigma}) \boldgam_{t-1-i} + A_{\sigma}^{t - TL - 1}\matx_1\right).
	\end{align*}
	Observe now that the term corresponding to $\sigma = +$ is canceled by $C_{\phi}^-$, because it involves only terms which have the products $A_-A_+$ or $C_-A_+$. Thus,
	\begin{align*}
	\Err^{(1)}_t &=  C_{\phi}'[A_-]\left(\boldgam_{t-1} + A^{i}_{-}) \boldgam_{t-1-i} + A_{-}^{t - TL - 1}\matx_1\right).
	\end{align*}
	This term can be then be controlled analogously to the term $\Err^{(1)}_t $ in Section~\ref{sec:error_calcs}, replacing $\Ast$ with $A_-$ and $C_{\phi}$ with $C_{\phi}^-$.

\end{document}